\documentclass[10pt]{article} % For LaTeX2e
\usepackage[accepted]{tmlr}
\usepackage{amsmath,amsfonts,bm}

\def\eqref#1{equation~\ref{#1}}
\def\1{\bm{1}}

\DeclareMathAlphabet{\mathsfit}{\encodingdefault}{\sfdefault}{m}{sl}
\SetMathAlphabet{\mathsfit}{bold}{\encodingdefault}{\sfdefault}{bx}{n}

\newcommand{\R}{\mathbb{R}}

\newcommand{\Sym}{\mathbb{S}}

\newcommand{\ip}[2]{\langle #1, #2 \rangle}
\newcommand{\norm}[1]{\left\lVert#1\right\rVert}

\newcommand{\sqr}[1]{\left[#1\right]}

\newcommand{\bbE}{\mathbb{E}}

\newcommand{\defeq}{\stackrel{\text{(def)}}{=}}

\newcommand{\bb}{\mathbf{b}}

\newcommand{\bX}{\mathbf{X}}
\newcommand{\bx}{\mathbf{x}}

\newcommand{\bg}{\mathbf{g}}

\newcommand{\bl}{\bm{\lambda}}

\newcommand{\bt}{\bm{\theta}}
\newcommand{\bsg}{\bm{\sigma}}

\newcommand{\bmu}{\bm{\mu}}
\newcommand{\btu}{\bm{\tau}}

\newcommand{\bmm}{\mathbf{m}}
\newcommand{\bS}{\mathbf{S}}
\newcommand{\bM}{\mathbf{M}}
\newcommand{\bN}{\mathbf{N}}
\newcommand{\bH}{\mathbf{H}}

\newcommand{\bK}{\mathbf{K}}
\newcommand{\bL}{\mathbf{L}}
\newcommand{\bV}{\mathbf{V}}
\newcommand{\bA}{\mathbf{A}}
\newcommand{\bB}{\mathbf{B}}
\newcommand{\bC}{\mathbf{C}}
\newcommand{\bD}{\mathbf{D}}

\newcommand{\bF}{\mathbf{F}}

\newcommand{\bW}{\mathbf{W}}

\newcommand{\cL}{\mathcal{L}}

\newcommand{\cD}{\mathcal{D}}
\newcommand{\cH}{\mathcal{H}}

\newcommand{\cN}{\mathcal{N}}

\newcommand{\Id}{\mathrm{I}}

\DeclareMathOperator{\tr}{tr}
\DeclareMathOperator{\diag}{diag}
\DeclareMathOperator{\vtr}{vec}
\DeclareMathOperator{\vech}{vech}
\DeclareMathOperator{\tril}{tril}

\DeclareMathOperator{\KLop}{KL}

\usepackage{hyperref}
\usepackage{url}
\usepackage{booktabs}       % professional-quality tables
\usepackage{amsfonts}       % blackboard math symbols
\usepackage{microtype}      % microtypography
\usepackage{xcolor}         % colors
\let\AND\relax   % Undefine \AND temporarily
\usepackage{algorithm}
\usepackage{algorithmic}
\usepackage{graphicx}
\usepackage{subfig}
\usepackage{xfrac}

\usepackage{amsmath}
\usepackage{amssymb}
\usepackage{mathtools}
\usepackage{amsthm}
\usepackage{amsfonts}

\usepackage{bm}
\usepackage{thmtools}
\usepackage{stackrel}
\usepackage{multirow}
\usepackage[textsize=small, textwidth=1.2cm]{todonotes}
\usepackage[normalem]{ulem}

\usepackage{etoc}

\newcommand{\appendixTOC}{
    \begin{center}
        \textbf{Table of Contents}
    \end{center}
    \etocsettocstyle{\normalfont}{} % Formatting for the ToC
    \localtableofcontents
}

\usepackage[capitalize,noabbrev]{cleveref}

\declaretheorem{lemma}
\declaretheorem{definition}
\declaretheorem{theorem}
\declaretheorem{assumption}
\theoremstyle{remark}
\declaretheorem{remark}

\title{Optimization Guarantees for Square-Root Natural-Gradient Variational Inference}

\author{\name Navish Kumar \email navish.kumar@unibas.ch \\
      \addr University of Basel, Basel, Switzerland\\
      \addr Department of Mathematics and Computer Science\\
      \AND\\
      \name Thomas M\"ollenhoff \email thomas.moellenhoff@riken.jp \\
      \addr RIKEN Center for AI Project, Tokyo, Japan \\
      \AND\\
      \name Mohammad Emtiyaz Khan \email emtiyaz.khan@riken.jp\\
      \addr RIKEN Center for AI Project, Tokyo, Japan \\
      \AND\\
      \name  Aurelien Lucchi \email aurelien.lucchi@unibas.ch \\
      \addr University of Basel, Basel, Switzerland\\
      \addr Department of Mathematics and Computer Science}

\def\month{05}  % Insert correct month for camera-ready version
\def\year{2025} % Insert correct year for camera-ready version
\def\openreview{\url{https://openreview.net/forum?id=OMOFmb6ve7}} % Insert correct link to OpenReview for camera-ready version

\begin{document}

\maketitle

\begin{abstract}
Variational inference with natural-gradient descent often shows fast convergence in practice, but its theoretical convergence guarantees have been challenging to establish. This is true even for the simplest cases that involve concave log-likelihoods and use a Gaussian approximation. We show that the challenge can be circumvented for such cases using a square-root parameterization for the Gaussian covariance. This approach establishes novel convergence guarantees for natural-gradient variational-Gaussian inference and its continuous-time gradient flow. Our experiments demonstrate the effectiveness of natural gradient methods and highlight their advantages over algorithms that use Euclidean or Wasserstein geometries. 
\end{abstract}
% \thomas{I would try to avoid having 1-2 words in a new line, please check in entire paper at the end. I think having Variational Inference in a single new line in the title looks a bit awkward, not sure how it can be fixed.}
\section{Introduction}
\label{intro}

Variational inference (VI) is widely used in many areas of machine learning and can provide a fast approximation to the posterior distribution \citep{Jaakkola96b,  wainwright2008graphical, Gr11, kingma2013auto}. VI works by formulating Bayesian inference as an optimization problem over a restricted class of distributions which is then solved, for example, by using gradient descent (GD) \citep{Gr11, ranganath2014black, BlCo15}. A faster alternative is to use natural-gradient descent (NGD) \citep{amari1998natural} which exploits the information geometry of the posterior approximation and often converges much faster than GD, for example, see \Cref{fig:convergence-comparison}(a), \citet[Fig.~1b]{khan2018fast}, or \citet[Fig.~3]{lin2019fast}. The updates can often be implemented using message passing \citep{knowles2011non, hoffman2013stochastic, khan2017conjugate} and sometimes even as Newton's method or deep-learning optimizers \citep{khan2017variational, khan2018fast1, Salimans_2013}. Due to this, there has been a lot of interest in NGD based VI \citep{honkela2008natural, hensman2012fast, salimbeni2018natural, osawa2019practical, adam2021dual, khan2023bayesian}.

Despite their fast convergence, little has been done to understand the theoretical convergence properties of NGD, and its continuous counterpart, which we call natural-gradient (NG) flow. Instead, most works focus on GD \citep{AlRi16, domke2019provable, Do20, domke2023provable, kim2023black}. For NGD, \citet{khan2015faster} derived convergence results for general stochastic, non-convex settings. Still, their convergence rates are slow and there is plenty of room for improvement, especially for concave log-likelihoods. For such likelihoods, recently, \citet{chen2023gradient} studied the continuous-time flow. There are also some studies on stochastic NG variational inference; for example, \citet{wu2024understanding} prove non-asymptotic convergence rates for conjugate likelihoods. However, the conjugacy assumption limits their applicability to more general settings. Both in discrete and continuous time, however, there is a gap between theory and practice for convergence guarantees of NGD.

A major difficulty in obtaining stronger convergence guarantees for NGD and NG flow lies in a subtle technical issue: commonly used parameterizations can often destroy the concavity property of the log-likelihood \citep[Sec. 4.3]{cherief2019generalization}. It is well-known that the expectation of a concave log-likelihood for a Gaussian approximation is \emph{not} concave with respect to either the covariance or precision matrix, but only with respect to the Cholesky factor of the covariance \citep{challis2013gaussian}. Therefore, commonly used parameterizations, such as natural or expectation parameterization of the exponential family, destroy the concavity property and prohibit the application of tools from convex optimization to analyze the properties of NGD. Our goal in this paper is to address this issue.

We circumvent this problem by considering variational-Gaussian inference that is defined by the square root of the Gaussian covariance matrix. We show that discretizing the NG flow with this square-root parameterization results in an NGD update based on the Cholesky factor. This idea was proposed by \citet{tan2021analytic}, though it was also previously explored in a broader context by \citet{lin23c, lin24f}. Our method extends this idea to include any square-root parameterization. This way of parameterizing helps maintain the concavity of log-likelihood functions, which in turn facilitates convergence guarantees for both NG flow and NGD. The key contributions of our work are outlined as follows:

\begin{itemize}
    \item We prove that the NG flow exhibits an exponential convergence rate for a strongly concave log-likelihood, as stated in \Cref{theorem:FR-flow}. When discretized using the square-root parameterization, this aligns with the update provided by \citet[Theorem 1]{tan2021analytic}. The proof of convergence hinges on establishing a Riemannian Polyak-Łojasiewicz (PL) inequality, which is shown to be valid when the lowest eigenvalue of the metric tensor is bounded.
    \item We prove an exponential rate of convergence (\Cref{theorem:global-tan-cholesky}) for NGD using Cholesky and under strongly-concave log-likelihood; the proof extends to any square-root parameterization.
    \item We present empirical results showcasing the fast convergence of NGD, attributed to its Newton-like update. These results are illustrated in \Cref{fig:small_scale} and \Cref{fig:large_scale}. It is important to highlight that our experiments employ the piecewise bounds described in \citep{marlin2011piecewise} for computing expectations. This raises the question of how the resulting bias affects NGD convergence. To address this, we include a discrete-time analysis with bias in our supplementary section, see \Cref{theorem:biased-tan-cholesky} in \Cref{app:ngd_w_bias}. The reason for our preference for the piecewise over the stochastic implementation is justified in \Cref{sec:piece_vs_stoc}.
\end{itemize}Our proof addresses the additional challenges posed by the non-smoothness of entropy and does not require any proximal or projection operators such as those used by \citet{domke2019provable, domke2023provable}. Our results can be useful for further research on proving such results for generic NGD updates.~The use of square-roots could lead to suboptimal convergence, for instance, it may not converge in one step for quadratic problems (as shown in \Cref{fig:convergence-comparison}(a)), but this only implies that even better results are possible for the general case. We provide an insight into why this happens in \Cref{sec:SR_VN_vs_VN}, noting that square-root-based updates essentially compute an approximate estimate of the inverse of the preconditioning matrix that is used in Newton's method.  

\begin{figure}[t!]
    \centering
    \subfloat[\centering ]{{\includegraphics[trim={2cm 4cm 0 0},clip,width=0.226\textheight]{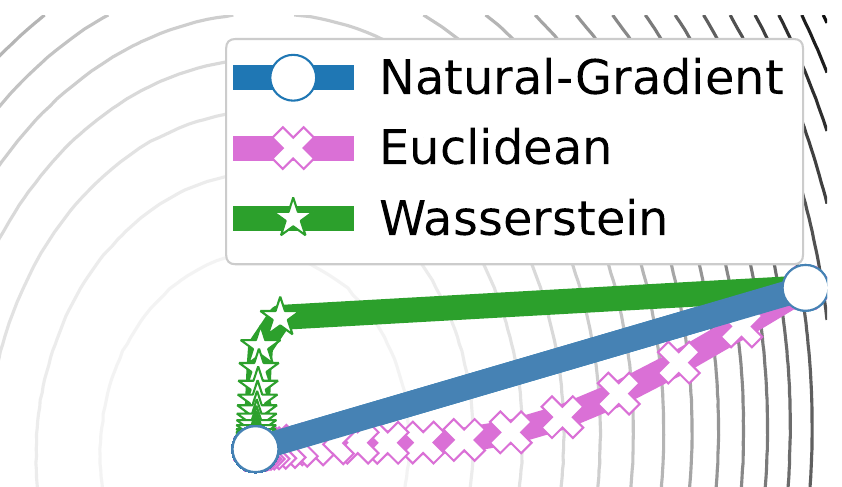} }}%
    \quad
    \subfloat[\centering ]{{\includegraphics[trim={2cm 4.1cm 0 0},clip,width=0.226\textheight,height=0.12\textheight]{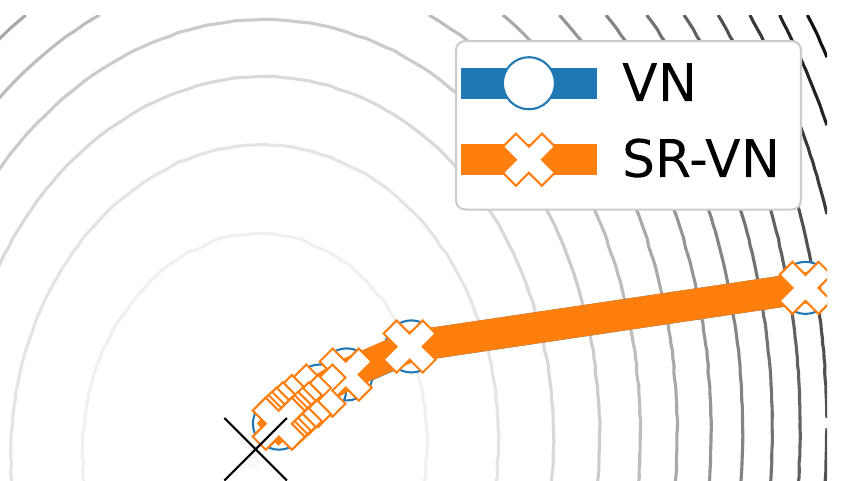}}}%
    \quad
    \subfloat[\centering ]{{\includegraphics[trim={2cm 4.1cm 4 4},clip,width=0.226\textheight,height=0.12\textheight]{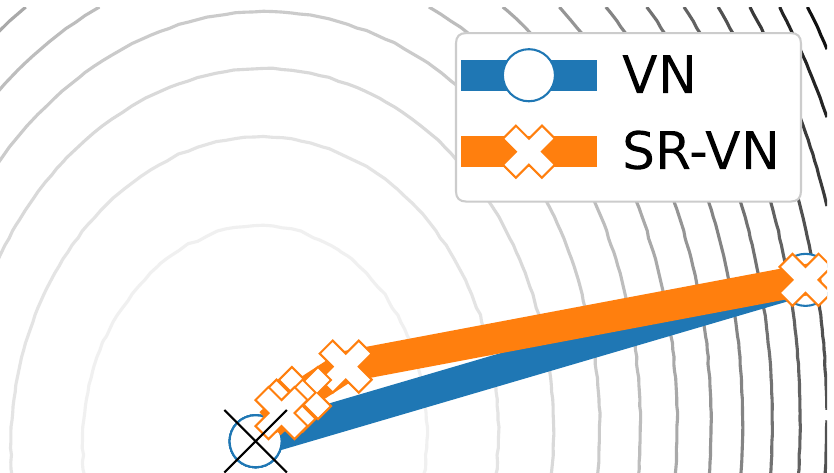}}}%
    \caption{\small{(a) Natural-gradient descent (circle) converges much faster (in just 1 step) than algorithms that use Euclidean (cross) or Wasserstein (star) geometries which take 10 and 470 iterations respectively. The illustration is on a 2-D Bayesian linear regression (quadratic loss in the background for all three figures). The 3 methods are taken from \cite{khan2023bayesian, ranganath2014black, lambert2022variational}, (b) For smaller step-sizes (of the order $\rho=10^{-5}$ or less), SR-VN and VN perform almost identically, (c) For larger step-sizes they start to differ, step-size for VN here is $1$.}}%
    \label{fig:convergence-comparison}%
\end{figure}

\section{Variational Inference}

The goal of variational inference is to approximate the posterior distribution by a simpler distribution, such as the Gaussian distribution. Given a Bayesian model with $n$ likelihoods, denoted by $p(\mathcal{D}_i|\bt)$ for the $i$'th data example with parameter $\bt$, and a prior $p(\bt)$ over the parameters, Bayesian inference aims to find the posterior $p(\bt\mid\mathcal{D}_1, \mathcal{D}_2, \ldots, \mathcal{D}_n)$. In variational inference, we aim to find an approximation $q_{\btu}(\bt)$ to the posterior by solving the following problem over the space $\btu\in\Omega $ of the valid \emph{variational parameters},
\begin{equation}
\label{eq:KL-form}
\min_{\btu \in \Omega } ~ \bbE_{q_{\btu}(\bt)}[\ell(\bt)] + \gamma \KLop \left[q_{\btu}(\bt) \| p(\bt) \right],
\end{equation}
where we denote $\ell(\bt) = -\sum_{i=1}^n \log p(\mathcal{D}_i|\bt)$, $\KLop[\cdot\|\cdot]$ to be the KL divergence, and $\gamma>0$ is a scalar parameter commonly referred to as the temperature. When $\gamma = 1$, the objective is referred to as the negative evidence lower bound (ELBO) \citep{Jaakkola96b, blei2017variational}, which is also the focus of our analysis. The value of $\gamma$ plays a crucial role in shaping the posterior variance.  If $\gamma$ is set too high, the KL term dominates and forces the variational posterior $q_{\btu}(\bt)$ to collapse to the prior $p(\bt)$, eliminating any learned posterior structure. Conversely, if $\gamma$ is set too low, the KL penalty becomes negligible, which can lead to high-variance gradient estimates and unstable training \citep{shen2024variational}.

% If $\gamma$ is set too high, the KL term dominates and the variational posterior will collapse to the prior. Conversely, if $\gamma$ is too low, insufficient regularization can lead to high variance and unstable training \citep{shen2024variational}.
% If $\gamma$ is set too high, it can result in large variance and unstable training. Conversely, if $\gamma$ is too low, it may cause the posterior to collapse, providing minimal practical benefits . 

In many problems, the negative log-likelihood (NLL) $\ell(\bt)$ is convex; it could also be an arbitrary loss function, giving rise to \emph{generalized} posterior \citep{zhang1999theoretical, catoni2007pac}. The regularizer, denoted by $R(\bt) = -\log p(\bt)$, too is chosen to be convex. Popular examples include Bayesian linear and logistic regressions, as well as Gaussian process models. In such cases, a more convenient form is used to separately define two terms as the expected log-loss and negative-entropy respectively, that is $f(\btu) = \bbE_{q_{\btu}(\bt)}[\ell(\bt) + \gamma R(\bt)],~\cH(\btu) = \bbE_{q_{\btu}(\bt)}[\log q_{\btu}(\bt)].$ Using these, we can rewrite \Cref{eq:KL-form} by expanding the KL divergence as 
\begin{equation}
\label{eq:general-elbo}
\min_{\btu \in \Omega } \cL(\btu) \quad \text{where} \quad \cL(\btu) := ~ f(\btu) + \gamma \cH(\btu),
\end{equation}
and where the first term involves the expectation of a convex loss function while the second term is the entropy. This reformulation allows us to connect VI to a wide variety of problems in supervised, unsupervised, and reinforcement learning which can be formulated as a minimization of the form
\begin{align}\label{eq:generic-opt}
\min_{\bt \in \R^p} \Bar\ell(\bt) \quad \text{where} \quad \Bar\ell(\bt) := \ell(\bt) + \gamma R(\bt).
\end{align}
VI approaches can be seen as a variant of such problems where instead of finding a single minimizer $\bt^*$, we seek a parametric distribution $q_{\btu^*}(\bt)$ by solving \Cref{eq:general-elbo}. This is closely related to techniques in robust or global optimization that use $q_{\btu}(\bt)$ to smooth the loss and mitigate the influence of local minima \citep{mobahi2015theoretical, leordeanu2008smoothing, hazan2016graduated}, often with $\gamma$ taking a value of $0$. %The development of efficient and scalable algorithms for solving \ref{eq:general-elbo} holds the potential to significantly impact and advance these diverse fields.
This approach is also referred to as Variational Optimization (VO)~\citep{staines2012variational}. VO can lead us to the minimum $\bt^*$ because, for $\gamma=0$, $\cL(\btu)$ is an upper bound on the minimum value of $\Bar\ell$, that is, $\min_{\bt} \Bar\ell(\bt) \leq \min_{\btu \in \Omega }\cL(\btu)$. Therefore minimizing $\cL(\btu)$ minimizes $\Bar\ell(\bt)$, and when the distribution $q$ puts all its mass on $\bt^*$, we recover the minimum value. Problem~\ref{eq:general-elbo} with $\gamma=0$ is also encountered across various other domains, including random search \citep{baba1981convergence}, stochastic optimization \citep{spall2005introduction}, and evolutionary strategies \citep{beyer2001theory}. In this context, $q_{\btu}(\bt)$ serves as the `search' distribution employed to locate the global minimum of a black-box function, $\ell(\bt)$. Notably, in the realm of reinforcement learning, this distribution can take on the form of a policy distribution designed to minimize the expected value-function $\ell(\bt)$ \citep{sutton1998introduction}, sometimes incorporating entropy regularization techniques \citep{williams1991function, teboulle1992entropic, mnih2016asynchronous}.

% In what follows, we will always use $\gamma=1$ for our analysis, which means we are working with the negative ELBO formulation throughout.

% -----------------------------------------------------------------------------------
\section{Background on NGD for VI}
We will review VI using GD and NGD methods and discuss the NG flow and NGD with the Cholesky factor, and for both, we will provide convergence guarantees in the later sections.

\subsection{Variational Gaussian Inference with GD}
Throughout, we will focus on VI with a Gaussian distribution $q_{\btu}(\bt):= \cN(\bt\,|\,\bmm, \bV)$ where $\bmm \in \R^d$ is the mean, $\bV \in \R^{d \times d}$ is the covariance. The variational parameter $\btu = (\bmm, \bV)$ is only defined within the set of valid parameters $\Omega'$ -- the set of $\btu$ where $\bV$ is positive definite, that is $\Omega ^\prime = \{(\bmm, \bV): \bmm\in\R^d, \bV \succ 0 \in\R^{d \times d}\}$. The problem \ref{eq:general-elbo} can then be rewritten as, ${\min_{\bmm, \bV} ~ \cL(\bmm, \bV).}$ The function $\cL$ is differentiable under mild conditions on the interior of $\Omega^\prime$ even when $\ell$ is not differentiable, as discussed by \cite{staines2012variational}. This makes it possible to
apply gradient-based optimization methods to optimize it. A straightforward approach to minimize $\cL$ is to use Gradient Descent (GD) as shown below,
\begin{equation}
    \begin{aligned}
    \text{\bf GD:} \quad \bV_{t+1} = \bV_t - \rho \nabla_\bV \cL(\bmm_t, \bV_t),\quad
    \bmm_{t+1} = \bmm_t - \rho \nabla_\bmm \cL(\bmm_t, \bV_t),
\end{aligned}
\end{equation}
where $\rho > 0$ is a step size. This approach is often referred to as the black-box VI (BBVI) in the literature \citep{ranganath2014black}, and offers simple and convenient updates to implement using modern automatic-differentiation methods and the reparameterization trick~\citep{Gr11, BlCo15, titsias2014doubly}. Unfortunately, such direct gradient-descent methods can be slow in practice.

\subsection{Variational Gaussian Inference with NGD}

An alternative approach is to use natural-gradient descent which exploits the information geometry \citep{amari2016information} of $q$ to speed up convergence.  Given the Fisher information matrix (FIM) \citep{amari2008information} of $q_{\btu}(\bt)$, that is $\mathbf{F}_{\btu} = \bbE_{q_{\btu}(\bt)}[\nabla_\tau \log q_{\btu}(\bt)(\nabla_\tau \log q_{\btu}(\bt))^\top]$, which is positive-definite for all $\btu \in \Omega$ in our setting, the NGD for VI in the natural-parameter space (with $\btu$ as natural parametrization) is given as follows \citep{khan2023bayesian}:
\begin{align}\label{eq:ngd}
    \text{\bf NGD:} \quad \btu_{t+1} = \btu_t- \rho \bF_{\btu_t}^{-1} \nabla_{\btu} \cL(\btu_t).
\end{align}The preconditioning of the gradient by the FIM leads to proper scaling in each dimension which often leads to faster convergence. A notable connection between NGD for \ref{eq:general-elbo} and Newton’s method for \ref{eq:generic-opt} was established by \cite{khan2017variational, khan2018fast1} in the case of a Gaussian distribution with mean $\bmm$ and the precision $\bS$, which is the inverse of the covariance $\bS = \bV^{-1}$. Denoting the expected gradient and the expected Hessian, respectively, as:
\begin{equation}
\begin{aligned}\label{eq:grads}
    \bg_t = \bbE_{\bt \sim \cN(\bmm_t, \bV_t)}[\nabla_{\bt} \Bar \ell(\bt)], \quad\bH_t = \bbE_{\bt \sim \cN(\bmm_t, \bV_t)}[\nabla^2_{\bt}  \Bar \ell(\bt)].
\end{aligned}
\end{equation}The resulting updates, 
\begin{equation}\label{eq:blr-updates}
\begin{aligned}
    \text{\bf VN:} \quad \bS_{t+1} = (1 - \gamma\rho)\bS_t + \rho \bH_t,\quad
    \bmm_{t+1} = \bmm_t - \rho (\bS_{t+1})^{-1}\bg_t,
\end{aligned}
\end{equation}referred to as the variational Newton (VN) update, optimize \ref{eq:general-elbo} instead of \ref{eq:generic-opt}. The standard Newton’s update for problem \ref{eq:generic-opt} is recovered by approximating the expectations at the mean $\bmm$ and using step-size $\rho = 1$ when $\gamma = 1$; see \cite{khan2023bayesian}. Since the precision $\bS$ lies in a positive-definite matrix space, the update \ref{eq:blr-updates} may violate this positivity constraint of the parameter space $\Omega$ \citep{khan2018fast}. For example, this happens when the loss $\ell(\bt)$ is non-convex.

% ++++++++++++++++++++++++++++++++++++++++++++++++++++++++++++++++++++++++
\subsection{Natural Gradient Flow}
The NG flow minimizes the KL divergence functional on a parameterized manifold of Gaussian densities -- the objective stated in \Cref{eq:KL-form}. That is, NG flow is a continuous time update of the NGD (\Cref{eq:ngd}, for $\gamma=1$) that is, 
\begin{align}\label{eq:NGD_flow}
    \frac{d\btu_t}{dt} = - \bF_{\btu_t}^{-1} \nabla_{\btu} \cL(\btu_t) = - \bF_{\btu_t}^{-1} \nabla_{\btu}\KLop\bigg[q_{\btu}(\bt)\mid\mid p(\bt\mid\cD)\bigg].
\end{align}The updates for the NG flow in $\btu = (\bmm, \bV)$ parameterization were provided in \citet[Equation (4.18)]{chen2023gradient} as:
\begin{equation}
\begin{aligned}\label{eq:NG flow}
        \frac{d\bV_t}{dt} = \bV_t - \bV_t\bH_t\bV_t, \quad
        \frac{d\bmm_t}{dt} = -\bV_t\bg_t.
\end{aligned}
\end{equation}Upon discretization with respect to the natural parameters, this results in the VN update with $\gamma=1$. Refer to \Cref{app:SR_VN_vs_VN} for a comparison of the $\bV$ updates, as obtained from VN. However, due to the non-convex nature of $\KLop$ even with concave log-likelihood for $(\bmm, \bV)$ parameterization (see further discussion in \Cref{sec:problems_w_convexity}), we instead focus on analyzing the NG flow in the square-root parameterization, that is $\btu = (\bmm, \bC) \in \Omega$, where $\Omega = \{(\bmm, \bC): \bmm\in\R^d, \bV=\bC\bC^\top \succ 0 \in\R^{d \times d} \}$ under the assumption that $\bC$ is square and lower-triangular with positive real diagonal entries, ensuring its invertibility. Let us define a function that takes the lower-triangular part of a matrix and halves its diagonal, that is
\begin{align}\label{eq:tril}
    \tril[\bA]_{ij} := \left\{\begin{array}{lr}
    A_{ij} & i>j,\\
    \frac{1}{2}A_{ii} & i=j,\\
    0 & i<j.
    \end{array}\right.
\end{align}Then, to derive flow equations for the Cholesky factor, we first note from \citet[Equation (7)]{murray2016differentiation} that the derivative of the Cholesky factor $\bC_t$ can be obtained from the derivative of the covariance $\bV_t = \bC_t\bC_t^\top$ as
\begin{align}
    \frac{d\bC_t}{dt} = \bC_t \tril[\bC_t^{-1}\frac{d\bV_t}{dt}\bC_t^{-\top}],
\end{align}which after using \Cref{eq:NG flow} leads to the following flow equations:
\begin{equation}\label{eq:SR-VN-flow}
    \begin{aligned}
        \frac{d\bC_t}{dt} = \bC_t \tril[\Id - \bC_t^\top\bH_t\bC_t], \quad
        \frac{d\bmm_t}{dt} = -\bC_t\bC_t^\top\bg_t.
    \end{aligned}
\end{equation}In \Cref{theorem:FR-flow}, we prove the convergence of the flow dynamics in \Cref{eq:SR-VN-flow} under the concave log-likelihood. Importantly, the NG flow remains the same regardless of the parameterization, and therefore the flow dynamics in \Cref{eq:NG flow} and \Cref{eq:SR-VN-flow} trace out the same trajectory on the Gaussian manifold. Hence, we focus on square-root parametrization solely for its theoretical convenience in proving convergence.

%------------------------------------------------
\subsection{Square-Root Variational Newton (SR-VN)}\label{sec:SR-VN}

%------------------------------------------------
\begin{algorithm}[t!]
\caption{\textbf{Square-Root Variational Newton (SR-VN)}}
\begin{algorithmic}[1]
\STATE {\bf{Parameter:}} Step-size $\rho>0$ is chosen using \Cref{eq:permissible-step-size};~$\tril$ is defined in \Cref{eq:tril}.
\STATE Initialize $(\bmm_0, \bC_0)$ to satisfy Assumption~\ref{assum:init}
\FOR{$t=0, \hdots, T$}
% \STATE $\bH_t = \bbE_{\bt \sim \cN(\bmm_t, \bV_t)}[\nabla^2_{\bt} \Bar \ell(\bt)]$
\STATE Update $\bg_t$ and $\bH_t$ using \Cref{eq:grads}
\STATE $\bC_{t+1} \leftarrow \bC_t - \rho \bC_t\tril[\bC_t^\top\bH_t\bC_t - \gamma \Id]$
\STATE $\bmm_{t+1} \leftarrow \bmm_t - \rho \bC_t\bC_t^\top \bg_t$
% \STATE $\bmm_{t+1} \leftarrow \bmm_t - \rho \bC_t\bC_t^\top \bbE_{\bt \sim \cN(\bmm_t, \bV_t)}[\nabla_{\bt} \Bar \ell(\bt)]$
\ENDFOR
\end{algorithmic}
\label{algo:TC}
\end{algorithm}
%------------------------------------------------
In this section, we introduce a different NGD update called square-root variational Newton (SR-VN), outlined in \Cref{algo:TC}, which utilizes the Cholesky factor $\bC$ of the covariance matrix $\bV$. SR-VN is derived through a forward Euler discretization of the NG flow defined in square-root parameterization (\Cref{eq:SR-VN-flow}). The resultant update aligns with independently proposed updates by \citet{tan2021analytic}, who defined a VN-like update using the Cholesky factor, given as
\begin{align*}
    \bC_{t+1} = \bC_t - \rho \bC_t \tril[\bC_t^\top\nabla_{\bC} \cL (\bmm_t, \bC_t)], \quad
    \bmm_{t+1} = \bmm_t - \rho \bC_t\bC_t^\top \nabla_{\bmm} \cL (\bmm_t, \bC_t).
\end{align*}
Now, as calculated in \Cref{eq:nablaf_C} in Appendix \ref{app:grad}, we have 
\begin{align}
    &\nabla_{\bC}\cL(\bmm_t, \bC_t)= (\bH_t-\gamma (\bC_t\bC_t^\top)^{-1})\bC_t.\label{eq: nabla_C_L}
\end{align}Finally, by seeing that $\nabla_{\bmm} \cL (\bmm_t, \bC_t)=\bg_t$, we can rewrite the updates from \citet[ Theorem 1]{tan2021analytic} as 
\begin{align}
    \text{\bf SR-VN:} ~ \bC_{t+1} = \bC_t - \rho \bC_t\tril[\bC_t^\top\bH_t\bC_t - \gamma \Id], \quad
    \bmm_{t+1} = \bmm_t - \rho \bC_t\bC_t^\top \bg_t, \label{eq:tan-cholesky}
\end{align}allowing us to immediately see that these updates indeed come from a direct forward Euler discretization of \Cref{eq:SR-VN-flow}. In \Cref{theorem:global-tan-cholesky}, we establish convergence guarantees for SR-VN. We also note here that other updates based on the Cholesky factor have been considered in the literature, for example, see \citep{sun2009efficient, salimbeni2018natural, glasmachers2010exponential}.

% ---------------------------------------------------------------
\subsection{SR-VN as an Approximation of VN}\label{sec:SR_VN_vs_VN}

It is interesting to see that SR-VN performs a Newton-like update in $\bmm$ (\Cref{eq:tan-cholesky}), but does not require a matrix inversion of the preconditioner (such as the inversion of the Hessian in VN, \Cref{eq:blr-updates}). This makes it easier to implement and analyze. However, we show in \Cref{app:SR_VN_vs_VN} that, in reality, SR-VN computes an approximation of the inverse and is therefore only suboptimal compared to VN. Looking at \Cref{fig:convergence-comparison}(b), we can see that when the step-size is very small, SR-VN and VN converge to the same dynamics because they both approximate the NG flow. However, the quadratic offset introduced by the inverse approximation used in SR-VN leads it to perform suboptimally compared to VN. VN, on the other hand, computes the full inverse of the FIM defined using natural parameters (\Cref{eq:ngd}), and it shows its superiority by achieving one-step convergence in \Cref{fig:convergence-comparison}(c) when the step-size is larger, specifically when it is set to 1.

%------------------------------------------------
\subsection{Variational Parameters and Convexity}\label{sec:problems_w_convexity}

NGD updates are often written in the natural-parameter space \ref{eq:ngd}, but this can destroy the convexity of the problem. Specifically, given a convex function $\bar{\ell}(\bt)$, it is well-known that $\bbE_{\bt \sim \cN(\bmm, \bV)}[\Bar \ell(\bt)]$ is jointly convex w.r.t. $\bmm$ and the Cholesky factor $\bC$ of $\bV$, but not w.r.t $\bV$ or its inverse \citep{challis2013gaussian, domke2019provable}. That is, even for the simplest convex cases, such as logistic regression, the convexity is lost in the natural and expectation parameter spaces. For example, given $n$ training data points $\{\bx_i, y_i\}_{i=1}^n$ with $\bx_i \in \R^d$ and $y_i \in \{-1, +1\}$, in Bayesian logistic regression we minimize the following loss that contains the $\ell_2$-regularization with regularization constant $\beta >0$:
\begin{align}\label{eq:logistic-loss}
    \Bar\ell_{lg}(\bt) = \sum_{i=1}^n \bigg[\log(1+\exp\{y_i(\bt^\top\bx_i)\})\bigg] + \frac{\beta}{2} \norm{\bt}^2
\end{align}The Hessian takes the form $\nabla^2_{\bt}\Bar\ell_{lg}(\bt) = \bX\bD\bX^\top + \beta \Id$, where $\bX = [\bx_1, \ldots, \bx_n]^\top \in \R^{n \times d}$ is the data matrix, $\bD \in \R^{d \times d}$ is a diagonal matrix such that $D_{ii} = \bsg(\bt^\top\bx_i)(1 - \bsg(\bt^\top\bx_i))$ with the sigmoid function $\bsg(z) = 1/(1+e^{-z})$. Given this form of the Hessian, it is clear that the Hessian is positive-definite (and hence regularized logistic loss is $\beta$-strongly convex), however, the expected loss $\bbE_{\bt \sim q_{\btu}(\bt)}[\Bar\ell_{lg}(\bt)]$ is non-convex in natural and expectation parameterizations, or even in the $(\bmm, \bV)$ parameterization of the Gaussian. 

This issue makes it harder to analyze convergence guarantees of NGD in \Cref{eq:ngd} in general, and therefore for the algorithm given in \Cref{eq:blr-updates}. The same is true for the NG flow in \Cref{eq:NG flow}, because $\KLop$ is non-convex in $\btu = (\bmm, \bV)$. The issue is specifically noted by \cite{cherief2019generalization} for NGD. They analyze the convergence of online VI by using tools from convex optimization, but their proof techniques do not extend to NGD because of the loss of convexity in the natural or expectation parameters. Our goal in this paper is to address this challenge. Our key idea is to simply switch to the square-root parameterization where the convexity properties are preserved. We consider specific NGD variants that use such square-root parameterization and derive strong convergence guarantees for them. Additionally, by leveraging the strong convexity of the log-likelihood we show that $\KLop$ functional satisfies a Riemannian PL inequality in square-root parameterization, which eventually leads us to prove convergence for NG flow.

% --------------------------------------------------------------------------------------------------------------------------------------------
\section{Assumptions}\label{sec:assumptions}

We will use the following assumptions for our analysis.
\begin{assumption}[Initialization]
\label{assum:init}
We initialize \Cref{algo:TC} with $\btu_0 = (\bmm_0, \bC_0) = (\bmm_0, C_0\Id)$ for a constant $C_0 > 0$ and $\bmm_0 \in \R^d$.
\end{assumption}

\begin{assumption}[Strong Convexity]\label{assum:strong-convexity}
    We assume that the negative log-likelihood $\Bar\ell(\bt)$ is $\delta$-strongly-convex in $\bt$, that is,
    $\nabla_{\bt}^2\Bar\ell(\bt) \succcurlyeq \delta\Id$. Then, from \cite[Theorem 9]{domke2019provable} $\cL(\bmm, \bC)$ is $\delta$-strongly-convex in square-root parametrization $\btu = (\bmm, \bC)$, that is
    \begin{align}\label{eq:strong_convex_new}
    \cL(\btu') \geq \cL(\btu) + \ip{\nabla_{\btu}\cL(\btu)}{\btu'-\btu} + \frac{\delta}{2}\norm{\btu'-\btu}^2
\end{align}
\end{assumption}

\begin{assumption}[Lipschitz Gradient]\label{assum:lipschitz}
    We assume that the negative log-likelihood $\Bar\ell(\bt)$ is $M$-Lipschitz-smooth in $\bt$, that is,
    $\nabla_{\bt}^2\Bar\ell(\bt) \preccurlyeq M\Id.$ Then, from \cite[Theorem 1]{Do20} $\cL(\bmm, \bC)$ is $M$-Lipschitz-smooth in square-root parametrization $\btu = (\bmm, \bC)$, that is
    \begin{align}\label{eq:lip_new}
    \cL(\btu') \leq \cL(\btu) + \ip{\nabla_{\btu}\cL(\btu)}{\btu'-\btu} + \frac{M}{2}\norm{\btu'-\btu}^2
\end{align}
\end{assumption}

\begin{assumption}[Bounded Iterates]\label{assum:bounded-iterates}
   Let the iterates generated at time $t \in \mathbb{N}_0$ from \Cref{algo:TC} be $\bC_t$. Then the following two statements are true.
    \begin{enumerate}
        \item There exist constants $0<\xi_l\leq \xi_u < \infty$ such that $\xi_l \leq \norm{\bC_t}_F \leq \xi_u$ for all $t \geq 0$.
        \item There exists a constant $\lambda_{\min} > 0$ such that $\bV_t \succcurlyeq \lambda_{\min}\Id$ for all $t \geq 0$.
    \end{enumerate}
\end{assumption}

\begin{remark}\label{remark:high_eigval}
The largest singular value of $\bV_t$ (that is, $\norm{\bV_t}$) can be upper bounded using Assumption \ref{assum:bounded-iterates} as follows: $\norm{\bV_t} = \norm{\bC_t\bC_t^\top} \leq \norm{\bC_t}^2 \leq \norm{\bC_t}^2_F \leq \xi_u^2$.
\end{remark}
Assumptions \ref{assum:strong-convexity} and \ref{assum:lipschitz} hold for commonly encountered loss functions, e.g. the $\ell_2$-regularized logistic loss \ref{eq:logistic-loss}. \Cref{assum:bounded-iterates} ensures that the updates remain bounded throughout optimization, and enforces the positive-definiteness of $\bV$.  

\textbf{The assumptions hold for Logistic Regression.}
We saw in \Cref{sec:problems_w_convexity} that the Hessian of the logistic loss takes the form $\bX\bD\bX^\top + \beta \Id$, where the diagonal matrix $\bD$ has entries bounded between $0$ and $1/4$. Hence, the Hessian of the regularized logistic loss is bounded as long as the data matrix $\bX$ is bounded (which is often the case), i.e., $\beta \Id \preccurlyeq \bH \preccurlyeq \zeta \Id$ where $\zeta = \lambda_{\max}(\bX\bD\bX^\top) + \beta$. Thus, the assumption of a bounded square root in part 1 of \Cref{assum:bounded-iterates} holds with an appropriate selection of step-size in \Cref{eq:tan-cholesky}. Additionally, \Cref{assum:strong-convexity} and part $2$ of \Cref{assum:bounded-iterates} are met because regularized logistic regression is $\beta$-strongly convex. Therefore, Assumptions~\ref{assum:init}--\ref{assum:bounded-iterates} apply to the original VN's $\bV$ update when working with the regularized logistic loss. As confirmed empirically, the same conclusion extends to the SR-VN's $\bC$ update. 

% One could also achieve this by appropriately choosing the step size.

%------------------------------------------------
\section{Analysis}
 In this section, we will analyze the convergence guarantees for the NG flow (\Cref{eq:SR-VN-flow}) and SR-VN (\Cref{algo:TC}). We will do this by applying tools from convex optimization, considering a set of assumptions related to boundedness and convexity.  To prove convergence, we need \Cref{assum:strong-convexity} and part $2$ of \Cref{assum:bounded-iterates} for the NG flow, and Assumptions \ref{assum:init}--\ref{assum:bounded-iterates} for SR-VN.

\subsection{Flow Convergence with Riemannian PL}
% ----------------------------------------------------
We first show that, under the assumption of strongly convex NLL (\Cref{assum:strong-convexity}), the $\KLop$ functional satisfies a local Riemannian PL inequality in the square-root parameterization (refer to \Cref{lemma:PL}). This result paves the way for proving the convergence of the NG flow, as outlined in \Cref{theorem:FR-flow}. 

We first state the following lemma that ensures that the FIM is bounded, and thus well-behaved, which is crucial for our entire analysis. The detailed proof for all the stated results can be found in \Cref{app:ngd_conv}. 

\begin{lemma}[Bounded FIM eigenvalues]\label{lemma:bdd_FIM_eig}
    For $\btu = (\bmm, \bC)$, given Assumptions \ref{assum:bounded-iterates}, we conclude that 
    \begin{align*}
        \lambda^{g}_{\min}\begin{pmatrix}
      \Id_{d\times d} & 0 \\
       0 & \Id_{d^2\times d^2} \\
  \end{pmatrix} \preccurlyeq \bF_{\btu}^{-1} \preccurlyeq \lambda^{g}_{\max}\begin{pmatrix}
      \Id_{d\times d} & 0 \\
       0 & \Id_{d^2\times d^2} \\
  \end{pmatrix},
    \end{align*}where 
\begin{align}
    \lambda^{g}_{\min} = \min\left\{\lambda_{\min}, \frac{\lambda_{\min}^2}{2}\right\}, \quad \text{and} \quad \lambda^{g}_{\max} = \xi_u^2.
\end{align}
\end{lemma}
We know that any minimizer $q_{\btu_*}(\bt) = \cN(\bt|\bmm_*, \bV_*)$ satisfies \citep{khan2023bayesian, lambert2022recursive}
\begin{equation}
\begin{aligned}\label{eq:limit_pnts}
    \bbE_{\bt \sim q_{\btu_*}(\bt)}[\nabla_{\bt} \Bar \ell(\bt)] = 0, \quad \text{and} \quad
    \bV^{-1}_* = \bbE_{\bt \sim q_{\btu_*}(\bt)}[\nabla^2_{\bt} \Bar \ell(\bt)].
\end{aligned}
\end{equation}

\paragraph{Existence of limit points. }The limit points $\bmm^*, \bC^*$ exist because of the strong convexity assumption (\Cref{assum:strong-convexity}) which implies there exists a $\btu^* =(\bmm^*, \bC^*)$ that uniquely minimizes $\cL$. The existence of this $\btu^*$ can be argued using the fact that $\cL$ is continuous and coercive and  $\Omega$ (the set of $\btu'$s) can be extended to a closed set (but the minimizer can be shown to lie inside $\Omega$ due to \Cref{eq:limit_pnts}). Then, a global minimum exists due to the Weierstrass extreme value theorem.

As a consequence of \Cref{lemma:bdd_FIM_eig}, one can prove the following lemma, given that NLL is strongly-convex, which establishes a PL inequality where we measure the magnitude of the gradient in the (natural) Riemannian geometry. This will be needed to prove the convergence guarantees for the Riemannian gradient flow and descent.

\begin{lemma}[Riemannian PL Inequality]\label{lemma:PL}
    The KL functional satisfies a local Polyak-Łojasiewicz (PL) inequality with PL constant $\mu = \delta\lambda^g_{\min}$, that is,
    \begin{align}\label{eq:PL_KL}
        \norm{\nabla_{\btu}\KLop\bigg[q_{\btu}(\bt)\mid\mid p(\bt\mid\cD)\bigg]}_{\bF_{\btu}^{-1}}^2 \geq 2\mu \left(\KLop\bigg[q_{\btu}(\bt)\mid\mid p(\bt\mid\cD)\bigg] - \KLop\bigg[q_{\btu_*}(\bt)\mid\mid p(\bt\mid\cD)\bigg]\right) .
    \end{align}
\end{lemma}
\begin{remark}
    \Cref{lemma:PL} is a 'local` PL inequality because it only holds for the chosen parametrization. This is because of the use of strong-convexity of ELBO for its proof (see \Cref{app:PL}), which only holds for $(\bmm, \bC)$ parametrization. Since one does not know if convexity will hold in all parametrizations there is no guarantee that the PL inequality will hold globally.
\end{remark}
Once the Riemannian PL condition is established, we can derive a convergence rate using a standard Lyapunov function analysis. As mentioned earlier, this analysis can be modified to incorporate a bias term as shown in \Cref{app:ngd_w_bias}.

% ========================================================================
\begin{theorem}\label{theorem:FR-flow}
Under Assumptions \ref{assum:strong-convexity} and \ref{assum:bounded-iterates}, the NG flow dynamics defined in \Cref{eq:SR-VN-flow} satisfy
\begin{align*}
    &\KLop\bigg[q_{\btu_t}(\bt)\mid\mid p(\bt\mid\cD)\bigg] -  \KLop\bigg[q_{\btu_*}(\bt)\mid\mid p(\bt\mid\cD)\bigg] \\
    &\qquad \leq e^{-2\mu t} \bigg(\KLop\bigg[q_{\btu_0}(\bt)\mid\mid p(\bt\mid\cD)\bigg] - \KLop\bigg[q_{\btu_*}(\bt)\mid\mid p(\bt\mid\cD)\bigg]\bigg). 
\end{align*}
\end{theorem}

% -------------------------------------------------------------------
\subsection{Convergence of SR-VN}\label{sec:conv-result}

Under \Cref{assum:init}--\ref{assum:strong-convexity}, we give what we believe are the first rigorous convergence guarantees for optimizing \ref{eq:general-elbo} using \Cref{algo:TC}. 
Specifically, we show an exponential rate of convergence for \Cref{algo:TC}, as stated in Theorem \ref{theorem:global-tan-cholesky}, with a detailed proof provided in \Cref{app:proof-tan-chol}. The literature has suggested various updates, as seen in references \cite{lin2019fast, lin2021tractable, lin2021structured, khan2023bayesian}. It is also worth highlighting that our proof extends to generic NGD algorithms that use \emph{any} square-root parametrization. The proof technique used to derive Theorem~\ref{theorem:global-tan-cholesky} can easily be adapted to prove convergence for these other variants and is omitted here to avoid redundancy.

% ****************************************
\begin{theorem}[Global Convergence]
\label{theorem:global-tan-cholesky}
 Given Assumptions \ref{assum:init}--\ref{assum:bounded-iterates} and considering the limit points $\bmm_*$ and $\bC_*$ in \Cref{eq:limit_pnts}, the updates provided by Algorithm \ref{algo:TC} satisfy
\begin{equation}\label{eq:global_convergence}
\begin{aligned}
    &\cL(\bmm_{t+1}, \bC_{t+1})  - \cL(\bmm_*, \bC_*) \leq(1 - 2\eta\delta)^{t+1} (\cL(\bmm_0, \bC_0)  - \cL(\bmm_*, \bC_*)),
\end{aligned}
\end{equation}where
\begin{equation}
\eta = \min\left\{\lambda_{\min} \rho-\frac{M\rho^2\xi_u^4}{2},\quad \sqrt{\frac{5}{2}}\rho\xi_l^2 - \frac{5\rho^2\xi_u^4M}{4}\right\}.
\label{thm:eta}
\end{equation}
For convergence, we require that $0 < (1 - 2\eta\delta) < 1$. This constraints the step-size $\rho > 0$ to satisfy certain conditions. The permissible step-sizes satisfying the above constraint are discussed in \Cref{sec:adm-rho}, along with the time complexity of SR-VN in \Cref{eq:time_complexity}.
\end{theorem}

% With time plots
% ************************
\newcommand{\rulesep}{\unskip\ \vrule\ }
% ---------------------------------------------
\begin{figure}[t!]
   \centering
\begin{tabular}{ccc}
\includegraphics[height=0.22\linewidth]{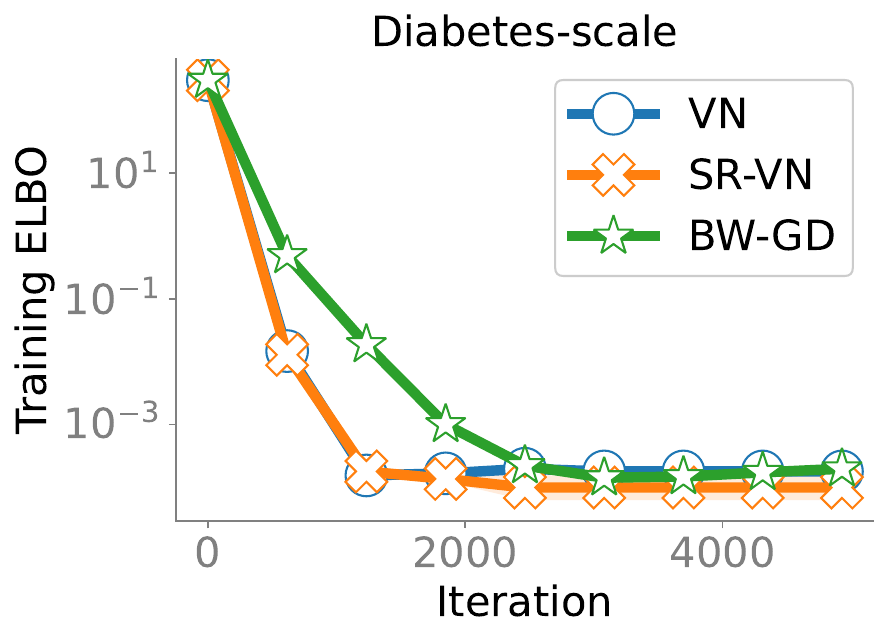}&
\includegraphics[width=0.3\linewidth]{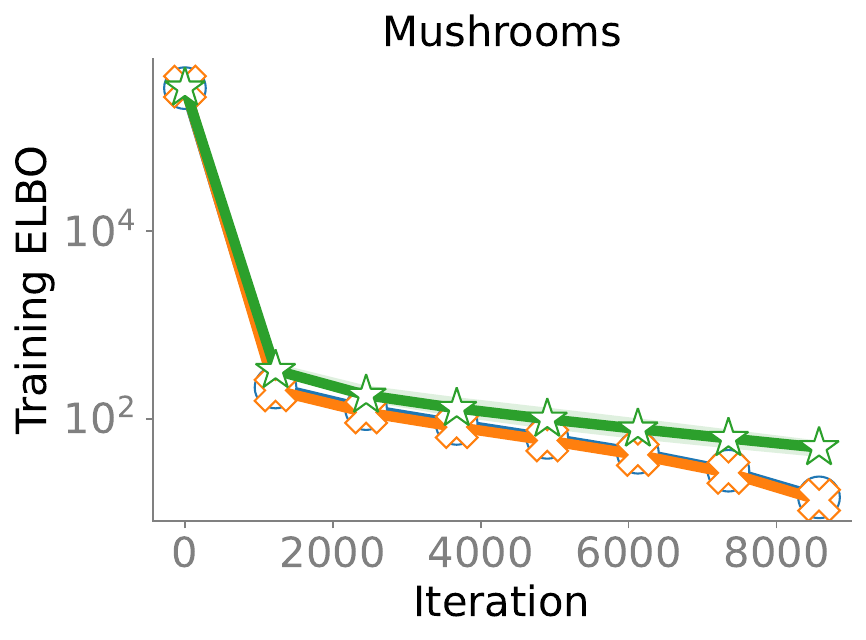}&
\includegraphics[width=0.3\linewidth]{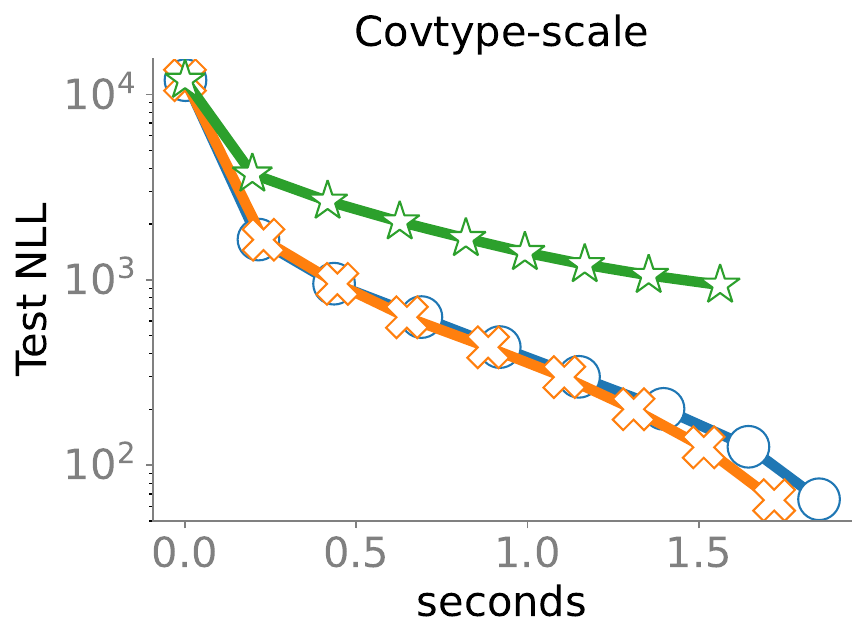}\\
\includegraphics[width=0.3\linewidth]{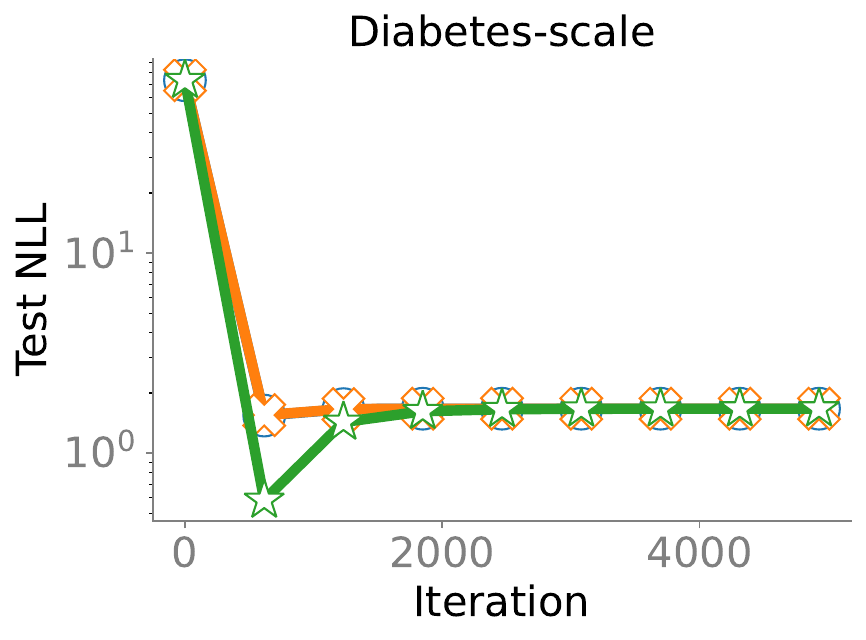}&
\includegraphics[width=0.3\linewidth]{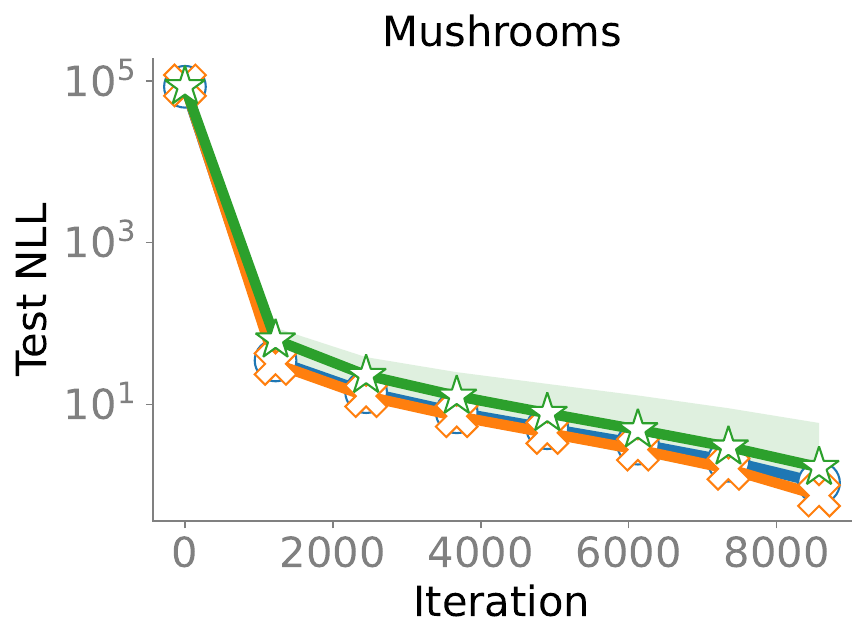}&
\includegraphics[width=0.3\linewidth]{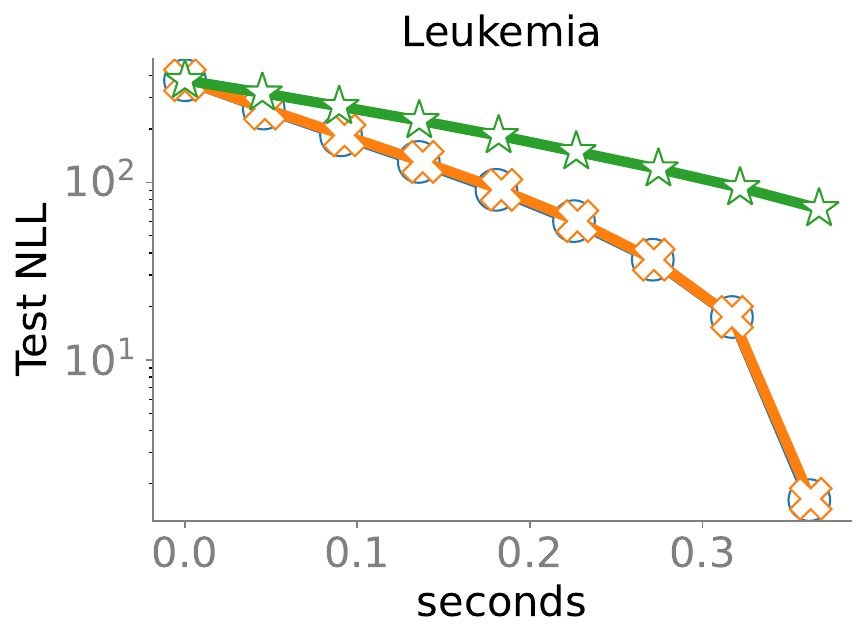}
\end{tabular}

    \caption{\small{For small-scale LIBSVM datasets, we show the training ELBO and test NLL w.r.t number of iterations (left and middle panel) and time (right panel). The $\min$ and $\max$ of the plotted values are displayed around their averages (taken over five random initializations). We also subtracted the $\min$ values achieved (for both training and testing) over all iterations from all three methods and only plotted the resulting values. We can see that, in most cases, SR-VN and VN perform similarly while BW-GD tends to be slower. However, sometimes SR-VN could also be slower than the two, for example, see \Cref{app:add_plots}.}}
    \label{fig:small_scale} % I can do without the label too
\end{figure}
% ---------------------------------------------

% ---------------------------------------------
\begin{figure}[t!]
   \centering
\begin{tabular}{ccc}
\includegraphics[width=0.3\linewidth]{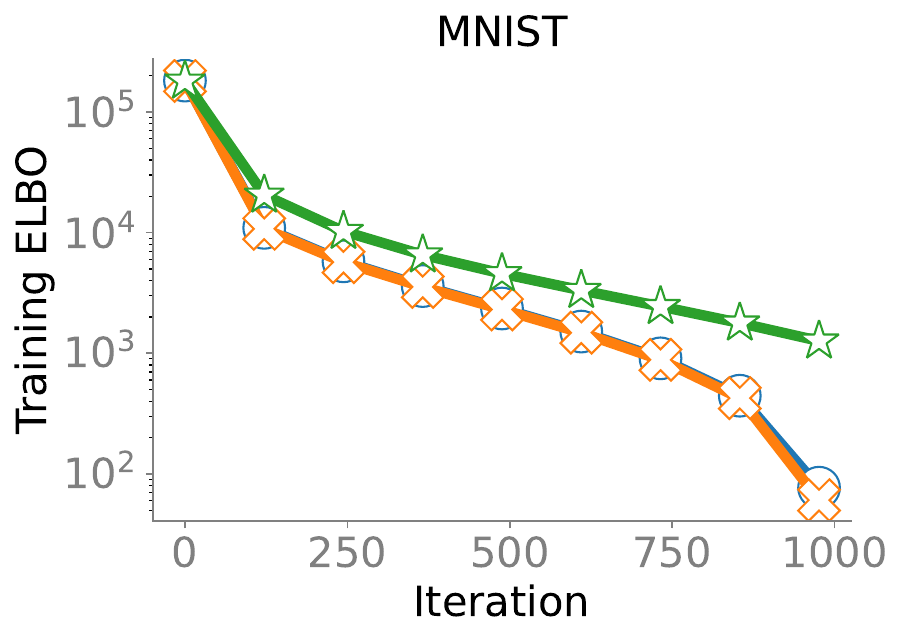}&
\includegraphics[width=0.3\linewidth]{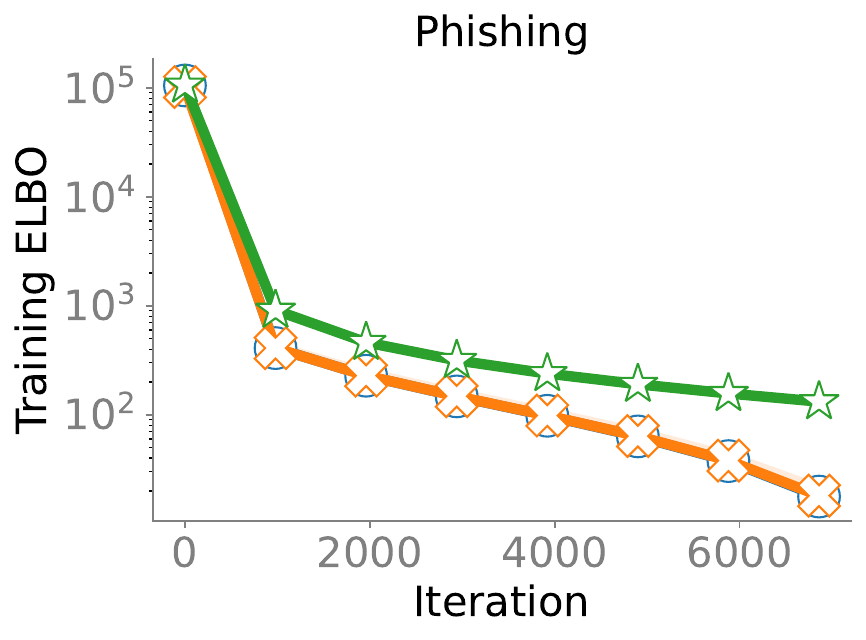}&
\includegraphics[width=0.3\linewidth]{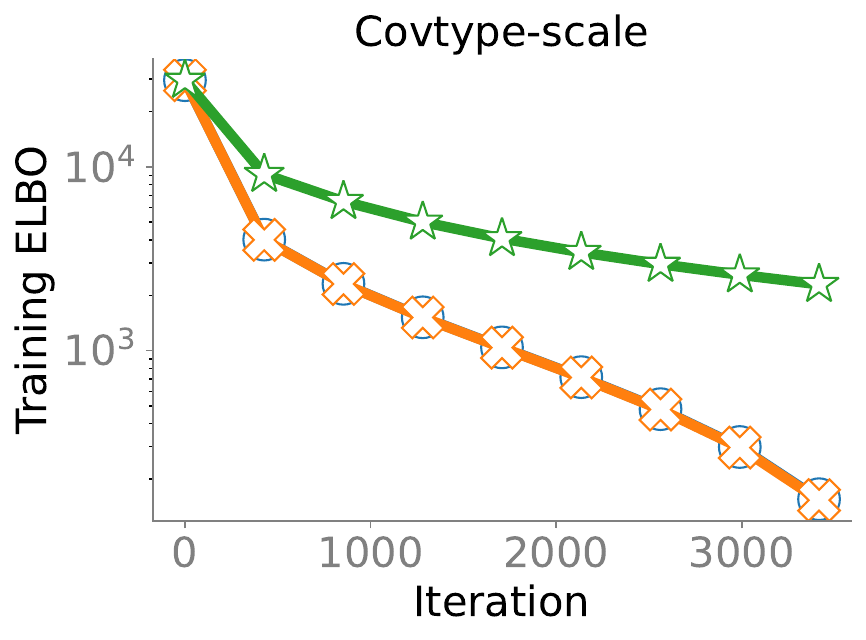}\\
\includegraphics[width=0.3\linewidth]{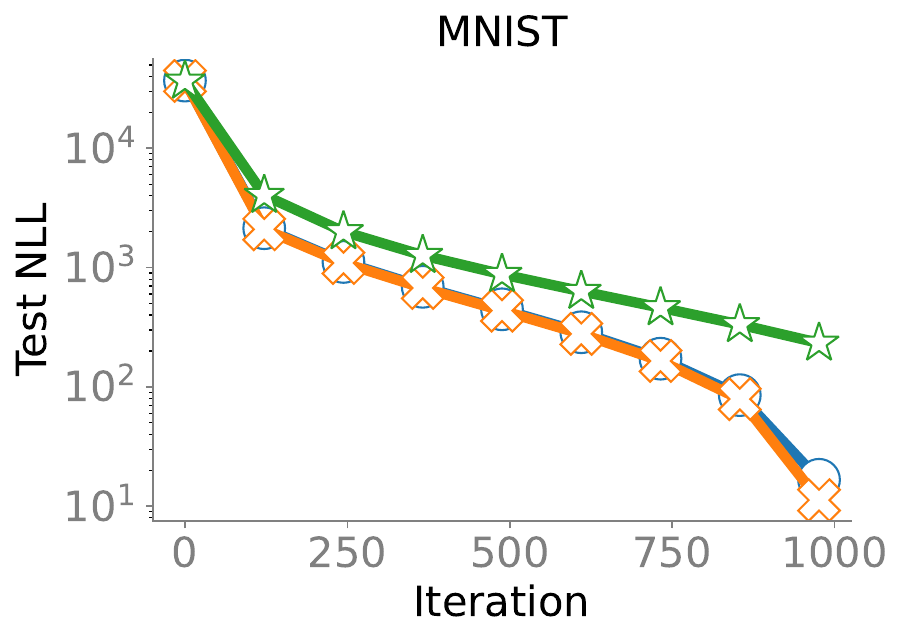}&
\includegraphics[width=0.3\linewidth]{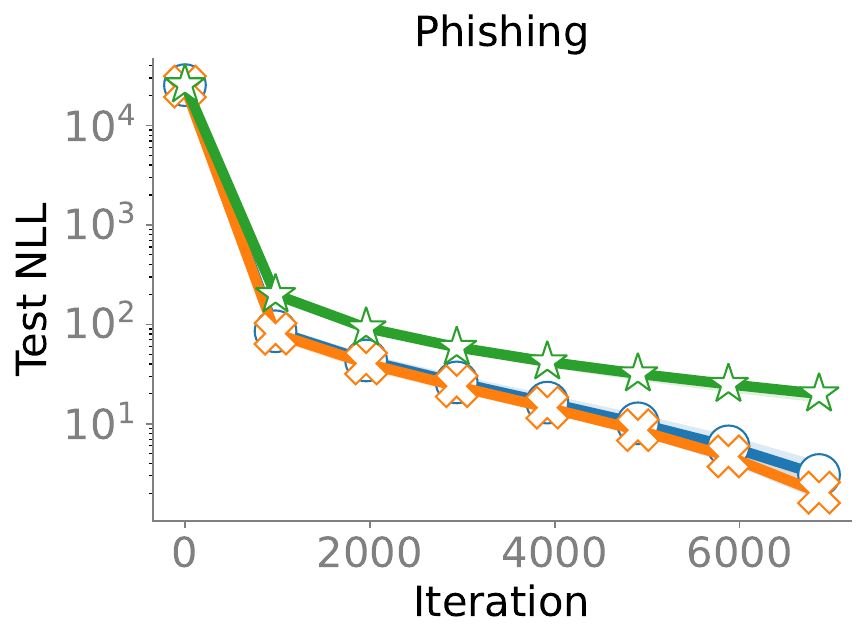}&
\includegraphics[width=0.3\linewidth]{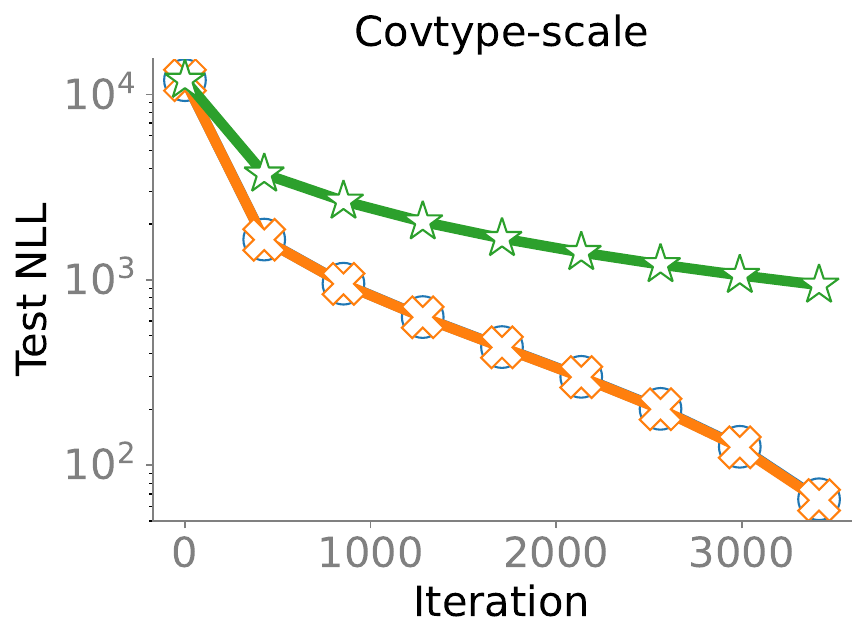}
\end{tabular}

\caption{\small{Comparison on Large-Scale Datasets. The same trends hold at scale: both SR-VN and VN exhibit comparable performance, whereas BW-GD tends to be slower.}}
\label{fig:large_scale}
\end{figure}
% ---------------------------------------------
% ************************

% -----------------------------------------------------------------
\section{Experiments} \label{sec:exp}

% ************************
In this section, we present results on the Bayesian logistic regression problem introduced in \ref{eq:logistic-loss} to optimize \ref{eq:general-elbo}, where we compare SR-VN to two other methods, namely 1) BW-GD  \citet[Algorithm 1]{lambert2022variational} and 2) Variational Newton in \Cref{eq:blr-updates}. The BW-GD updates are given as:
\begin{gather*}
    \text{BW-GD:} \quad \bmm_{t+1} = \bmm_t -\alpha~ \bbE_{\cN(\bmm_t, \bV_t)}[\nabla_{\bt} \Bar \ell(\bt)], \\
    \bM_t = \Id - \alpha \left(\bbE_{\cN(\bmm_t, \bV_t)}[\nabla^2_{\bt} \Bar \ell(\bt)] - \bS_t\right), \quad
    \bV_{t+1} = \bM_t\bV_t\bM_t\label{eq:BW-GD}
    % \bV_{t+1} &= \text{clip}^{1/\delta} ~\bV_t_{+},
\end{gather*}where $\alpha>0$ is the step-size. Due to space constraints, details of all the datasets used and the setting of various algorithmic parameters for these methods are given in Appendix \ref{app:exp}. The additional results and experiments are provided in Appendix \ref{app:algo-details}.

 % and for $\epsilon > 0$, the eigenvalue clipping operation is defined as $\text{clip}^{\epsilon}: \quad \bV = \sum_{i=1}^d \lambda_i \bu_i\bu_i^\top \mapsto \text{clip}^{\epsilon}\bV = \sum_{i=1}^d (\lambda_i \wedge \epsilon)~\bu_i\bu_i^\top$
\subsection{Bayesian logistic regression}

We run our experiments on both small-scale and large-scale datasets to compare the aforementioned methods. The way we measure the performance of methods is via two metrics: the negative ELBO on the training dataset and the NLL on the test dataset. We report the average of the NLL over all test points, i.e. we compare the NLL on the test set computed as follows: $-\sum_{i=1}^{n_{\text{test}}} \log(1+\exp\{-y_i(\Hat\bt^\top\bx_i)\}/n_{\text{test}}$ where $\Hat\bt$ is parameter estimate and $n_{\text{test}}$ is the number of examples in the test set.

\textbf{Datasets.} We consider eight different LIBSVM datasets\footnote{Available at \url{https://www.csie.ntu.edu.tw/~cjlin/libsvmtools/datasets/}} \citep{chang2011libsvm}, consisting of five small and three large-scale datasets. The description of these datasets is provided in \Cref{tab:dataset-hyperparameters} of \Cref{app:exp}. Here, we show results for two small-scale datasets (see \Cref{fig:small_scale}), namely Diabetes-scale ($n = 768, d = 8, n_{\text{train}} = 614$) and Mushrooms ($n = 8124, d = 112, n_{\text{train}} = 64,99$). For large-scale datasets (see \Cref{fig:large_scale}), we show MNIST ($n = 70,000, d = 784, n_{\text{train}} = 60,000$), Covtype-scale ($n = 581,012, d = 54, n_{\text{train}} = 500,000$), and Phishing ($n = 11,055, d = 68, n_{\text{train}} = 8,844$) datasets. All these datasets fall in the case where $n > d$. For the $d > n$ case, we choose the Leukemia dataset ($n = 38, d = 7,129, n_{\text{train}} = 34$) and show time plots to be inclusive of all possible different settings. Note that different datasets use binary labels other than $\{-1, +1\}$, we mapped them all to the same labels so that loss evaluation is uniform across all datasets. All experiments are performed on NVIDIA GeForce RTX 3090 GPUs.

We can observe that in the majority of instances, SR-VN performs similarly to VN and both of them are always better than BW-GD. Most often, VN is slightly better than SR-VN, indicating its superiority, and motivating its analysis as a future work. As a side note, one can observe that the wall-clock time shoots up for SR-VN in the rightmost panel of \Cref{fig:small_scale}. The reason for this lies in the difference between the computational complexity of algorithms presented, this is discussed in more detail in \Cref{sec:runtime_comp}.

% -----------------------------------------------------------------

\section{Related Works}
\paragraph{NGVI.} NGVI approximates intractable posterior distributions by employing natural gradients, which consider the geometry of the parameter space, instead of standard gradients to update the variational parameters \citep{hoffman2013stochastic}. This approach is believed to lead to more robust and efficient optimization \citep{amari1998natural}, especially when dealing with complex models and high-dimensional data \citep{Yuesong}. Deterministic NGVI \citep{honkela2008natural, honkela2010approximate, godichon2024natural}, in contrast to its stochastic counterpart, computes the exact natural gradient at each iteration. This can be computationally more demanding but can also lead to more accurate and stable updates. When feasible, the benefit of this approach is that one can use a constant step-size.

\paragraph{Convergence Analysis.} While NGVI has shown promising empirical performance$-$for example, in latent Dirichlet
allocation topic models \citep{hoffman2013stochastic}, Bayesian
neural networks\citep{khan2018fast, khan2018fast1, osawa2019practical, Yuesong}, probabilistic graphical models \citep{johnson2016composing}, and
large-scale Gaussian processes \citep{hensman2013gaussian, hensman2015scalable, salimbeni2018natural}$-$the theoretical analysis of its convergence rates is still an active area of research \citep{theis2015trust}. Existing works often focus on specific model classes or makes simplifying assumptions to derive convergence guarantees. For instance, \citep{wu2024understanding} showed that stochastic NGVI, exhibits a non-asymptotic convergence rate of $\mathcal{O}(1/T)$ for conjugate models, where $T$ represents the number of iterations. This rate is comparable to SGD, but stochastic NGVI often has has better constants in the convergence rate, leading to faster convergence in practice \citep{khan2018fast}. However, for non-conjugate models, stochastic NGVI with canonical parameterization operates on a non-convex objective, making it challenging to establish a global convergence rate.  In the context of Gaussian process models, \citep{tang2019variational} demonstrated the effectiveness of natural gradients in both conjugate and non-conjugate scenarios, showing potential for faster convergence compared to traditional gradient-based methods.  

Despite these advancements, a comprehensive understanding of the convergence behavior of NGVI in general settings remains an open question \citep{regier2017fast}.  This is partly due to the difficulty in analyzing NGD for commonly-used parameterizations that can destroy the concavity of the log-likelihood \citep{cherief2020convergence}.  Our work addressed this issue by considering a square-root parameterization for variational Gaussian inference, which helps maintain the concavity of the log-likelihood and facilitates the convergence analysis. We proved an exponential convergence rate for both NG flow and NGD under this parameterization, bridging the gap between theory and practice for NGD convergence guarantees. Similar efforts to analyze NGD convergence have been made by \cite{khan2015faster} for general stochastic, non-convex settings and \citep{chen2023gradient} for continuous-time flow with concave log-likelihoods. However, these studies either have slow convergence rates or focus on restrictive settings, leaving room for improvement in more general cases.

\paragraph{Parameterization Choices.} The choice of parameterization plays a crucial role in the effectiveness of NGVI \citep{lin2019fast}. Different parameterizations can lead to different geometries in the parameter space, affecting the trajectory of the optimization process. Studies have explored various parameterizations for NGVI, including natural and expectation parameters \citep{honkela2010approximate, khan2023bayesian, godichon2024natural}.  \cite{cai2024batch} proposed an alternative approach based on a score-based divergence that can be optimized by a closed-form proximal update for Gaussian variational families with full covariance matrices. They proved that in the limit of infinite batch size, the variational parameter updates converge exponentially quickly to the target mean and covariance. Our work focused on square-root parameterizations for variational Gaussian inference, extending the idea introduced by \cite{tan2021analytic} who utilized the Cholesky factor. This parameterization helps maintain the concavity of log-likelihood functions, which in turn facilitates convergence guarantees for both NG flow and NGD. While this approach may lead to suboptimal convergence in some cases, it provides a valuable framework for analyzing and understanding the convergence behavior of deterministic NGVI.  

% Other works have explored alternative parameterizations for improving VI, such as the work by \citep{Salimans_2013} who proposed a method for efficiently computing natural gradient updates in a fixed-form VI setting.

% -----------------------------------------------------------------
\section{Discussion and Conclusion} \label{sec:conc}

In this paper, we studied the performance of NG flow and NGD for optimizing the ELBO, a popular objective function in VI. Our findings suggest that the effectiveness of these methods can be affected by the choice of parameterization and the underlying geometry of the optimization landscape; however, a definitive link has yet to be established. Motivated by the need for a principled convergence analysis, we focused on a square-root parameterization for variational Gaussian inference. This approach, building upon the work of \cite{domke2019provable}, has the promise of preserving the concavity of the log-likelihood, a property often lost in traditional parameterizations. By preserving concavity, we were able to establish theoretical guarantees for the convergence of both NG flow and NGD under this parameterization, thus bridging a gap between theoretical analysis and practical implementation of NGD.

Our empirical evaluations further validated the benefits of the square-root parameterization. We observed that it generally performs comparably to the standard natural parameterization, exhibiting fast convergence due to its Newton-like update. Moreover, our experiments, particularly the toy (ill-conditioned) problem in \Cref{fig:convergence-comparison}, highlighted the role of the square-root parameterization in maintaining stable updates, especially in scenarios where traditional parameterizations might be sensitive to parameter choices or errors in the Fisher information matrix estimation.

Despite the advancements, challenges and open questions remain. Our findings pave the way for future research in several key directions, this includes extending our analysis to stochastic NGVI$-$common in real-world settings$-$to study how noise influences optimization and whether the square-root parameterization remains effective, along with the validity of key assumptions (e.g., Assumptions \ref{assum:strong-convexity}-\ref{assum:bounded-iterates}) for more complex applications like deep learning. 

% \begin{itemize}
    % \item \textbf{Extension to Stochastic NGVI:} While our primary focus was on deterministic NGVI, it is important to extend our analysis to the stochastic setting, which is more common in real-world applications. This means looking at how noise affects the optimization process and whether our findings about the square-root parameterization still hold when things are less predictable.

    % We also need to consider whether our assumptions, which were helpful for our analysis, are still realistic in more complex situations, like those found in deep learning. Some of these assumptions, for example, Assumption \ref{assum:strong-convexity}-\ref{assum:bounded-iterates}, might be too strict and could prevent our results from being applicable to a wider range of problems.

    % Future work could focus on analyzing convergence rates and stability in these more general settings, with fewer restrictions on the types of problems we can handle. This will help us understand how well stochastic NGVI works in practice and how to make it even better.

    % \item \textbf{Generalization to Other Models:} Our analysis focused on variational Gaussian inference. A natural extension is to generalize our findings to other variational families and model classes. This includes investigating the applicability of square-root parameterizations to other distributions and exploring the convergence properties of NGVI in more complex models.

% \end{itemize}

\subsubsection*{Broader Impact Statement}

Our work makes several contributions: i) we propose a theoretical analysis using a square-root parameterization for the Gaussian covariance, which we expect can encourage further research into establishing convergence guarantees for natural-gradient variational inference, particularly for cases involving concave log-likelihoods. ii) Our experimental results reveal that natural gradient methods outperform algorithms relying on Euclidean or Wasserstein geometries, showcasing their practical advantages in terms of convergence speed and effectiveness, which could impact the development of more robust inference algorithms in probabilistic modeling.

\subsubsection*{Acknowledgments}
Navish Kumar and Aurelien Lucchi acknowledge the financial support of the Swiss National Foundation, SNF grant No 207392.

% % ---------------------------------------------
% \begin{figure}[htb!]
%    \centering
% \begin{tabular}{cc}
% \includegraphics[width=0.88\linewidth]{tmlr-style-file-main/Images/breast-cancer_scale.pdf}
% \end{tabular}
% \caption{\small{Bayesian logistic regression with a non-convex regularizer.}}
% \label{fig:non-convex}
% \end{figure}
% ---------------------------------------------

% \section{Rebuttal}

\bibliography{references}

\begin{thebibliography}{77}
\providecommand{\natexlab}[1]{#1}
\providecommand{\url}[1]{\texttt{#1}}
\expandafter\ifx\csname urlstyle\endcsname\relax
  \providecommand{\doi}[1]{doi: #1}\else
  \providecommand{\doi}{doi: \begingroup \urlstyle{rm}\Url}\fi

\bibitem[Adam et~al.(2021)Adam, Chang, Khan, and Solin]{adam2021dual}
Vincent Adam, Paul Chang, Mohammad Emtiyaz~E Khan, and Arno Solin.
\newblock Dual parameterization of sparse variational gaussian processes.
\newblock \emph{Advances in Neural Information Processing Systems}, 34:\penalty0 11474--11486, 2021.

\bibitem[Alquier et~al.(2016)Alquier, Ridgway, and Chopin]{AlRi16}
Pierre Alquier, James Ridgway, and Nicolas Chopin.
\newblock On the properties of variational approximations of {G}ibbs posteriors.
\newblock \emph{J. Mach. Learn. Res. (JMLR)}, 17:\penalty0 239:1--239:41, 2016.

\bibitem[Amari(1998)]{amari1998natural}
Shun-Ichi Amari.
\newblock Natural gradient works efficiently in learning.
\newblock \emph{Neural computation}, 10\penalty0 (2):\penalty0 251--276, 1998.

\bibitem[Amari(2008)]{amari2008information}
Shun-Ichi Amari.
\newblock Information geometry and its applications: Convex function and dually flat manifold.
\newblock In \emph{LIX Fall Colloquium on Emerging Trends in Visual Computing}, pp.\  75--102. Springer, 2008.

\bibitem[Amari(2016)]{amari2016information}
Shun-ichi Amari.
\newblock \emph{Information geometry and its applications}, volume 194.
\newblock Springer, 2016.

\bibitem[Baba(1981)]{baba1981convergence}
N~Baba.
\newblock Convergence of a random optimization method for constrained optimization problems.
\newblock \emph{Journal of Optimization Theory and Applications}, 33:\penalty0 451--461, 1981.

\bibitem[Beyer(2001)]{beyer2001theory}
Hans-Georg Beyer.
\newblock \emph{The theory of evolution strategies}.
\newblock Springer Science \& Business Media, 2001.

\bibitem[Blei et~al.(2017)Blei, Kucukelbir, and McAuliffe]{blei2017variational}
David~M Blei, Alp Kucukelbir, and Jon~D McAuliffe.
\newblock Variational inference: A review for statisticians.
\newblock \emph{Journal of the American statistical Association}, 112\penalty0 (518):\penalty0 859--877, 2017.

\bibitem[Blundell et~al.(2015)Blundell, Cornebise, Kavukcuoglu, and Wierstra]{BlCo15}
Charles Blundell, Julien Cornebise, Koray Kavukcuoglu, and Daan Wierstra.
\newblock Weight uncertainty in neural networks.
\newblock In \emph{International Conference on Machine Learning (ICML)}, 2015.

\bibitem[Cai et~al.(2024)Cai, Modi, Pillaud-Vivien, Margossian, Gower, Blei, and Saul]{cai2024batch}
Diana Cai, Chirag Modi, Loucas Pillaud-Vivien, Charles~C Margossian, Robert~M Gower, David~M Blei, and Lawrence~K Saul.
\newblock Batch and match: black-box variational inference with a score-based divergence.
\newblock \emph{arXiv preprint arXiv:2402.14758}, 2024.

\bibitem[Catoni(2007)]{catoni2007pac}
O~Catoni.
\newblock Pac-bayesian supervised classification: The thermodynamics of statistical learning. institute of mathematical statistics lecture notes—monograph series 56.
\newblock \emph{IMS, Beachwood, OH. MR2483528}, 5544465, 2007.

\bibitem[Challis \& Barber(2013)Challis and Barber]{challis2013gaussian}
Edward Challis and David Barber.
\newblock Gaussian kullback-leibler approximate inference.
\newblock \emph{Journal of Machine Learning Research}, 14\penalty0 (8), 2013.

\bibitem[Chang \& Lin(2011)Chang and Lin]{chang2011libsvm}
Chih-Chung Chang and Chih-Jen Lin.
\newblock {LIBSVM}: a library for support vector machines.
\newblock \emph{ACM transactions on intelligent systems and technology (TIST)}, 2\penalty0 (3):\penalty0 1--27, 2011.

\bibitem[Chen et~al.(2023)Chen, Huang, Huang, Reich, and Stuart]{chen2023gradient}
Yifan Chen, Daniel~Zhengyu Huang, Jiaoyang Huang, Sebastian Reich, and Andrew~M Stuart.
\newblock Gradient flows for sampling: Mean-field models, gaussian approximations and affine invariance.
\newblock \emph{arXiv preprint arXiv:2302.11024}, 2023.

\bibitem[Ch{\'e}rief-Abdellatif(2020)]{cherief2020convergence}
Badr-Eddine Ch{\'e}rief-Abdellatif.
\newblock Convergence rates of variational inference in sparse deep learning.
\newblock In \emph{International Conference on Machine Learning}, pp.\  1831--1842. PMLR, 2020.

\bibitem[Ch{\'e}rief-Abdellatif et~al.(2019)Ch{\'e}rief-Abdellatif, Alquier, and Khan]{cherief2019generalization}
Badr-Eddine Ch{\'e}rief-Abdellatif, Pierre Alquier, and Mohammad~Emtiyaz Khan.
\newblock A generalization bound for online variational inference.
\newblock In \emph{Asian Conference on Machine Learning (ACML)}, 2019.

\bibitem[Dattorro(2015)]{Dattorro2015}
Jon Dattorro.
\newblock Convex optimization and euclidean distance geometry.
\newblock Book, 2015.
\newblock \url{https://ccrma.stanford.edu/~dattorro/0976401304_v2015.04.11.pdf}.

\bibitem[Diao et~al.(2023)Diao, Balasubramanian, Chewi, and Salim]{diao2023forward}
Michael~Ziyang Diao, Krishna Balasubramanian, Sinho Chewi, and Adil Salim.
\newblock Forward-backward gaussian variational inference via jko in the bures-wasserstein space.
\newblock In \emph{International Conference on Machine Learning}, pp.\  7960--7991. PMLR, 2023.

\bibitem[Domke(2019)]{domke2019provable}
Justin Domke.
\newblock Provable gradient variance guarantees for black-box variational inference.
\newblock \emph{Advances in Neural Information Processing Systems}, 32, 2019.

\bibitem[Domke(2020)]{Do20}
Justin Domke.
\newblock Provable smoothness guarantees for black-box variational inference.
\newblock In \emph{International Conference on Machine Learning (ICML)}, 2020.

\bibitem[Domke et~al.(2024)Domke, Gower, and Garrigos]{domke2023provable}
Justin Domke, Robert Gower, and Guillaume Garrigos.
\newblock Provable convergence guarantees for black-box variational inference.
\newblock \emph{Advances in neural information processing systems}, 36, 2024.

\bibitem[Glasmachers et~al.(2010)Glasmachers, Schaul, Yi, Wierstra, and Schmidhuber]{glasmachers2010exponential}
Tobias Glasmachers, Tom Schaul, Sun Yi, Daan Wierstra, and J{\"u}rgen Schmidhuber.
\newblock Exponential natural evolution strategies.
\newblock In \emph{Proceedings of the 12th annual conference on Genetic and evolutionary computation}, pp.\  393--400, 2010.

\bibitem[Godichon-Baggioni et~al.(2024)Godichon-Baggioni, Nguyen, and Tran]{godichon2024natural}
A~Godichon-Baggioni, D~Nguyen, and M-N Tran.
\newblock Natural gradient variational bayes without fisher matrix analytic calculation and its inversion.
\newblock \emph{Journal of the American Statistical Association}, pp.\  1--12, 2024.

\bibitem[Graves(2011)]{Gr11}
Alex Graves.
\newblock Practical variational inference for neural networks.
\newblock In \emph{Advances in Neural Information Processing Systems (NeurIPS)}, 2011.

\bibitem[Hazan et~al.(2016)Hazan, Levy, and Shalev-Shwartz]{hazan2016graduated}
Elad Hazan, Kfir~Yehuda Levy, and Shai Shalev-Shwartz.
\newblock On graduated optimization for stochastic non-convex problems.
\newblock In \emph{International conference on machine learning}, pp.\  1833--1841. PMLR, 2016.

\bibitem[Hensman et~al.(2012)Hensman, Rattray, and Lawrence]{hensman2012fast}
James Hensman, Magnus Rattray, and Neil Lawrence.
\newblock Fast variational inference in the conjugate exponential family.
\newblock \emph{Advances in neural information processing systems}, 25, 2012.

\bibitem[Hensman et~al.(2013)Hensman, Fusi, and Lawrence]{hensman2013gaussian}
James Hensman, Nicolo Fusi, and Neil~D Lawrence.
\newblock Gaussian processes for big data.
\newblock \emph{arXiv preprint arXiv:1309.6835}, 2013.

\bibitem[Hensman et~al.(2015)Hensman, Matthews, and Ghahramani]{hensman2015scalable}
James Hensman, Alexander Matthews, and Zoubin Ghahramani.
\newblock Scalable variational gaussian process classification.
\newblock In \emph{Artificial Intelligence and Statistics}, pp.\  351--360. PMLR, 2015.

\bibitem[Hoffman et~al.(2013)Hoffman, Blei, Wang, and Paisley]{hoffman2013stochastic}
Matthew~D Hoffman, David~M Blei, Chong Wang, and John Paisley.
\newblock Stochastic variational inference.
\newblock \emph{Journal of Machine Learning Research}, 2013.

\bibitem[Honkela et~al.(2008)Honkela, Tornio, Raiko, and Karhunen]{honkela2008natural}
Antti Honkela, Matti Tornio, Tapani Raiko, and Juha Karhunen.
\newblock Natural conjugate gradient in variational inference.
\newblock In \emph{Neural Information Processing: 14th International Conference, ICONIP 2007, Kitakyushu, Japan, November 13-16, 2007, Revised Selected Papers, Part II 14}, pp.\  305--314. Springer, 2008.

\bibitem[Honkela et~al.(2010)Honkela, Raiko, Kuusela, Tornio, and Karhunen]{honkela2010approximate}
Antti Honkela, Tapani Raiko, Mikael Kuusela, Matti Tornio, and Juha Karhunen.
\newblock Approximate riemannian conjugate gradient learning for fixed-form variational bayes.
\newblock \emph{The Journal of Machine Learning Research}, 11:\penalty0 3235--3268, 2010.

\bibitem[Jaakkola \& Jordan(1997)Jaakkola and Jordan]{Jaakkola96b}
Tommi~S Jaakkola and Michael~I Jordan.
\newblock A variational approach to bayesian logistic regression models and their extensions.
\newblock In \emph{International Conference on Artificial Intelligence and Statistics (AISTATS)}, 1997.

\bibitem[Johnson et~al.(2016)Johnson, Duvenaud, Wiltschko, Adams, and Datta]{johnson2016composing}
Matthew~J Johnson, David~K Duvenaud, Alex Wiltschko, Ryan~P Adams, and Sandeep~R Datta.
\newblock Composing graphical models with neural networks for structured representations and fast inference.
\newblock \emph{Advances in neural information processing systems}, 29, 2016.

\bibitem[Khan \& Lin(2017)Khan and Lin]{khan2017conjugate}
Mohammad Khan and Wu~Lin.
\newblock Conjugate-computation variational inference: Converting variational inference in non-conjugate models to inferences in conjugate models.
\newblock In \emph{Artificial Intelligence and Statistics}, pp.\  878--887. PMLR, 2017.

\bibitem[Khan et~al.(2018)Khan, Nielsen, Tangkaratt, Lin, Gal, and Srivastava]{khan2018fast1}
Mohammad Khan, Didrik Nielsen, Voot Tangkaratt, Wu~Lin, Yarin Gal, and Akash Srivastava.
\newblock Fast and scalable bayesian deep learning by weight-perturbation in adam.
\newblock In \emph{International conference on machine learning}, pp.\  2611--2620. PMLR, 2018.

\bibitem[Khan \& Nielsen(2018)Khan and Nielsen]{khan2018fast}
Mohammad~Emtiyaz Khan and Didrik Nielsen.
\newblock Fast yet simple natural-gradient descent for variational inference in complex models.
\newblock In \emph{2018 International Symposium on Information Theory and Its Applications (ISITA)}, pp.\  31--35. IEEE, 2018.

\bibitem[Khan \& Rue(2023)Khan and Rue]{khan2023bayesian}
Mohammad~Emtiyaz Khan and H{\aa}vard Rue.
\newblock The bayesian learning rule.
\newblock \emph{Journal of Machine Learning Research}, 24\penalty0 (281):\penalty0 1--46, 2023.

\bibitem[Khan et~al.(2016)Khan, Babanezhad, Lin, Schmidt, and Sugiyama]{khan2015faster}
Mohammad~Emtiyaz Khan, Reza Babanezhad, Wu~Lin, Mark Schmidt, and Masashi Sugiyama.
\newblock Faster stochastic variational inference using proximal-gradient methods with general divergence functions.
\newblock In \emph{UAI'16}, UAI'16, pp.\  319–328, Arlington, Virginia, USA, 2016. AUAI Press.

\bibitem[Khan et~al.(2017)Khan, Lin, Tangkaratt, Liu, and Nielsen]{khan2017variational}
Mohammad~Emtiyaz Khan, Wu~Lin, Voot Tangkaratt, Zuozhu Liu, and Didrik Nielsen.
\newblock Variational adaptive-newton method for explorative learning.
\newblock \emph{arXiv preprint arXiv:1711.05560}, 2017.

\bibitem[Kim et~al.(2024)Kim, Oh, Wu, Ma, and Gardner]{kim2023black}
Kyurae Kim, Jisu Oh, Kaiwen Wu, Yian Ma, and Jacob Gardner.
\newblock On the convergence of black-box variational inference.
\newblock \emph{Advances in Neural Information Processing Systems}, 36, 2024.

\bibitem[Kingma \& Welling(2013)Kingma and Welling]{kingma2013auto}
Diederik~P Kingma and Max Welling.
\newblock Auto-encoding variational bayes.
\newblock \emph{arXiv preprint arXiv:1312.6114}, 2013.

\bibitem[Knowles \& Minka(2011)Knowles and Minka]{knowles2011non}
David Knowles and Tom Minka.
\newblock Non-conjugate variational message passing for multinomial and binary regression.
\newblock \emph{Advances in Neural Information Processing Systems}, 24, 2011.

\bibitem[Lambert et~al.(2022{\natexlab{a}})Lambert, Bonnabel, and Bach]{lambert2022recursive}
Marc Lambert, Silvere Bonnabel, and Francis Bach.
\newblock The recursive variational gaussian approximation (r-vga).
\newblock \emph{Statistics and Computing}, 32\penalty0 (1):\penalty0 10, 2022{\natexlab{a}}.

\bibitem[Lambert et~al.(2022{\natexlab{b}})Lambert, Chewi, Bach, Bonnabel, and Rigollet]{lambert2022variational}
Marc Lambert, Sinho Chewi, Francis Bach, Silv{\`e}re Bonnabel, and Philippe Rigollet.
\newblock Variational inference via wasserstein gradient flows.
\newblock \emph{Advances in Neural Information Processing Systems}, 35:\penalty0 14434--14447, 2022{\natexlab{b}}.

\bibitem[Leordeanu \& Hebert(2008)Leordeanu and Hebert]{leordeanu2008smoothing}
Marius Leordeanu and Martial Hebert.
\newblock Smoothing-based optimization.
\newblock In \emph{2008 IEEE Conference on Computer Vision and Pattern Recognition}, pp.\  1--8. IEEE, 2008.

\bibitem[Lin et~al.(2019)Lin, Khan, and Schmidt]{lin2019fast}
Wu~Lin, Mohammad~Emtiyaz Khan, and Mark Schmidt.
\newblock Fast and simple natural-gradient variational inference with mixture of exponential-family approximations.
\newblock In \emph{International Conference on Machine Learning}, pp.\  3992--4002. PMLR, 2019.

\bibitem[Lin et~al.(2021{\natexlab{a}})Lin, Nielsen, Khan, and Schmidt]{lin2021structured}
Wu~Lin, Frank Nielsen, Mohammad~Emtiyaz Khan, and Mark Schmidt.
\newblock Structured second-order methods via natural gradient descent.
\newblock \emph{arXiv preprint arXiv:2107.10884}, 2021{\natexlab{a}}.

\bibitem[Lin et~al.(2021{\natexlab{b}})Lin, Nielsen, Khan, and Schmidt]{lin2021tractable}
Wu~Lin, Frank Nielsen, Mohammad~Emtiyaz Khan, and Mark Schmidt.
\newblock Tractable structured natural-gradient descent using local parameterizations.
\newblock In \emph{International Conference on Machine Learning}, pp.\  6680--6691. PMLR, 2021{\natexlab{b}}.

\bibitem[Lin et~al.(2023)Lin, Duruisseaux, Leok, Nielsen, Khan, and Schmidt]{lin23c}
Wu~Lin, Valentin Duruisseaux, Melvin Leok, Frank Nielsen, Mohammad~Emtiyaz Khan, and Mark Schmidt.
\newblock Simplifying momentum-based positive-definite submanifold optimization with applications to deep learning.
\newblock In Andreas Krause, Emma Brunskill, Kyunghyun Cho, Barbara Engelhardt, Sivan Sabato, and Jonathan Scarlett (eds.), \emph{Proceedings of the 40th International Conference on Machine Learning}, volume 202 of \emph{Proceedings of Machine Learning Research}, pp.\  21026--21050. PMLR, 23--29 Jul 2023.
\newblock URL \url{https://proceedings.mlr.press/v202/lin23c.html}.

\bibitem[Lin et~al.(2024)Lin, Dangel, Eschenhagen, Neklyudov, Kristiadi, Turner, and Makhzani]{lin24f}
Wu~Lin, Felix Dangel, Runa Eschenhagen, Kirill Neklyudov, Agustinus Kristiadi, Richard~E. Turner, and Alireza Makhzani.
\newblock Structured inverse-free natural gradient descent: Memory-efficient and numerically-stable {KFAC}.
\newblock In Ruslan Salakhutdinov, Zico Kolter, Katherine Heller, Adrian Weller, Nuria Oliver, Jonathan Scarlett, and Felix Berkenkamp (eds.), \emph{Proceedings of the 41st International Conference on Machine Learning}, volume 235 of \emph{Proceedings of Machine Learning Research}, pp.\  29974--29991. PMLR, 21--27 Jul 2024.
\newblock URL \url{https://proceedings.mlr.press/v235/lin24f.html}.

\bibitem[Liu et~al.(2024)Liu, Ghosal, Balasubramanian, and Pillai]{liu2024towards}
Tianle Liu, Promit Ghosal, Krishnakumar Balasubramanian, and Natesh Pillai.
\newblock Towards understanding the dynamics of gaussian-stein variational gradient descent.
\newblock \emph{Advances in Neural Information Processing Systems}, 36, 2024.

\bibitem[Malagò \& Pistone(2015)Malagò and Pistone]{MaPi15}
Luigi Malagò and Giovanni Pistone.
\newblock Information geometry of the {G}aussian distribution in view of stochastic optimization.
\newblock In \emph{Proceedings of the 2015 {ACM} Conference on Foundations of Genetic Algorithms}, pp.\  150--162, 2015.

\bibitem[Marlin et~al.(2011)Marlin, Khan, and Murphy]{marlin2011piecewise}
Benjamin~M Marlin, Mohammad~Emtiyaz Khan, and Kevin~P Murphy.
\newblock Piecewise bounds for estimating bernoulli-logistic latent gaussian models.
\newblock In \emph{ICML}, pp.\  633--640, 2011.

\bibitem[Mnih et~al.(2016)Mnih, Badia, Mirza, Graves, Lillicrap, Harley, Silver, and Kavukcuoglu]{mnih2016asynchronous}
Volodymyr Mnih, Adria~Puigdomenech Badia, Mehdi Mirza, Alex Graves, Timothy Lillicrap, Tim Harley, David Silver, and Koray Kavukcuoglu.
\newblock Asynchronous methods for deep reinforcement learning.
\newblock In \emph{International conference on machine learning}, pp.\  1928--1937. PMLR, 2016.

\bibitem[Mobahi \& Fisher~III(2015)Mobahi and Fisher~III]{mobahi2015theoretical}
Hossein Mobahi and John Fisher~III.
\newblock A theoretical analysis of optimization by gaussian continuation.
\newblock In \emph{Proceedings of the AAAI Conference on Artificial Intelligence}, volume~29, 2015.

\bibitem[Murray(2016)]{murray2016differentiation}
Iain Murray.
\newblock Differentiation of the cholesky decomposition.
\newblock \emph{arXiv preprint arXiv:1602.07527}, 2016.

\bibitem[Osawa et~al.(2019)Osawa, Swaroop, Khan, Jain, Eschenhagen, Turner, and Yokota]{osawa2019practical}
Kazuki Osawa, Siddharth Swaroop, Mohammad Emtiyaz~E Khan, Anirudh Jain, Runa Eschenhagen, Richard~E Turner, and Rio Yokota.
\newblock Practical deep learning with bayesian principles.
\newblock \emph{Advances in neural information processing systems}, 32, 2019.

\bibitem[Petersen \& Pedersen(2012)Petersen and Pedersen]{cookbook}
K.~B. Petersen and M.~S. Pedersen.
\newblock The matrix cookbook, nov 2012.
\newblock URL \url{http://www2.compute.dtu.dk/pubdb/pubs/3274-full.html}.
\newblock Version 20121115.

\bibitem[Ranganath et~al.(2014)Ranganath, Gerrish, and Blei]{ranganath2014black}
Rajesh Ranganath, Sean Gerrish, and David Blei.
\newblock Black box variational inference.
\newblock In \emph{Artificial intelligence and statistics}, pp.\  814--822. PMLR, 2014.

\bibitem[Regier et~al.(2017)Regier, Jordan, and McAuliffe]{regier2017fast}
Jeffrey Regier, Michael~I Jordan, and Jon McAuliffe.
\newblock Fast black-box variational inference through stochastic trust-region optimization.
\newblock \emph{Advances in Neural Information Processing Systems}, 30, 2017.

\bibitem[Salimans \& Knowles(2013)Salimans and Knowles]{Salimans_2013}
Tim Salimans and David~A. Knowles.
\newblock Fixed-form variational posterior approximation through stochastic linear regression.
\newblock \emph{Bayesian Analysis}, 8\penalty0 (4), December 2013.
\newblock ISSN 1936-0975.
\newblock \doi{10.1214/13-ba858}.
\newblock URL \url{http://dx.doi.org/10.1214/13-BA858}.

\bibitem[Salimbeni et~al.(2018)Salimbeni, Eleftheriadis, and Hensman]{salimbeni2018natural}
Hugh Salimbeni, Stefanos Eleftheriadis, and James Hensman.
\newblock Natural gradients in practice: Non-conjugate variational inference in gaussian process models.
\newblock In \emph{International Conference on Artificial Intelligence and Statistics}, pp.\  689--697. PMLR, 2018.

\bibitem[Shen et~al.(2024{\natexlab{a}})Shen, Daheim, Cong, Nickl, Marconi, Bazan, Yokota, Gurevych, Cremers, Khan, and M\"{o}llenhoff]{Yuesong}
Yuesong Shen, Nico Daheim, Bai Cong, Peter Nickl, Gian~Maria Marconi, Clement Bazan, Rio Yokota, Iryna Gurevych, Daniel Cremers, Mohammad~Emtiyaz Khan, and Thomas M\"{o}llenhoff.
\newblock Variational learning is effective for large deep networks.
\newblock In \emph{Proceedings of the 41st International Conference on Machine Learning}, ICML'24. JMLR.org, 2024{\natexlab{a}}.

\bibitem[Shen et~al.(2024{\natexlab{b}})Shen, Daheim, Cong, Nickl, Marconi, Bazan, Yokota, Gurevych, Cremers, Khan, et~al.]{shen2024variational}
Yuesong Shen, Nico Daheim, Bai Cong, Peter Nickl, Gian~Maria Marconi, Clement Bazan, Rio Yokota, Iryna Gurevych, Daniel Cremers, Mohammad~Emtiyaz Khan, et~al.
\newblock Variational learning is effective for large deep networks.
\newblock \emph{arXiv preprint arXiv:2402.17641}, 2024{\natexlab{b}}.

\bibitem[Spall(2005)]{spall2005introduction}
James~C Spall.
\newblock \emph{Introduction to stochastic search and optimization: estimation, simulation, and control}.
\newblock John Wiley \& Sons, 2005.

\bibitem[Staines \& Barber(2012)Staines and Barber]{staines2012variational}
Joe Staines and David Barber.
\newblock Variational optimization.
\newblock \emph{arXiv preprint arXiv:1212.4507}, 2012.

\bibitem[Sun et~al.(2009)Sun, Wierstra, Schaul, and Schmidhuber]{sun2009efficient}
Yi~Sun, Daan Wierstra, Tom Schaul, and J{\"u}rgen Schmidhuber.
\newblock Efficient natural evolution strategies.
\newblock In \emph{Proceedings of the 11th Annual conference on Genetic and evolutionary computation}, pp.\  539--546, 2009.

\bibitem[Sutton et~al.(1998)Sutton, Barto, et~al.]{sutton1998introduction}
Richard~S Sutton, Andrew~G Barto, et~al.
\newblock Introduction to reinforcement learning. vol. 135, 1998.

\bibitem[Tan(2025)]{tan2021analytic}
Linda~SL Tan.
\newblock Analytic natural gradient updates for cholesky factor in gaussian variational approximation.
\newblock \emph{Journal of the Royal Statistical Society Series B: Statistical Methodology}, pp.\  qkaf001, 2025.

\bibitem[Tang \& Ranganath(2019)Tang and Ranganath]{tang2019variational}
Da~Tang and Rajesh Ranganath.
\newblock The variational predictive natural gradient.
\newblock In \emph{International Conference on Machine Learning}, pp.\  6145--6154. PMLR, 2019.

\bibitem[Teboulle(1992)]{teboulle1992entropic}
Marc Teboulle.
\newblock Entropic proximal mappings with applications to nonlinear programming.
\newblock \emph{Mathematics of Operations Research}, 17\penalty0 (3):\penalty0 670--690, 1992.

\bibitem[Theis \& Hoffman(2015)Theis and Hoffman]{theis2015trust}
Lucas Theis and Matt Hoffman.
\newblock A trust-region method for stochastic variational inference with applications to streaming data.
\newblock In \emph{International conference on machine learning}, pp.\  2503--2511. PMLR, 2015.

\bibitem[Titsias \& L{\'a}zaro-Gredilla(2014)Titsias and L{\'a}zaro-Gredilla]{titsias2014doubly}
Michalis Titsias and Miguel L{\'a}zaro-Gredilla.
\newblock Doubly stochastic variational bayes for non-conjugate inference.
\newblock In \emph{International conference on machine learning}, pp.\  1971--1979. PMLR, 2014.

\bibitem[Wainwright et~al.(2008)Wainwright, Jordan, et~al.]{wainwright2008graphical}
Martin~J Wainwright, Michael~I Jordan, et~al.
\newblock Graphical models, exponential families, and variational inference.
\newblock \emph{Foundations and Trends{\textregistered} in Machine Learning}, 1\penalty0 (1--2):\penalty0 1--305, 2008.

\bibitem[Williams \& Peng(1991)Williams and Peng]{williams1991function}
Ronald~J Williams and Jing Peng.
\newblock Function optimization using connectionist reinforcement learning algorithms.
\newblock \emph{Connection Science}, 3\penalty0 (3):\penalty0 241--268, 1991.

\bibitem[Wu \& Gardner(2024)Wu and Gardner]{wu2024understanding}
Kaiwen Wu and Jacob~R Gardner.
\newblock Understanding stochastic natural gradient variational inference.
\newblock \emph{arXiv preprint arXiv:2406.01870}, 2024.

\bibitem[Zhang(1999)]{zhang1999theoretical}
Tong Zhang.
\newblock Theoretical analysis of a class of randomized regularization methods.
\newblock In \emph{Proceedings of the twelfth annual conference on Computational learning theory}, pp.\  156--163, 1999.

\end{thebibliography}
\bibliographystyle{tmlr}

\newpage
\appendix
\addcontentsline{toc}{part}{Appendix} % Add appendix title to the main ToC

% Now create a new table of contents just for the appendix
\appendixTOC % Call the custom command to generate the appendix ToC

% **********************************************************************************

\section{Gradient Calculations}
\label{app:cal}

This section focuses on showing calculations that lead to the gradients used to prove results in the main paper. In Appendix D in \cite{khan2023bayesian}, the gradients used to define VN were calculated using the moment parametrization (which we will define in a bit) $\btu = \bmu = (\bmu_1, \bmu_2)$, i.e. gradient of the objective $\cL(\bmu)$ were calculated under this parameterization. The goal of this section is to use these gradients to derive the gradients w.r.t $\btu = (\bmm, \bV)$ and $\btu = (\bmm, \bC)$ parameterizations. 

Before we begin, it is useful to learn about the two important parameterizations for the Exponential family. The first one is called \emph{moments} or \emph{mean/expectation parameters}, and are given as:
\begin{equation}
  \bmu_{1} = \bmm \in \R^d, ~ \bmu_2 = \bS^{-1} + \bmm \bmm^\top \in \Sym^d_{++},
  \label{eq:mudef}
\end{equation}
where $\Sym^d_{++} = \{\bX \in \R^{d \times d} : \bX = \bX^\top, \bX \succ 0\}$ is the set of symmetric positive-definite matrices. The second parametrization is given via the \emph{natural parameters} as
\begin{equation}
  \bl_{1} = \bS \bmm \in \R^d, ~ \bl_2 = -\frac{1}{2} \bS \in \Sym^d_{++},
  \label{eq:msl}
\end{equation}
The way these two parameterizations are interrelated is via the following relation
\begin{equation}\label{eq:entropy_mom_nat_parm}
   \bmu = \nabla_{\bl} \cH^*(\bl) \iff \nabla_{\bmu} \cH(\bmu) = \bl,
\end{equation}
where $\cH^*(\bl)$ is a convex function, and its \emph{convex conjugate} $\cH(\bmu)$ returns the \emph{entropy}\footnote{We take $\cH(\bmu)$ as the negative entropy in this paper.} of the Gaussian distribution given its moments $\bmu$. These details are explained in Chapter~3 in \cite{wainwright2008graphical}, see also the paper by \cite{MaPi15}.

% -------------------------------------------------
\subsection{Gradients in \texorpdfstring{$(\bmm,\bV)$}{ } Parametrization}

Now let us look at the derivatives of $f$ w.r.t $\bmu_1$ and $\bmu_2$, but first we note that: 
    \begin{align}
    \nabla_{\bmu_1} \cH(\bmu) = \bl_1 \stackrel{(\ref{eq:entropy_mom_nat_parm}),(\ref{eq:msl})}{=} \bS\bmm,
\end{align}
% ============================== Definition of gradients =======================================
Let us define 
\begin{align*}
    \bg(\btu) := \bbE_{\bt \sim q_{\btu}}[\nabla_{\bt}\Bar{\ell}(\bt)], \quad \bH(\btu) := \bbE_{\bt \sim q_{\btu}}[\nabla^2_{\bt}\Bar{\ell}(\bt)],
\end{align*} where $\btu = (\bmm, \bV)$ when $q_{\btu} = \mathcal{N}(\bmm, \bV)$, and when this is the case, we omit writing $\bg(\bmm, \bV), \bH(\bmm, \bV)$ in favour of only writing $\bg, \bH$.
% Note that
% $$\nabla_{\bmm}\cL(\bmm_t, \bC_t) =  \bH_t,$$and
% \begin{gather}\label{eq:nabla-mu2}
%     \nabla_{\bV}\cL(\bmm_t, \bC_t) = \frac{1}{2} (\bH_t - \gamma\bS_t) \notag\\
%     = \frac{1}{2} (\bH_t - \gamma\left(\bC_t\bC_t^\top\right)^{-1}),
% \end{gather}
% where their derivation is provided in Appendix \ref{app:cal}, see \ref{eq:d-f-m} and \ref{eq:d-f-V}, respectively. 
Then, in Equation~10 in \cite{khan2023bayesian}, it is given that:
\begin{align}
    \nabla_{\bmu_1} f(\bmu) =  \bbE_{\bt \sim \mathcal{N}(\bmm, \bV)}[\nabla_{\bt} \Bar\ell(\bt)] - \bbE_{\bt \sim \mathcal{N}(\bmm, \bV)}[\nabla_{\bt}^2 \Bar\ell(\bt)]\bmm = \bg - \bH\bmm.
\end{align}
Combining the above two equations, we get:
\begin{align}
    \nabla_{\bmu_1} \cL(\bmu) &= \nabla_{\bmu_1}\bigg(\bbE_{\bt \sim \mathcal{N}(\bmm, \bV)}[\Bar\ell(\bt)] + \gamma \cH(\bmu) \bigg)\notag\\
    &=\nabla_{\bmu_1}\bbE_{\bt \sim \mathcal{N}(\bmm, \bV)}[\Bar\ell(\bt)] + \gamma \nabla_{\bmu_1} \cH(\bmu)\notag\\
    &= \bg - \bH\bmm + \gamma \bS\bmm, \label{eq:nabla_mu_1}
\end{align}
and from Equation~11 in \cite{khan2023bayesian}, we have
\begin{align}
    \nabla_{\bmu_2} \cL(\bmu) = \frac{1}{2} (\bH - \gamma \bS)\label{eq:nabla-mu2}
\end{align}
This further leads to\begin{align*}
    (\nabla_{\bmm} \cL(\bmu))_{i} &= \frac{\partial \cL(\bmu)}{\partial m_{i}} = \frac{\partial \cL}{\partial \bmu_1}^\top \frac{\partial \bmu_1}{\partial m_i} + \tr \left( \frac{\partial \cL}{\partial \bmu_2}^\top \frac{\partial \bmu_2}{\partial m_i} \right) \\
    &=  \left(\frac{\partial \cL}{\partial \bmu_1}\right)^\top_i + \sum_{k,l} \left( \frac{\partial \cL}{\partial \bmu_2}^\top \right)_{kl} \left( \frac{\partial \bmu_2}{\partial m_i} \right)_{lk} \\
    &=  \left(\frac{\partial \cL}{\partial \bmu_1}\right)^\top_i + \sum_{k \neq i} \left( \frac{\partial \cL}{\partial \bmu_2}^\top \right)_{ki} m_k  + \sum_{l \neq i} \left( \frac{\partial \cL}{\partial \bmu_2}^\top \right)_{il} m_l + 2 \left( \frac{\partial \cL}{\partial \bmu_2}^\top \right)_{ii} m_i \\
    &=  \left(\frac{\partial \cL}{\partial \bmu_1}\right)^\top_i + 2\sum_{k \neq i} \left( \frac{\partial \cL}{\partial \bmu_2}^\top \right)_{ki} m_k  + 2 \left( \frac{\partial \cL}{\partial \bmu_2}^\top \right)_{ii} m_i \\
    &=  \left(\nabla_{\bmu_1} \cL\right)^\top_i + 2\underbrace{(\nabla_{\bmu_2} \cL \bmm)_i}_{i\text{-th row}}
\end{align*}
\begin{align}
        \Rightarrow \nabla_\bmm \cL &= \nabla_{\bmu_1} \cL + 2\nabla_{\bmu_2} \cL \bmm \notag\\
    &\stackrel{(\ref{eq:nabla_mu_1}), (\ref{eq:nabla-mu2})}{=} \bigg(\bg - \bH \bmm + \gamma\bS\bmm_t\bigg) + 2\bigg(\frac{1}{2}(\bH - \gamma\bS)^\top \bmm\bigg) \notag\\
    &= \bg \label{eq:d-f-m} .
\end{align}
Similarly,
\begin{align}
    (\nabla_{\bV} \cL(\bmu))_{ij} &= \frac{\partial \cL(\bmu)}{\partial V_{ij}} = \frac{\partial \cL}{\partial \bmu_1}^\top \frac{\partial \bmu_1}{\partial V_{ij}} + \tr \left( \frac{\partial \cL}{\partial \bmu_2}^\top \frac{\partial \bmu_2}{\partial V_{ij}} \right) \notag\\
    &= \sum_{k,l} \left( \frac{\partial \cL}{\partial \bmu_2}^\top \right)_{kl} \left( \frac{\partial \bmu_2}{\partial S^{-1}_{ij}} \right)_{lk} \notag\\
    &= \left( \frac{\partial \cL}{\partial \bmu_2}^\top \right)_{ji} = \left( \frac{\partial \cL}{\partial \bmu_2} \right)_{ij},\notag\\
    \Rightarrow \nabla_\bV \cL &= \nabla_{\bmu_2} \cL \notag\\
    \Rightarrow \nabla_{(\bS)^{-1}} \cL(\bmu)  &= \left(\frac{\partial \cL(\bmu)}{\partial \bmu_2} \right)^\top\frac{\partial \bmu_2}{\partial (\bS)^{-1}} \notag\\
    &= \frac{1}{2}(\bH - \gamma\bS). \label{eq:d-f-V}
\end{align}

% ======================= Gradient wrt S =======================================
% Next, we calculate the gradient w.r.t $\bS$, i.e.
% \begin{align}
%     \nabla_{\bS} \cL(\bmu)  &= \frac{\partial \cL(\bmu)}{\partial \bmu_1} \circ\frac{\partial \bmu_1}{\partial \bS} + \left(\frac{\partial \cL(\bmu)}{\partial \bmu_2} \right)^\top\frac{\partial \bmu_2}{\partial \bS} \\
%     &= \left(\frac{\partial \cL(\bmu)}{\partial \bmu_2} \right)^\top\frac{\partial \bmu_2}{\partial \bS}\\
%     &\stackrel{\S}{=} -(\bS)^{-1}\:\frac{1}{2}(\bH - \bS)\:(\bS)^{-1}\label{eq:flipping-of-sign}
% \end{align}where $\S$ comes from the application of the fact that for a non-singular matrix $\bA \in \R^{d \times d}$, we have $$\frac{\partial (\bA^{-1})_{kl}}{\partial (A_{ij})} = - (\bA^{-1})_{ki}~(\bA^{-1})_{jl},$$ as given in \cite[Equation 60]{cookbook}.

% -------------------------------------------------
\subsection{Gradients in \texorpdfstring{$(\bmm,\bC)$}{ } Parametrization} \label{app:grad}

Let us represent $\cL(\bmu(\bC))$ as $\cL(\bmu_1, \bmu_2(\bC))$ where $\bmu_1 = \bmm$ (independent of $\bC$) and $\bmu_2(\bC) = \bV = \bC \bC^\top$. Using the chain rule for composite matrix functions given in Equation 1884 in \cite{Dattorro2015} combined with Equation 137 in \cite{cookbook}, we arrive at
\begin{align*}
\nabla_{\bC} \cL(\bmu(\bC)) &= \tr\left((\nabla_{\bmu_1}\cL(\bmu))^\top~(\nabla_{\bC} \bmu_1(\bC))  \right) + \tr\left((\nabla_{\bmu_2}\cL(\bmu))^\top~\nabla_{\bC} \bmu_2(\bC)\right) \\
&= \tr\left((\nabla_{\bmu_2}\cL(\bmu))^\top~\nabla_{\bC} \bmu_2(\bC)\right),
\end{align*}
which leads to
\begin{equation*}
\frac{\partial \cL(\bmu(\bC))}{\partial C_{ij}} = \tr \left( \frac{\partial \cL}{\partial \bmu_2}^\top \frac{\partial \bmu_2}{\partial C_{ij}} \right).
%\nabla_\bC \cL (\bmu_2(\bC)) = \tr \left( \frac{\partial \cL}{\partial \bmu_2}^\top \frac{\partial \bmu_2}{\partial \bC} \right).
\end{equation*}
Note that $\frac{\partial \bmu_2}{\partial \bC}$ is a fourth-order tensor whose entries are 
\begin{equation}\label{eq: al-nablaF-C}
\frac{\partial (\bC \bC^\top)_{kl}}{\partial C_{ij}} = 
\begin{cases}
C_{lj}   & k = i, l \neq i \\
2 C_{ij} & k = i, l = i \\
C_{kj}   & k \neq i, l = i \\
0 & \text{ else}.
\end{cases}
\end{equation}
where we used the expression $(\bC \bC^\top)_{kl} = \sum_n C_{kn} C^\top_{nl} = \sum_n C_{kn} C_{ln}$. Then
\begin{align*}
\frac{\partial \cL(\bmu(\bC))}{\partial C_{ij}} &= \tr \left( \frac{\partial \cL}{\partial \bmu_2}^\top \frac{\partial \bmu_2}{\partial C_{ij}} \right) \\
&= \sum_{k,l} \left( \frac{\partial \cL}{\partial \bmu_2}^\top \right)_{kl} \left( \frac{\partial \bmu_2}{\partial C_{ij}} \right)_{lk}\\
&= \sum_{k=i,l \neq i} \left( \frac{\partial \cL}{\partial \bmu_2}^\top \right)_{kl} \left( \frac{\partial \bmu_2}{\partial C_{ij}} \right)_{lk} + \sum_{k=i,l=i} \left( \frac{\partial \cL}{\partial \bmu_2}^\top \right)_{kl} \left( \frac{\partial \bmu_2}{\partial C_{ij}} \right)_{lk} + \sum_{k \neq i,l = i} \left( \frac{\partial \cL}{\partial \bmu_2}^\top \right)_{kl} \left( \frac{\partial \bmu_2}{\partial C_{ij}} \right)_{lk} \\
&= \sum_{l \neq i} \left( \frac{\partial \cL}{\partial \bmu_2}^\top \right)_{il} C_{lj} + 2 \left( \frac{\partial \cL}{\partial \bmu_2}^\top \right)_{ii} C_{ij} + \sum_{k \neq i} \left( \frac{\partial \cL}{\partial \bmu_2}^\top \right)_{ki} C_{kj} \\
&= \sum_{l} \left( \frac{\partial \cL}{\partial \bmu_2}^\top \right)_{il} C_{lj} + \sum_{k} \left( \frac{\partial \cL}{\partial \bmu_2}^\top \right)_{ki} C_{kj} \\
&= \left( \frac{\partial \cL}{\partial \bmu_2}^\top \bC \right)_{ij} + \left( \frac{\partial \cL}{\partial \bmu_2} \bC \right)_{ij}.
\end{align*}
Therefore, we conclude that
\begin{align}
\frac{\partial \cL(\bmu_2(\bC))}{\partial \bC} &= (\nabla_{\bmu_2} \cL)^\top \bC + (\nabla_{\bmu_2} \cL) \bC = 2(\nabla_{\bmu_2} \cL) \bC, \label{eq:2_nabla_f_V}\\
\Rightarrow \frac{\partial \cL(\bmu_2(\bC))}{\partial \bC} &\stackrel{(\ref{eq:nabla-mu2})}{=} \bH\bC - \gamma (\bC^\top)^{-1}\label{eq:nablaf_C}
\end{align}
as $\nabla_{\bmu_2} \cL = \nabla_\bV \cL = (\nabla_\bV \cL)^\top$.

% ------------------------------------------------
\section{SR-VN as an Approximation of VN (Detailed)}\label{app:SR_VN_vs_VN}

For $\gamma = 1$, the update of $\bS_t$ from VN in \Cref{eq:blr-updates} can be rewritten in terms of $\bV_t$ as,
\begin{align*}
   \bV_{t+1}^{-1} &= (1-\rho)\bV_t^{-1}+\rho \bH_t\\
   &= (1-\rho)\bC_t^{-\top}\bC_t^{-1}+\rho \bH_t\\
   &= \bC_t^{-\top}\bigg((1-\rho)\Id+\rho \bC_t^\top\bH_t\bC_t\bigg)\bC_t^{-1}\\
   \Rightarrow \bV_{t+1} &= \frac{\bC_t}{(1-\rho)}\underbrace{\bigg(\Id+\frac{\rho}{(1-\rho)} \bC_t^\top\bH_t\bC_t\bigg)^{-1}}_\text{Inverse}\bC_t^{\top}
\end{align*}This above inverse can be approximated using the following truncated Neumann series for any symmetric matrix $\bA$,
$$
   (\Id + \bA)^{-1} = \Id - \bA + \bA^2 + \mathcal{O}(\norm{\bA}^3).
$$to give
\begin{align*}
     \bV_{t+1} \approx \frac{\bC_t}{(1-\rho)}&\Bigg[\Id - \frac{\rho}{(1-\rho)} \bC_t^\top\bH_t\bC_t + \left(\frac{\rho}{(1-\rho)}\right)^2 \bC_t^\top\bH_t\bC_t\bC_t^\top\bH_t\bC_t\Bigg]\bC_t^{\top}
\end{align*}
Let us define, $\bW_t = \bC_t^\top\bH_t\bC_t,~ \hat{\bW}_t = \tril[\bW_t],$ and make use of the following identity: $\bW = \hat{\bW}_t + \hat{\bW}_t^\top$ to arrive at
\begin{align}
    \bV_{t+1} &\approx \frac{\bC_t}{(1-\rho)}\Bigg[\Id - \frac{\rho}{(1-\rho)} \sqr{ \hat{\bW}_t^\top + \hat{\bW}_t} + \left(\frac{\rho}{(1-\rho)}\right)^2\sqr{\hat{\bW}_t^\top + \hat{\bW}_t}^2\Bigg]\bC_t^{\top}\notag \\
    &= \frac{\bC_t}{(1-\rho)}\Bigg[\bigg(\Id - \frac{\rho}{2(1-\rho)}\bW_t\bigg)^2 + \frac{3}{4}\left(\frac{\rho}{(1-\rho)}\right)^2\bW_t^2\Bigg]\bC_t^{\top}\notag\\
    % &\stackrel{(a)}{\approx} \bC_t\bigg[(1+\rho+\rho^2)\Id - (\rho+2\rho^2) \bW_t+\rho^2\bW_t^2)\bigg]\bC_t^\top + \mathcal{O}(\rho^2)
    &\stackrel{(a)}{\approx} \bC_t\bigg[(1+\rho)\Id - \rho \bW_t\bigg]\bC_t^\top + \mathcal{O}(\rho^2)
    \label{eq:V_t_VN},
\end{align}where $(a)$ comes from also expanding $(1-\rho)^{-1} = 1+\rho+\rho^2 + \mathcal{O}(\rho^2)$, at all occurrences.

Now, let us write down the update for $\bV_{t+1}$ using the updates from SR-VN (\ref{eq:tan-cholesky}) (writing $\alpha$ for the step-size for SR-VN here). We have
\begin{align}
    &\bV_{t+1} = \bC_{t+1}\bC_{t+1}^\top\notag\\
    &= (\bC_t - \alpha \bC_t \tril[\bC_t^\top\bH_t\bC_t - \Id])(\bC_t - \alpha \bC_t \tril[\bC_t^\top\bH_t\bC_t - \Id])^\top\notag\\
    &= \bC_t\Bigg[\left(\left(1 + \frac{\alpha}{2}\right)\Id - \alpha \tril[\bC_t^\top\bH_t\bC_t]\right)\left(\left(1 + \frac{\alpha}{2}\right)\Id - \alpha \tril[\bC_t^\top\bH_t\bC_t]\right)^\top\Bigg]\bC_t^\top\notag\\
    &= \bC_t\Bigg[\left(\left(1 + \frac{\alpha}{2}\right)\Id - \alpha \hat{\bW}_t\right)\left(\left(1 + \frac{\alpha}{2}\right)\Id - \alpha \hat{\bW}_t\right)^\top\Bigg]\bC_t^\top\notag\\
    &=\bC_t\Bigg[\Id + \alpha(\Id - [\hat{\bW}_t+\hat{\bW}_t^\top])+\alpha^2\left(\frac{\Id - 2[\hat{\bW}_t+\hat{\bW}_t^\top] + 4\hat{\bW}_t\hat{\bW}_t^\top}{4} \right)\Bigg]\bC_t^\top\notag\\
    &=\bC_t\bigg[(1+\rho)\Id - \rho \bW_t\bigg]\bC_t^\top + \mathcal{O}(\rho^2)\label{eq:V_t_SR_VN}
    % &=\bC_t\Bigg[\bigg(\frac{2-2\rho}{2-3\rho}\bigg)^2\bigg(\Id - \frac{\rho}{2(1-\rho)}\bW_t\bigg)^2 + \bigg(\frac{2\rho}{4(2-3\rho)}\bigg)^2\frac{4\hat{\bW}_t\hat{\bW}_t^\top - \bW_t^2}{4}\Bigg]\bC_t^\top,\label{eq:V_t_SR_VN}
\end{align}where $\alpha$ is chosen to be $\rho$. Comparing \Cref{eq:V_t_VN} with \Cref{eq:V_t_SR_VN} shows that the two updates for $\bV$ only match up to the first order. The presence of a quadratic difference in \Cref{eq:V_t_SR_VN} arises due to an alternative approximation of the inverse, highlighting the distinctions in the update methods and exposing the less-than-optimal nature of SR-VN.

\section{Analysis Proofs}\label{app:ngd_conv}
\begin{table}[H]\caption{Vectorization Notation for a square matrix $\bA$.}
\begin{center}% used the environment to augment the vertical space
% between the caption and the table
\begin{tabular}{r c p{10cm} }
\toprule
$\bar{\bA}$ & $\defeq$ & lower triangular matrix derived from $\bA$ by replacing all supra-diagonal elements by zero\\
$\diag(\bA)$ & $\defeq$ & diagonal matrix derived from $\bA$ by replacing all non-diagonal elements by zero\\
$\bar{\bar{\bA}}$ & $\defeq$ & $\bar{\bar{\bA}} - \bar{\bA} - \diag(\bA)/2$\\
$\vtr(\bA)$ & $\defeq$ & vector obtained by stacking the columns of $\bA$ in order from left to right\\  
$\vech(\bA)$ & $\defeq$ & vector obtained from $\vtr(\bA)$ by omitting supra-diagonal elements.\\  
$\bK$ & $\defeq$ & commutation matrix such that $\bK \vtr(\bA) = \vtr(\bA^\top)$\\  
$\bL$ & $\defeq$ & elimination matrix such that $\bL \vtr(\bA) = \vech(\bA)$\\  
% \multicolumn{3}{c}{\underline{Decision Variables}}\\
% \multicolumn{3}{c}{}\\
% $y_f$ & $=$ & $\left\{\begin{array}{rl}
% 1,  & \text{if Supplier located at site $f$ is open} \\
% 0,  & \text{otherwise} \end{array} \right.$\\
\bottomrule
\end{tabular}
\end{center}
\label{tab:TableOfNotationForMyResearch}
\end{table}

\subsection{Proof of \texorpdfstring{\cref{lemma:bdd_FIM_eig}}{ }}\label{app:bdd_FIM_eig}

\begin{proof}
    % Given \Cref{assum:bdd_eig_V} and \Cref{remark:high_eigval}.
In \cite[Lemma 1]{tan2021analytic}, the inverse FIM in $(\bmm, \bC)$ parametrisation has the following form:
\begin{align}
  \bF^{-1}_{\btu}(\bmm, \bC)=\begin{pmatrix}
      \bF_{\bmm}^{-1} & 0 \\
       0 & \bF_{\bC}^{-1} \\
  \end{pmatrix}=\begin{pmatrix}
      \bV & 0 \\
       0 & \bF_{\bC}^{-1} \\
  \end{pmatrix},
\end{align}where, given $\bN = (\bK + \Id_{d^2})/2$,
\begin{align}
    \bF_{\bC}^{-1} = \frac{1}{2} \bL(\Id_d \otimes \bC)\bL^\top(\bL\bN\bL^\top)^{-1}\bL(\Id_d \otimes \bC^\top)\bL^\top.
\end{align}Let us work first to prove the LHS of the claim. From part $2$ of \Cref{assum:bounded-iterates}, we first see that $\bF_{\bmm}^{-1} = \bV \succcurlyeq \lambda_{\min}\Id$. For $\bF_{\bC}^{-1}$, we proceed as follows:
\begin{align*}
    \norm{\bF_{\bC}^{-1}}_2 \geq \frac{1}{\norm{\bF_{\bC}}_2} \stackrel{\cite[\text{~Lemma 1(i)}]{tan2021analytic}}{=} \frac{1}{\norm{2\bL(\Id \otimes \bC^{-\top})\bN(\Id \otimes \bC^{-1})\bL^\top}_2}.
\end{align*}Now, notice that $\norm{\bN}_2 \leq \frac{1}{2}(\norm{\bK}_2 + \norm{\Id_{d^2}}_2) \leq 1$ since $\norm{\bK}_2 \leq 1$, while $\bK$ being a permutation matrix. Since $\bL$ is also form of a permutation matrix (with additional zeros at appropriate places), we have that $\norm{\bK}_2 \leq 1$. Combining these facts with the norm inequality, we obtain
\begin{align*}
     \norm{\bF_{\bC}^{-1}}_2 &\geq \frac{1}{\norm{2\bL(\Id \otimes \bC^{-\top})\bN(\Id \otimes \bC^{-1})\bL^\top}_2} \geq \frac{1}{2\norm{\bL}_2^2\norm{\bN}_2\norm{\Id \otimes \bC^{-\top}}_2\norm{\Id \otimes \bC^{-1}}_2}\\\
     &\geq \frac{1}{2\norm{\Id \otimes \bC^{-\top}}_2\norm{\Id \otimes \bC^{-1}}_2}\\
     &\stackrel{(a)}{\geq} \frac{1}{2\norm{\bC^{-1}}_2^2}\\
     &\geq \frac{\lambda_{\min}^2}{2} \qquad \left(\because \norm{\bC^{-1}}_2^2 \geq \norm{\bV^{-1}}_2=\frac{1}{\lambda_{\min}} \right)
\end{align*}where $(a)$ comes from using $\norm{\bA \otimes \bB}_2 \leq \norm{\bA}_2\norm{\bB}_2,~\forall \bA, \bB \in \R^{d \times d}$, and $\norm{\bB^\top}_2 = \norm{\bB}_2$.

Now let us prove the RHS of the claim. From \Cref{remark:high_eigval}, we have $\bF_{\bmm}^{-1} = \bV \preccurlyeq \xi_u^2\Id$. For $\bF_{\bC}^{-1}$, we use the norm inequality as follows:
\begin{align}
    \norm{\bF_{\bC}^{-1}}_2 &\leq \frac{1}{2} \norm{\bL}^4_2\norm{(\bL\bN\bL^\top)^{-1}}_2\norm{\Id_d \otimes \bC}_2\norm{\Id_d \otimes \bC^\top}_2\notag\\
    &\leq \frac{1}{2}\norm{(\bL\bN\bL^\top)^{-1}}_2\norm{\bC}^2_2\notag\\
    &\stackrel{\text{\Cref{assum:bounded-iterates}}}{\leq} \frac{\xi_u^2}{2}\norm{(\bL\bN\bL^\top)^{-1}}_2 \label{eq:LNL_inv}
\end{align}
Let us see how the operator $\bL\bN\bL^\top: \R^{d(d+1)/2} \to \R^{d(d+1)/2}$ acts on $\vech(\bA) \in \R^{d(d+1)/2}$. By definition, we have
\begin{align}
    (\bL\bN\bL^\top) \vech(\bA) &= \left(\bL~\frac{\bK + \Id_{d^2}}{2}~\bL^\top\right) \vech(\bA) \notag\\
    &= \frac{1}{2}( (\bL\bK\bL^\top) \vech(\bA) + \underbrace{(\bL\bL^\top)}_{\Id_{d(d+1)/2}} \vech(\bA))\notag\\
    &= \frac{1}{2} (\bL\bK\vtr(\Bar{\bA})+\vech(\bA)), \qquad (\because \bL^\top\vech(\bA) = \vtr(\Bar{\bA}))\notag\\
    &= \frac{1}{2} (\bL\vtr(\Bar{\bA}^\top)+\vech(\bA)),\qquad (\because \bK\vtr(\Bar{\bA}) = \vtr(\Bar{\bA}^\top))\notag\\
    &= \frac{1}{2} (\vech(\Bar{\bA}^\top)+\vech(\bA))), \qquad (\because \bL\vtr(\Bar{\bA}^\top) = \vech(\Bar{\bA}^\top)),\label{eq:tril_equiv}
\end{align}where \Cref{eq:tril_equiv} implies that the operator $\bL\bN\bL^\top$ is equivalent to halving the sub-diagonal entries while keeping the diagonal of the matrix $\bA$ intact. Therefore, $\bL\bN\bL^\top$ has eigenvalues between $1/2$ and 1, which implies that its inverse has eigenvalues bounded between $1$ and $2$, leading to $\norm{(\bL\bN\bL^\top)^{-1}}_2 \leq 2$. Thus, after plugging in \Cref{eq:LNL_inv} this bound, we have 
\begin{align}
    \norm{\bF_{\bC}^{-1}}_2 &\leq \xi_u^2,
\end{align}allowing us to complete the proof of the main claim.
\end{proof}

\subsection{Proof of \texorpdfstring{\cref{lemma:PL}}{ }}\label{app:PL}

\begin{proof}
Notice that \Cref{eq:PL_KL} is equivalent to 
    \begin{align}\label{eq:PL_ELBO}
        \norm{\nabla_{\btu}\cL(\btu)}_{\bF_{\btu}^{-1}}^2 \geq 2\mu \left(\cL(\btu) - \cL(\btu_*)\right),
    \end{align} because $\KLop\bigg[q_{\btu}(\bt)\mid\mid p(\bt)\bigg]$ can be written as $$\KLop\bigg[q_{\btu}(\bt)\mid\mid p(\bt)\bigg] = -1/2\log(\det \bV) + \bbE_{\bt \sim \mathcal{N}(\bmm, \bV)}[\Bar \ell(\bt)] + \text{const.} \stackrel{(\ref{eq:general-elbo})}{=} \cL(\btu) + \text{const.}$$
Then simply note that by \Cref{lemma:bdd_FIM_eig}, we obtain that
\begin{align*}
     \norm{\nabla_{\btu} \cL}^2_{\bF_{\btu}^{-1}} 
     &\geq \lambda_{\min}^g\norm{\nabla_{\btu} \cL}^2\stackrel{(\ref{eq:strong_convex_new})}{\geq} 2\delta\lambda_{\min}^g  \bigg( \cL(\bmm, \bC) - \cL(\bmm_*, \bC_*)\bigg)
\end{align*}
where now $\mu=\delta\lambda_{\min}$ is indeed the desired the PL constant.
\end{proof}

\subsection{Proof of \texorpdfstring{\cref{theorem:FR-flow}}{ }}\label{app:ng_flow}

\begin{proof}
For conciseness, we will refer $\KLop\bigg[q_{\btu_t}(\bt)\mid\mid p(\bt\mid\cD)\bigg] = \KLop_t$ throughout the proof. Let us define the Lyapunov function $\mathcal{E}(t) = e^{2\mu t}(\KLop_t - \KLop_*),$ then
\begin{align*}
&\dot{\mathcal{E}}_t =2\mu e^{2\mu t} (\KLop_t - \KLop_*)+ e^{2\mu t}\ip{\nabla_{\btu}\KLop_t}{\dot{\btu}} \stackrel{(\ref{eq:NGD_flow})}{=} e^{2\mu t}(2\mu  (\KLop_t - \KLop_*)-\norm{\nabla_{\btu}\KLop_t}_{\bF_{\btu}^{-1}}^2) \stackrel{\text{\Cref{lemma:PL}}}{\leq}0.
\end{align*}The energy functional $\mathcal{E}(t) $ is therefore monotonically decreasing, and hence
\begin{align}
    \mathcal{E}(t) = e^{2\mu t}(\KLop_t - \KLop_*) \leq \mathcal{E}(0) = \KLop_0 - \KLop_*
\end{align}
\end{proof}
\subsection{Proof of \Cref{theorem:global-tan-cholesky}}\label{app:proof-tan-chol}

% Using the notation defined in \cite{tan2021analytic}, for a square matrix $\bA$, let $\Bar{\bA}$ and $\diag(\bA)$ be the lower triangular and diagonal matrix derived from $\bA$ respectively by replacing all supra-diagonal and non-diagonal elements by zero, and define
% %\begin{align*}
%     $\Bar{\Bar{\bA}} = \Bar{\bA} - \diag(\bA) / 2$.
% %\end{align*} 
\begin{proof}
From \Cref{assum:lipschitz}, we start by writing the following Taylor expansion for $\cL(\bmm, \bC)$ when both $\bmm$ and $\bC$ change. 
\begin{align}
\cL(\bmm_{t+1}, \bC_{t+1}) &\leq \cL(\bmm_t, \bC_t) + \nabla_\bmm \cL (\bmm_t, \bC_t)^\top (\bmm_{t+1} - \bmm_t) + \tr((\nabla_\bC \cL (\bmm_t, \bC_t))^\top (\bC_{t+1} - \bC_t)) \nonumber\\
&+ \frac{M}{2} \| \bmm_{t+1} - \bmm_t \|_2^2 +  \frac{M}{2}  \| \bC_{t+1} - \bC_t \|_F^2.\label{eq:tan-joint-taylor}
\end{align}
By plugging in values from \Cref{eq:tan-cholesky} in the joint Taylor expansion above, we arrive at:
\begin{align}
    &\cL(\bmm_{t+1}, \bC_{t+1}) -  \cL(\bmm_t, \bC_t) \notag\\
    &\leq \underbrace{-(\nabla_{\bmm}\cL(\bmm_t, \bC_t)^\top(\rho \bC_t \bC_t^\top \nabla_{\bmm} \cL(\bmm_t, \bC_t)) + \frac{M\rho^2}{2}\norm{ \bC_t \bC_t^\top \nabla_{\bmm} \cL(\bmm_t, \bC_t)}^2}_{A} \notag\\
    &\underbrace{-\rho\tr\bigg((\nabla_\bC \cL (\bmm_t, \bC_t))^\top\:\bC_t\Bar{\Bar{\bH}}_t\bigg) + \frac{M\rho^2}{2} \| \bC_t\Bar{\Bar{\bH}}_t\|_F^2}_{B}, \label{eq:taylor-A-B}
\end{align}where $\Bar{\Bar{\bH}}_t = \tril[\bC_t^\top\nabla_{\bC} \cL (\bmm_t, \bC_t)]$.
Using the following facts 
\begin{gather*}
    - \nabla_{\bmm}\cL(\bmm_t, \bC_t)^\top \bC_t \bC_t^\top \nabla_{\bmm} \cL(\bmm_t, \bC_t) \leq - \lambda_{\min} \norm{\nabla_{\bmm}\cL(\bmm_t, \bC_t)}^2\\
    \norm{\bC_t \bC_t^\top \nabla_{\bmm} \cL(\bmm_t, \bC_t)} \leq \norm{\bC_t \bC_t^\top}\norm{\nabla_{\bmm} \cL(\bmm_t, \bC_t)} \stackrel{\text{\Cref{remark:high_eigval}}}{\leq} \xi_u^2\norm{\nabla_{\bmm} \cL(\bmm_t, \bC_t)} \quad\quad 
\end{gather*}
and applying the lower for bound $\lambda_{\min}$  assumed in \Cref{assum:bounded-iterates} to term $A$, we arrive at
\begin{align*}
    A \leq  -\lambda_{\min} \rho\norm{\nabla_{\bmm}\cL(\bmm_t, \bC_t)}^2 + \frac{M\rho^2\xi_u^4}{2}\norm{\nabla_{\bmm}\cL(\bmm_t, \bC_t)}^2,
\end{align*}
which then gives us
\begin{align}
    A \leq \omega_m \norm{\nabla_{\bmm}\cL(\bmm_t, \bC_t)}^2 \label{eq:A-omega},
\end{align}where $$\omega_m = \frac{M\rho^2\xi_u^4}{2}-\lambda_{\min} \rho.$$
Now for term $B$, let us first look at the following term by repeatedly applying the matrix norm inequality ($\norm{\mathbf{A}\mathbf{B}}_F \leq \norm{\mathbf{A}}_F\norm{\mathbf{B}}_F$) to obtain
\begin{align*}
    \| \bC_t \Bar{\Bar{\bH}}_t\|_F^2 &\leq \norm{\bC_t}_F^2 \norm{\Bar{\Bar{\bH}}_t}_F^2 \stackrel{(a)}{\leq} \norm{\bC_t}_F^2\left(2\norm{\Bar{\bH}}_F^2 + \frac{1}{2}\norm{\diag(\bH)}_F^2\right) \stackrel{(b)}{\leq} \norm{\bC_t}_F^2\left(2\norm{\bH}_F^2 + \frac{1}{2}\norm{\diag(\bH)}_F^2\right)\\
    &\leq \norm{\bC_t}_F^2\left(2\norm{\bH}_F^2 + \frac{1}{2}\norm{\bH}_F^2\right) = \frac{5}{2}\norm{\bC_t}_F^2\norm{\bH}_F^2\\
    &= \frac{5}{2}\norm{\bC_t}_F^2\norm{\bC^\top \Bar{\bK}}_F^2 \leq \frac{5}{2}\norm{\bC_t}_F^4\norm{\Bar{\bK}}_F^2 \stackrel{(c)}{\leq} \frac{5}{2}\norm{\bC_t}_F^4\norm{\bK}_F^2 \\
    &= \frac{5}{2}\norm{\bC_t}_F^4\norm{\nabla_{\bC} \cL(\bmm_t, \bC_t)}_F^2,
\end{align*}
where $(a)$ comes from using the fact that $\forall \bA, \bB \in \R^{D \times D}$, we have $\norm{\bA - \bB}_F^2 \leq 2\norm{\bA}_F^2 + 2\norm{\bB}_F^2$, and choosing $\bA = \Bar\bH$ and $\bB = \diag(\bH)/2$. In $(b)$ and $(c)$ we have used $\norm{\Bar{\bA}}_F \leq \norm{\bA}_F$, as $\Bar{\bA}$ is a lower triangular form obtained by sending all supra-diagonal elements of $\bA$ to zero. For the first term in $B$, we have
\begin{align*}
    \tr\bigg((\nabla_\bC \cL (\bmm_t, \bC_t))^\top\:\bC_t \Bar{\Bar{\bH}}_t\bigg) &\leq \norm{\nabla_{\bC} \cL(\bmm_t, \bC_t)}_F \| \bC_t \Bar{\Bar{\bH}}_t\|_F \quad \quad \left(\because \tr(\bA\bB)\leq \norm{\bA}_F\norm{\bB}_F\right)\\
    &\leq  \sqrt{\frac{5}{2}}\norm{\bC_t}_F^2\norm{\nabla_{\bC} \cL(\bmm_t, \bC_t)}_F^2
\end{align*} 
This leads us to finally obtain
\begin{align*}
    &B \leq -\sqrt{\frac{5}{2}}\rho\norm{\bC_t}_F^2\norm{\nabla_{\bC} \cL(\bmm_t, \bC_t)}_F^2  + \frac{5M\rho^2}{4}\: \| \bC_t\|^4_F \norm{\nabla_{\bC} \cL(\bmm_t, \bC_t)}^2_F
\end{align*}
Now using the bounds in \Cref{assum:bounded-iterates}, we arrive at
\begin{align}
    B \leq \omega_C\norm{\nabla_{\bC} \cL(\bmm_t, \bC_t)}^2_F \label{eq:B-omega}
\end{align}where $$\omega_C =  \frac{5\rho^2\xi_u^4M}{4} - \sqrt{\frac{5}{2}}\rho\xi_l^2.$$ 
Now plugging  \Cref{eq:A-omega} and \Cref{eq:B-omega} in \Cref{eq:taylor-A-B} to get
\begin{align}
    &\Delta^{t+1} - \Delta^{t} \leq \omega_m\norm{\nabla_{\bmm} \cL(\bmm_t, \bC_t)}^2 + \omega_C\norm{\nabla_{\bC} \cL(\bmm_t, \bC_t)}^2_F \notag\\
    \Rightarrow \quad &\Delta^{t+1} - \Delta^{t} \leq -\eta \norm{\nabla_{(\bmm, \bC)} \cL(\bmm_t, \bC_t)}^2,\label{eq:Delta_dif}
\end{align}
where $\eta = \min\{-\omega_m, -\omega_C\}$ and we require that $\eta > 0$. Now from \Cref{assum:strong-convexity}, by joint $\delta$-strongly convexity of $\cL$ in $(\bmm, \bC)$, we have 
\begin{gather}
\cL(\bmm_t, \bC_t) - \cL(\bmm^*, \bC^*) \leq \frac{1}{2\delta}\norm{\nabla_{(\bmm, \bC)} \cL(\bmm_t, \bC_t)}^2 \notag\\
\Rightarrow  \Delta^{t} \leq \frac{1}{2\delta}\norm{\nabla_{(\bmm, \bC)} \cL(\bmm_t, \bC_t)}^2 \notag\\
\Rightarrow \quad \Delta^{t} \stackrel{(\ref{eq:Delta_dif})}{\leq} \frac{1}{2\eta\delta}(\Delta^{t} - \Delta^{t+1})\notag\\
\Rightarrow \quad \Delta^{t+1} \leq (1 - 2\eta\delta) \Delta^{t}
\end{gather}
from which we can conclude that we can indeed obtain a descent for $\cL$ in $(\bmm, \bC)$ if we require 
\begin{gather}
    0 < 1 - 2\eta\delta < 1
    \Rightarrow 0 < \eta < \frac{1}{2\delta}
\end{gather}
This would further lead us to the global convergence as
\begin{align}
    \Delta^{t+1} \leq (1 - 2\eta\delta)^{t+1} \Delta^{0}.
\end{align}
\end{proof}

% ---------------------------------------------------------
\subsection{Permissible Step Sizes for SR-VN}\label{sec:adm-rho}

\Cref{theorem:global-tan-cholesky} requires the condition $0 < \eta < \frac{1}{2\delta}$ to hold. This might prompt the reader to inquire whether a specific step size $\eta$ can be determined to ensure the fulfillment of this condition. We answer this question below in the context of a constant step size that we also employ in our experimental work.

\textbf{Constant step size.} Let $\lambda_{\max} = \xi_u^2$ be the upper bound for the largest eigenvalue of $\bV_t$, as derived in \Cref{remark:high_eigval}, then \Cref{thm:eta} becomes $$\eta = \min\left\{\lambda_{\min} \rho-\frac{M\rho^2\xi_u^4}{2},\quad \sqrt{\frac{5}{2}}\rho\xi_l^2 - \frac{5\rho^2\xi_u^4M}{4}\right\}.$$
Note that both arguments inside the $\min$ above are quadratics in $\rho$. Let us choose $\rho$ such that
\begin{align}\label{eq:permissible-step-size}
    \rho \leq \max\{\rho_1, \rho_2\},
\end{align}where $$\rho_1 = \frac{\lambda_{\min}}{\lambda_{\max}^2 M}\quad \text{and} \quad\rho_2= \sqrt{\frac{2}{5}}\frac{\xi_l^2}{M \xi_u^4}$$ are the extremal points of the quadratic equations in the first and second terms inside the $\min$, respectively. Choosing $\rho = \rho_1$ for the first argument and $\rho = \rho_2$ for the second argument yields
\begin{equation}\label{eq:gamma_step_size}
\eta \leq \min\left\{\frac{\lambda_{\min}^2}{2 \lambda_{\max}^2 M},~ \frac{\xi_l^4}{2 M \xi_u^4} \right\}. 
\end{equation}In order to ensure that the rate of convergence is contractive, we require $(1 - 2 \eta \delta) < 1$, that is $0 < 2 \eta \delta < 1$. With the choice above, $\eta$ is obviously positive, so let us focus on $2 \eta \delta < 1 \implies \eta < \frac{1}{2 \delta}$. From the first quantity in the min appearing in \Cref{eq:gamma_step_size}, we need \begin{gather}
    \frac{\lambda_{\min}^2}{2 \lambda_{\max}^2 M} < \frac{1}{2 \delta} 
    \implies \textcolor{blue}{\frac{\lambda_{\min}^2}{\lambda_{\max}^2}} \cdot \textcolor{magenta}{\frac{\delta}{M}} < 1. \label{eq:perm-cond}
\end{gather}The quantity $\textcolor{blue}{\frac{\lambda_{\min}}{\lambda_{\max}}}$ is a lower bound on the inverse condition number of the preconditioning matrix $\bV_t$ and is thus upper bounded by 1. Since $M$ is the smoothness constant of $\cL(\bmm, \bC)$, then the quantity $\textcolor{magenta}{\frac{\delta}{M}}$ is upper bounded by the inverse condition number of the Hessian of $\cL(\bmm, \bC)$ which is itself upper bounded by 1. This means that \Cref{eq:perm-cond} is always satisfied with the proposed choice of step size. The same argument applies to the second quantity in the min appearing in \Cref{eq:gamma_step_size}.

\textbf{Intuition:} We here give an intuitive explanation for the choice of the step size $\rho$. By \Cref{theorem:global-tan-cholesky}, we see that the speed of convergence is exponential with respect to time $t$. The contraction factor $C_t := (1 - 2\eta\delta)^{t+1}$ directly depends on the choice of step size through $\eta$ defined in \Cref{thm:eta}. From the latter equation, one can observe that larger step sizes can be used if the objective function is smoother (which is controlled by $M$). In turn, larger step sizes allow for faster convergence as the algorithm takes larger steps towards the minimum, which means that $(1 - 2\eta\delta)$ is smaller and therefore $C_t$ contracts at a faster rate.

\paragraph{Complexity Calculation.}

To compute the complexity, we require $\mathcal{L}(\bmm_t, \bC_t) - \mathcal{L}(\bmm_*, \bC_*) \leq \epsilon$, meaning from \Cref{eq:global_convergence} we must have \begin{align*} (1-2\eta\delta)^t(\mathcal{L}(\bmm_0, \bC_0) - \mathcal{L}(\bmm_*, \bC_*)) &\leq \epsilon\ \implies t \leq \frac{\log(\epsilon/(\mathcal{L}(\bmm_0, \bC_0) - \mathcal{L}(\bmm_*, \bC_*)))}{\log(1-2\eta\delta)}. \end{align*}Now using \Cref{eq:gamma_step_size}, we arrive at 
\begin{align}\label{eq:time_complexity}
    t \geq \min\left\{\frac{\log(\epsilon/(\mathcal{L}(\bmm_0, \bC_0) - \mathcal{L}(\bmm_*, \bC_*)))}{\log(1-\frac{\delta}{M}\frac{\lambda_{\min}^2}{\lambda_{\max}^2})}, \frac{\log(\epsilon/(\mathcal{L}(\bmm_0, \bC_0) - \mathcal{L}(\bmm_*, \bC_*)))}{\log(1-\frac{\delta}{M}\frac{\xi_l^4}{\xi_u^4})}\right\},
\end{align}
where the switching of inequality happens because the $\log(\cdot)$ term in the denominator is negative for both terms inside $\min$.

% -------------------------------------------------------------------
\section{Biased SR-VN Convergence }\label{app:ngd_w_bias}

% ---------------------------
\begin{definition}[Biased Gradient Oracle]
    A map $\bg: \R^{d+d\times d}\times \cD \to \R^{d}$ s.t.
    \begin{align}
        \hat{\bg}(\btu) = \bg + \bb_g(\btu).
    \end{align}for a bias $b_g : \R^{d+d\times d} \to \R^{d}$
\end{definition}
\begin{definition}[Biased Hessian Oracle]
    A map $\bH: \R^{d+d\times d}\times \cD \to \R^{d \times d}$ s.t.
    \begin{align}
        \hat{\bH}(\btu) = \bH + \bb_H(\btu).
    \end{align}for a bias $b_H : \R^{d+d\times d} \to \R^{d \times d}$
\end{definition}
We assume that the bias is bounded:
\begin{assumption}(Bounded biases)\label{assum:bdd_bias}
    There exists constants $0 \leq m_g,m_H < 1$, and $\zeta_m^2, \zeta_H^2 \geq 0$ such that
    \begin{enumerate}
        \item $\norm{\bb_g(\btu)}^2 \leq m_g\norm{\nabla_{\bmm} \cL(\bmm_t, \bC_t)}^2 + \zeta_g^2, \quad \forall \btu \in \R^{d+d\times d}.$
        \item $\norm{\bb_H(\btu)\bC}^2_F \leq m_H\norm{\nabla_{\bC} \cL(\bmm_t, \bC_t)}_F + \zeta_H^2, \quad \forall \btu \in \R^{d+d\times d}.$
    \end{enumerate}
\end{assumption}

With the prerequisites established, we now present our theorem demonstrating the convergence of the biased SR-VN update. This bias may arise from estimating the expectation using piecewise bounds presented in  \cite{marlin2011piecewise}.
% ****************************************
\begin{theorem}[With Bias Convergence]
\label{theorem:biased-tan-cholesky}
 Given Assumptions \ref{assum:init}--\ref{assum:bdd_bias} and considering the limit points $\bmm_*$ and $\bC_*$ in \Cref{eq:limit_pnts}, the updates provided by \Cref{algo:TC} satisfy
\begin{equation*}
\begin{aligned}
    &\cL(\bmm_{t+1}, \bC_{t+1})  - \cL(\bmm_*, \bC_*) \leq(1 - 2\eta\delta)^{t+1} (\cL(\bmm_0, \bC_0)  - \cL(\bmm_*, \bC_*))+\omega_m\zeta_g^2+\omega_C\zeta_H^2,
\end{aligned}
\end{equation*}
\end{theorem}

\begin{proof}
Following similar steps, we can arrive at the following versions of \Cref{eq:A-omega} and \Cref{eq:B-omega}
\begin{align}
    A \leq \omega_m \norm{\hat{\bH_t}}^2, \qquad B \leq \omega_C\norm{\nabla_{\bC} \cL(\bmm_t, \bC_t)}^2_F,
\end{align}where $\nabla_{\bC}\cL(\bmm_t, \bC_t)= (\hat{\bH}_t-\gamma (\bC_t\bC_t^\top)^{-1})\bC_t$. Now using \Cref{assum:bdd_bias}, one can get
\begin{align}
    A \leq \omega_m \norm{\hat{\bH_t}}^2 = \omega_m \norm{\bH_t + \bb_g(\btu_t)}^2 \leq \omega_m \bigg(\norm{\bH_t}^2+\norm{\bb_g(\btu_t)}^2\bigg) \leq \omega_m \bigg((1+m_g)\norm{\bH_t}^2+\zeta_g^2\bigg).\label{eq:A-bias}
\end{align}Similarly,
\begin{align}
    B &\leq \omega_C\norm{\nabla_{\bC} \cL(\bmm_t, \bC_t)}^2_F = \omega_C\norm{\bigg(\bH_t + \bb_H(\btu_t)-\gamma (\bC_t\bC_t^\top)^{-1}\bigg)\bC_t}^2_F\notag\\
    &\leq \omega_C\norm{\bigg(\bH_t-\gamma (\bC_t\bC_t^\top)^{-1}\bigg)\bC_t}^2_F + \omega_C\norm{\bb_H(\btu_t)\bC_t}_F^2\notag\\
    &=\omega_C\bigg(\norm{\nabla_{\bC} \cL(\bmm_t, \bC_t)}^2_F +\norm{\bb_H(\btu_t)\bC_t}_F^2\notag\bigg)\\
    &\leq\omega_C \bigg((1+m_H)\norm{\nabla_{\bC} \cL(\bmm_t, \bC_t)}^2+\zeta_H^2\bigg)\label{eq:B-bias}
\end{align}
Now plugging  \Cref{eq:A-bias} and \Cref{eq:B-bias} in \Cref{eq:taylor-A-B} to get
\begin{align}
    &\Delta^{t+1} - \Delta^{t} \notag\\
    &\leq \omega_m(1+m_g)\norm{\nabla_{\bmm} \cL(\bmm_t, \bC_t)}^2 + \omega_C(1+m_H)\norm{\nabla_{\bC} \cL(\bmm_t, \bC_t)}^2_F + \omega_m\zeta_g^2+\omega_C\zeta_H^2 \\
    \Rightarrow \quad &\Delta^{t+1} - \Delta^{t} \leq -\eta \norm{\nabla_{(\bmm, \bC)} \cL(\bmm_t, \bC_t)}^2+ \omega_m\zeta_g^2+\omega_C\zeta_H^2,\label{eq:Delta_diff_bias}
\end{align}
where $\eta = \min\{-\omega_m(1+m_g), -\omega_C(1+m_H)\}$ and we require that $\eta > 0$. Now from \Cref{assum:strong-convexity}, by joint $\delta$-strongly convexity of $\cL$ in $(\bmm, \bC)$, we have 
\begin{gather}
\cL(\bmm_t, \bC_t) - \cL(\bmm^*, \bC^*) \leq \frac{1}{2\delta}\norm{\nabla_{(\bmm, \bC)} \cL(\bmm_t, \bC_t)}^2 \notag\\
\Rightarrow  \Delta^{t} \leq \frac{1}{2\delta}\norm{\nabla_{(\bmm, \bC)} \cL(\bmm_t, \bC_t)}^2 \notag\\
\Rightarrow \quad \Delta^{t} \stackrel{(\ref{eq:Delta_diff_bias})}{\leq} \frac{1}{2\eta\delta}(\Delta^{t} - \Delta^{t+1})+\frac{1}{2\eta\delta}(\omega_m\zeta_g^2+\omega_C\zeta_H^2)\notag\\
\Rightarrow \quad \Delta^{t+1} \leq (1 - 2\eta\delta) \Delta^{t}+ \omega_m\zeta_g^2+\omega_C\zeta_H^2
\end{gather}
% from which we can conclude that we can indeed obtain a descent for $\cL$ in $(\bmm, \bC)$ if we require 
% \begin{gather}
%     0 < 1 - 2\eta\delta < 1
%     \Rightarrow 0 < \eta < \frac{1}{2\delta}
% \end{gather}
% This would further lead us to the global convergence as
% \begin{align}
%     \Delta^{t+1} \leq (1 - 2\eta\delta)^{t+1} \Delta^{0}.
% \end{align}
\end{proof}
\begin{remark}
It is worth noting that the piecewise method for computing expectations, given in \cite{marlin2011piecewise}, involves computing a 1-D integral which one can compute below numerical float precision. In other words, the bias with piecewise bounds can be made arbitrarily small by increasing the number of pieces. The bias is therefore very small, and hence the terms $\zeta_g^2$ and $\zeta_H^2$ are negligible which shows that SR-VN converges as long as this holds.
\end{remark}

\section{Piecewise versus stochastic implementation}\label{sec:piece_vs_stoc}

In this section, we make a comparison between the piecewise and the stochastic implementation of the expectation in ELBO, see \Cref{fig:piece_vs_stoc}. This plot highlights why the piecewise implementation is a more accurate and faster representation of the deterministic algorithm we analyzed, and thus the preferred choice for the experiments carried out in the paper.

% ---------------------------------------------
\begin{figure}[t!]
   \centering
\includegraphics[width=0.9\linewidth]{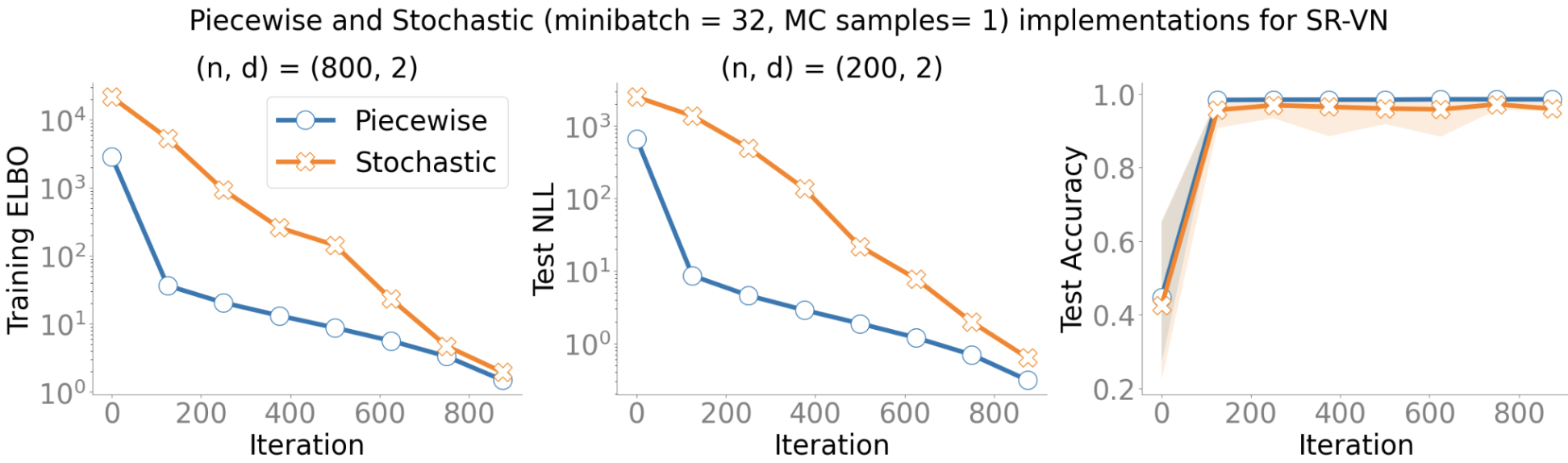}
\caption{\small{The provided comparison highlights the differences between the piecewise implementation and the
(minibatch + MC sampling) stochastic implementation for SR-VN. The results clearly demonstrate that the
stochastic approach is much slower compared to the piecewise method. Additionally, while the stochastic
implementation converges to a slightly higher ELBO, the piecewise implementation reaches a more accurate
solution more rapidly. These findings are based on iterates averaged over $5$ different initializations. The dataset taken is the same as displayed in \Cref{fig:runtime_comp}, taken with different $n$’s.}}
\label{fig:piece_vs_stoc}
\end{figure}
% ---------------------------------------------

\newpage
% -------------------------------------------------------------------
\section{Experiments Detail}\label{app:exp}

For all experiments, we first use grid search to tune model hyper-parameters, where the search is performed in a specific range of values. The resultant values were then fixed during our experiments. The statistics of the datasets and the model hyper-parameters used are given in \Cref{tab:dataset-hyperparameters}.  In practice, implementing deterministic NGD is often feasible by employing piecewise analytical linear and quadratic bounds to compute expected values. For example, please refer to \citep{marlin2011piecewise}. Our experiments adopt this technique and replicate the setup outlined in \citep{khan2017conjugate}.

\begin{table}[H] % Use table for one-column layout
    \centering
    \caption{Dataset Statistics and Model Hyperparameters}
    \label{tab:dataset-hyperparameters}
    \begin{tabular}{|c|c|c|c|c|c|c|c|}
        \hline
        \multirow{2}{*}{Dataset} & \multirow{2}{*}{$\beta$} & \multirow{2}{*}{$N$} & \multirow{2}{*}{$d$} & \multirow{2}{*}{$N_{\text{train}}$} & \multicolumn{3}{c|}{Step Sizes} \\
        \cline{6-8}
         &  &  &  &  & VN & SQ-VN & BW-GD \\
        \hline
        Australian-scale & $10^{-5}$ & 690 & 14 & 552 & $5 \times 10^{-3}$ & $5 \times 10^{-3}$ & $4.4 \times 10^{-3}$ \\
        Diabetes-scale & $10^{-2}$ & 768 & 8 & 614 & $5 \times 10^{-3}$ & $5 \times 10^{-3}$ & $9 \times 10^{-4}$ \\
        breast-cancer & $10^{-1}$ & 683 & 10 & 546 & $9 \times 10^{-3}$ & $6 \times 10^{-3}$ & $6.3 \times 10^{-3}$ \\
        Mushrooms  & $10^{-2}$ & 8124 & 112 & 6499 & $2.5 \times 10^{-4}$ & $2.5 \times 10^{-4}$ & $8.5 \times 10^{-5}$ \\
        \hline
        Phishing & $10^{-2}$ & 11055 & 68 & 8844 & $4 \times 10^{-4}$ & $4 \times 10^{-4}$ & $8 \times 10^{-5}$ \\
        MNIST & $10^{-1}$ & 70000 & 784 & 60000 & $5 \times 10^{-6}$ & $5 \times 10^{-6}$ & $10^{-6}$ \\
        Covtype-scale & $2 \times 10^{-2}$ & 581012 & 54 & 500000 & $10^{-5}$ & $10^{-5}$ & $10^{-6}$ \\ \hline
        Leukemia & $2 \times 10^{-1}$ & 38 & 7129 & 34 & $5 \times 10^{-6}$ & $5 \times 10^{-6}$ & $1.5 \times 10^{-6}$ \\ \hline
    \end{tabular}
\end{table}

\newpage

\subsection{Runtime comparison}\label{sec:runtime_comp}

It is worth noting that each algorithm differs in computational complexity \textit{not} because of the training data size---which is the same across all---but due to its specific update operations after computing the expected gradient $\mathbf{g}$ and Hessian $\bH$. Since these quantities cost the same for each method, we focus on the subsequent matrix/vector steps.

% ---------------------------------------------
\begin{figure}[H]
   \centering
\includegraphics[width=\linewidth]{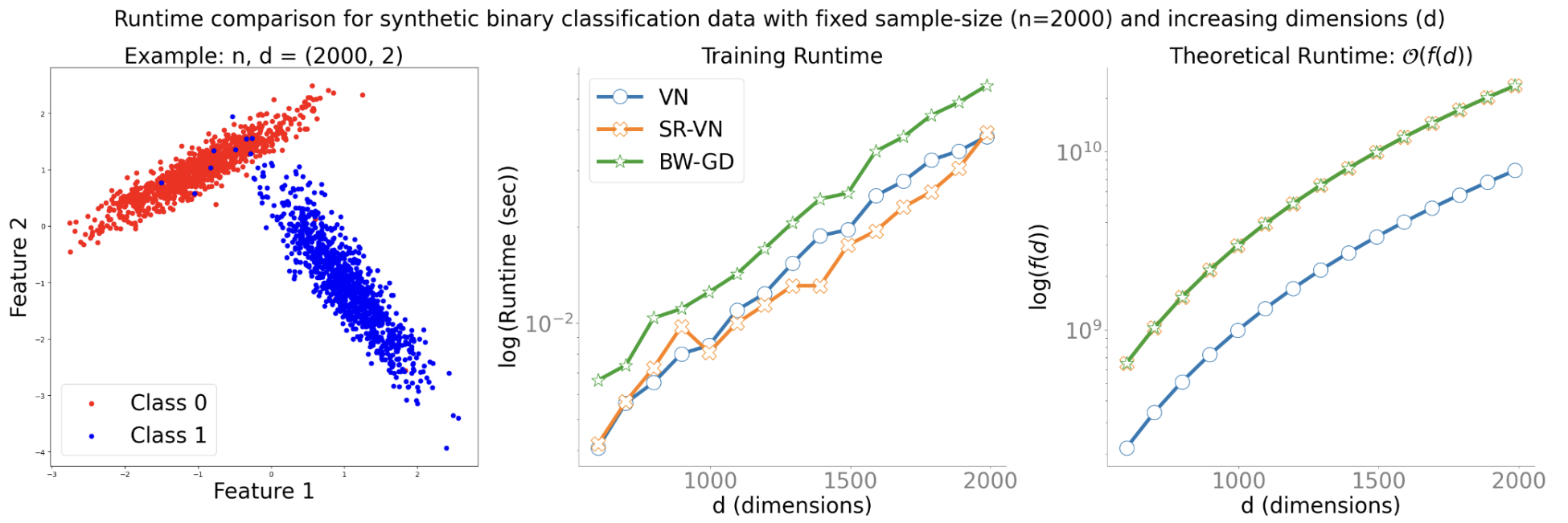}
\caption{\small{We compared the run times of all three algorithms across increasing dimensions while keeping the
number of data points $(n)$ fixed. The run times were averaged over $1000$ iterations per algorithm. The theoretical
run times, derived in our rebuttal answer, are functions of $d$ and are plotted to reflect the floating point operations.
The results indicate that the run times can vary significantly among the algorithms, depending on the specific
data and dimensionality.}}
\label{fig:runtime_comp}
\end{figure}
% ---------------------------------------------

\begin{enumerate}
    \item \textbf{Complexity for VN:}
    \begin{itemize}
        \item $\bS$ update: scalar multiplication + matrix addition $\to \mathcal{O}(2d^2)$.
        \item $\bmm$ update: matrix inversion $\to \mathcal{O}(d^3)$, plus matrix-vector product $\to \mathcal{O}(d^2)$, plus matrix addition $\to \mathcal{O}(d^2)$, for a total $\mathcal{O}(2d^2 + d^3)$.
        \item \textbf{Overall cost:} $\mathcal{O}(d^3 + 4d^2)$.
    \end{itemize}

    \item \textbf{Complexity for SR-VN:}
    \begin{itemize}
        \item $\bC$ update: includes $\bC^\top \bH\bC$ ($\mathcal{O}(d^3)$) plus some matrix additions/subtractions ($\mathcal{O}(d^2)$) and multiplications ($\mathcal{O}(d^3)$).
        \item $\bmm$ update: $\bC\bC^\top$ multiplication ($\mathcal{O}(d^3)$) plus matrix-vector products ($\mathcal{O}(d^2)$).
        \item \textbf{Overall cost:} $\mathcal{O}(3\,d^3 + \tfrac{9}{2}\,d^2 + \tfrac{d}{2})$.
    \end{itemize}

    \item \textbf{Complexity for BW-GD:}
    \begin{itemize}
        \item $\bmm$ update: scalar multiplication + matrix subtraction $\to \mathcal{O}(3d^2)$.
        \item $\bM$ update: two matrix additions + scalar operations $\to \mathcal{O}(2d^2)$.
        \item $\bV$ update: matrix inversion $\to \mathcal{O}(d^3)$ plus two matrix multiplications $\to \mathcal{O}(2d^3)$, totaling $\mathcal{O}(3d^3)$.
        \item \textbf{Overall cost:} $\mathcal{O}(3d^3 + 5d^2)$.
    \end{itemize}
\end{enumerate}
 These matrix/vector operations can dominate the wall-clock time for each algorithm. To illustrate how runtime scales in practice, we compare VN, SR-VN, and BW-GD in \Cref{fig:runtime_comp} as the dimension grows, highlighting their differing execution costs.

\newpage
% -----------------
\subsection{Algorithmic Details and Additional Results}\label{app:algo-details}

In \Cref{tab:dataset-performance}, we provide descriptive statistics for different metrics achieved by each algorithm when it is run for every dataset considered.

\begin{table}[H] 
    \centering
    \caption{Performance Comparison on different Datasets}
    \label{tab:dataset-performance}
    \begin{tabular}{|c|c|c|c|c|c|}
        \hline
        Dataset & Algorithm & Train ELBO & Test NLL & Test Accuracy & Time (sec) \\
        \hline
        \multirow{3}{*}{Australian-scale} & VN & 196.82 & 54.61 & 0.877 & 0.53\\
         & BW-GD & 196.82 & 54.61 & 0.877 & 0.53\\
         & SQ-VN & 196.82 & 54.61 & 0.877 & 0.53\\
         \hline
        \multirow{3}{*}{Diabetes-scale} & VN & 301.18 & 79.72 & 0.74 & 0.7\\
         & BW-GD & 301.18 & 79.72 & 0.74 & 0.49\\
         & SQ-VN & 301.18 & 79.72 & 0.74 & 0.53\\
         \hline
        \multirow{3}{*}{Breast-cancer} & VN & 52.86 & 13.62 & 0.956 & 0.61\\
         & BW-GD & 52.86 & 13.62 & 0.956 & 0.61\\
         & SQ-VN & 52.86 & 13.62 & 0.956 & 0.61\\
         \hline
        \multirow{3}{*}{mushrooms} & VN & 57.51 & 6.69 & 1.0 & 1.04\\
         & BW-GD & 78.41 & 7.02 & 0.999 & 1.04\\
         & SQ-VN & 53.56 & 6.43 & 0.999 & 1.05\\
         \hline
        \multirow{3}{*}{Leukeumia} & VN & 339.41 & 727.96 & 0.706 & 0.36\\
         & BW-GD & 564.65 & 796.76 & 0.676 & 0.38\\
         & SQ-VN & 339.38 & 727.99 & 0.706 & 0.37\\
         \hline
         \hline
        \multirow{3}{*}{Phishing} & VN & 1267.77 & 334.19 & 0.938 & 0.99 \\
         & BW-GD & 1377.645 & 349.10 & 0.933 & 0.97\\
         & SQ-VN & 1268.06 & 333.37 & 0.939 & 0.98\\
        \hline
        \multirow{3}{*}{MNIST} & VN & 6317.52 & 1017.64 & 0.989 & 0.22 \\
         & BW-GD & 6367.23 & 1102.34 & 0.988 & 0.23\\
         & SQ-VN & 6301.87 & 1012.37 & 0.989 & 0.2\\
        \hline
        \multirow{3}{*}{Covtype-scale} & VN & 27066.64 & 10741.37 & 0.745 & 0.95\\
         & BW-GD & 28036.23 & 11259.20 & 0.714 & 0.98\\
         & SQ-VN & 27065.02 & 10740.73 & 0.745 & 0.91\\
        \hline
    \end{tabular}
\end{table}

% ************************
\newpage
\subsection{Iteration Plots with Accuracy}\label{app:add_plots}
% ---------------------------------------------
% \begin{figure}[H]
%    \centering
% \begin{tabular}{ccc}
%     \includegraphics[width=\linewidth]{Images/iter/australian_scale.pdf}
% \end{tabular}
% \caption{}
% \label{}
% \end{figure}
% ---------------------------------------------

% ---------------------------------------------
\begin{figure}[htb!]
   \centering
\begin{tabular}{ccc}
    \includegraphics[width=\linewidth]{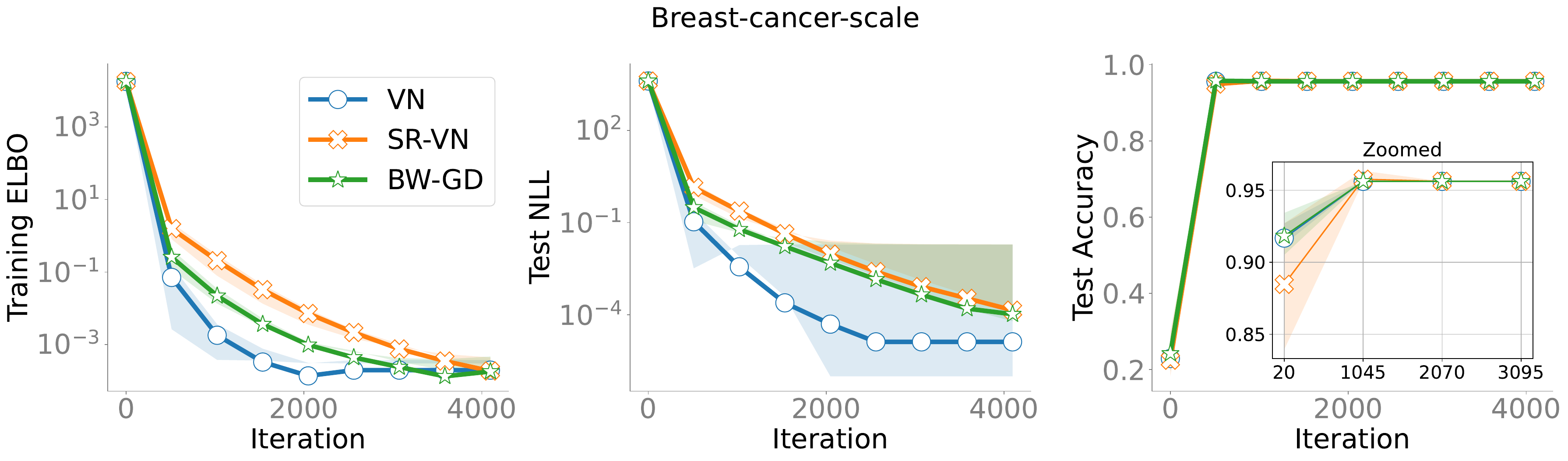}\\
    \includegraphics[width=\linewidth]{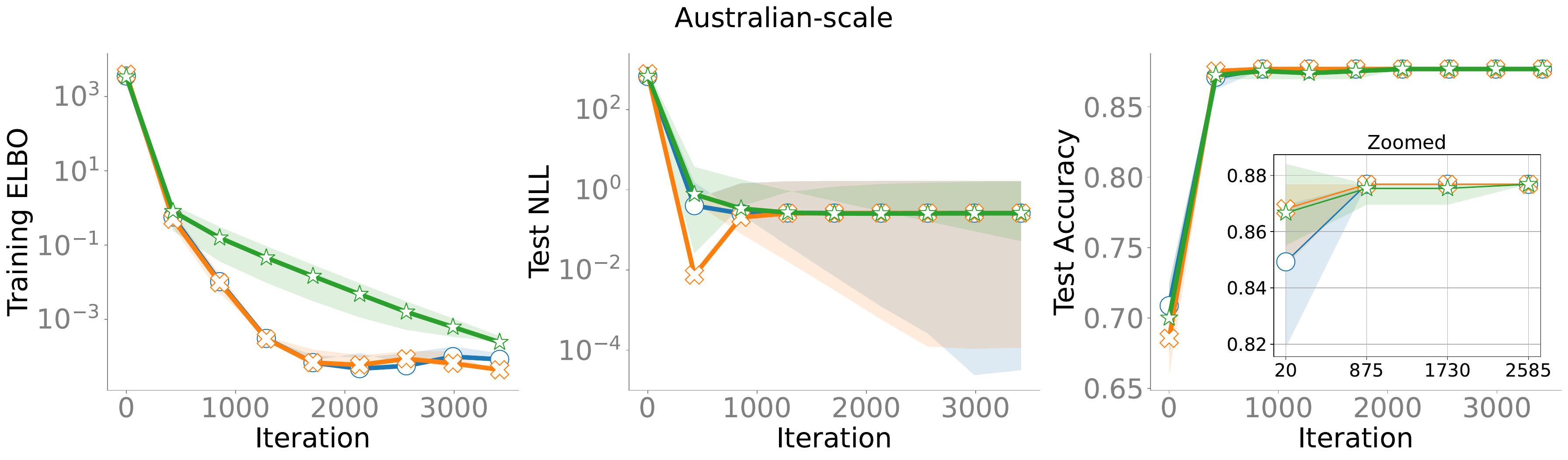}\\
    \includegraphics[width=\linewidth]{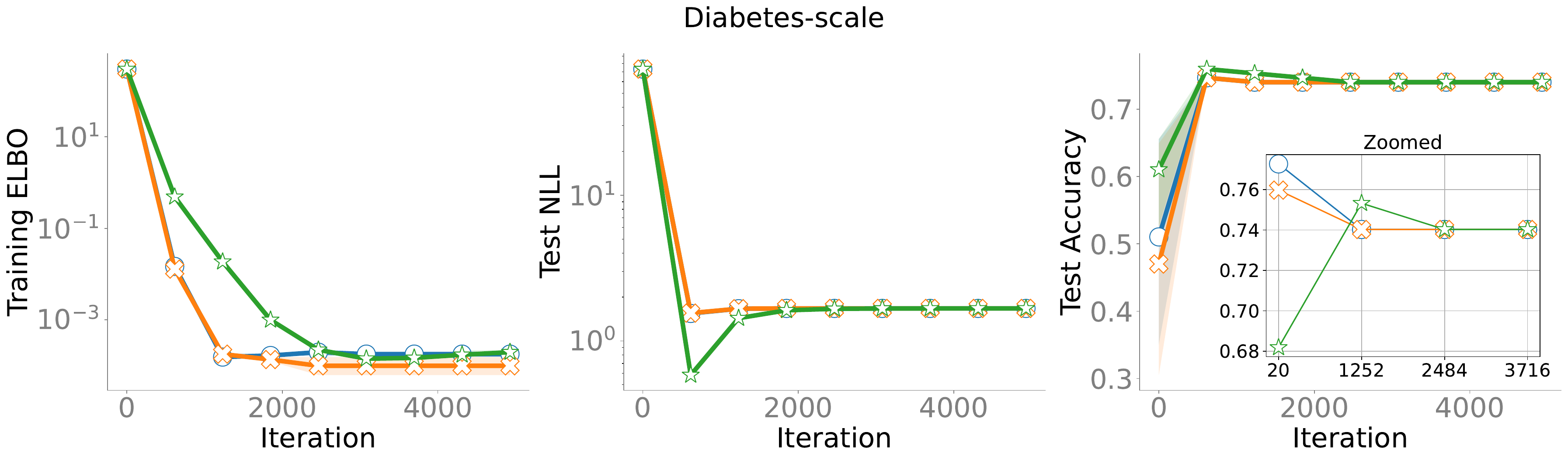}\\
    \includegraphics[width=\linewidth]{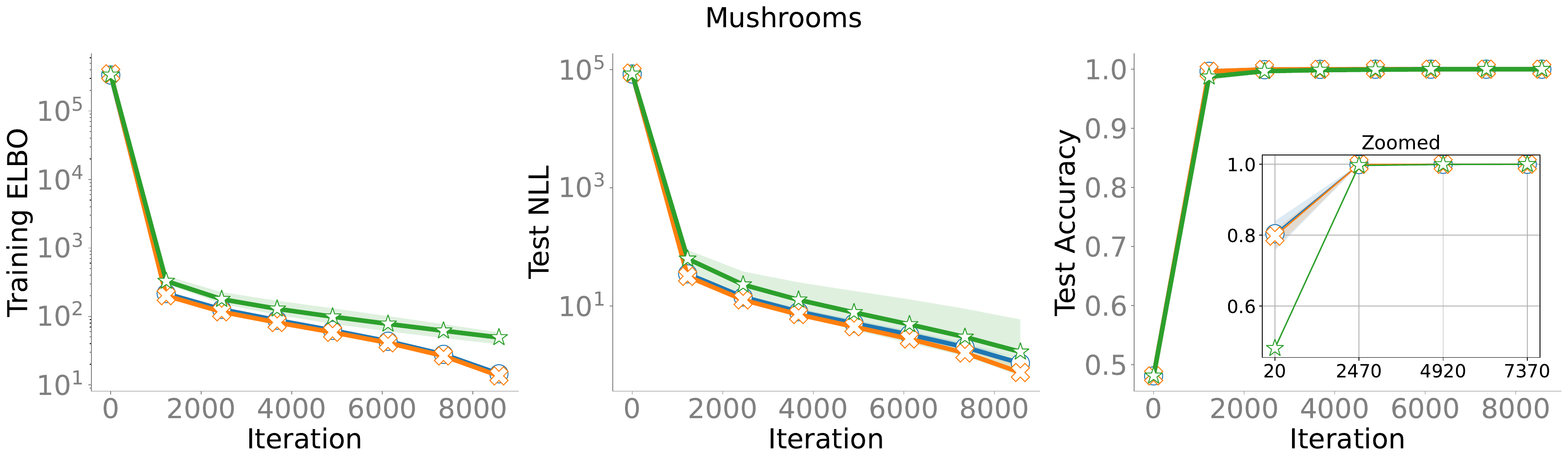}\\
\end{tabular}
\caption{Plots with respect to the number of iterations for small-scale datasets.}
% \label{}
\end{figure}
% ---------------------------------------------

% -------
% ---------------------------------------------
% \begin{figure}[H]
%    \centering
% \begin{tabular}{ccc}
%     \includegraphics[width=\linewidth]{Images/iter/mnist.pdf}
% \end{tabular}
% \caption{}
% \label{}
% \end{figure}
% ---------------------------------------------
\newpage
% ---------------------------------------------
\begin{figure}[htb!]
   \centering
\begin{tabular}{ccc}
    \includegraphics[width=\linewidth]{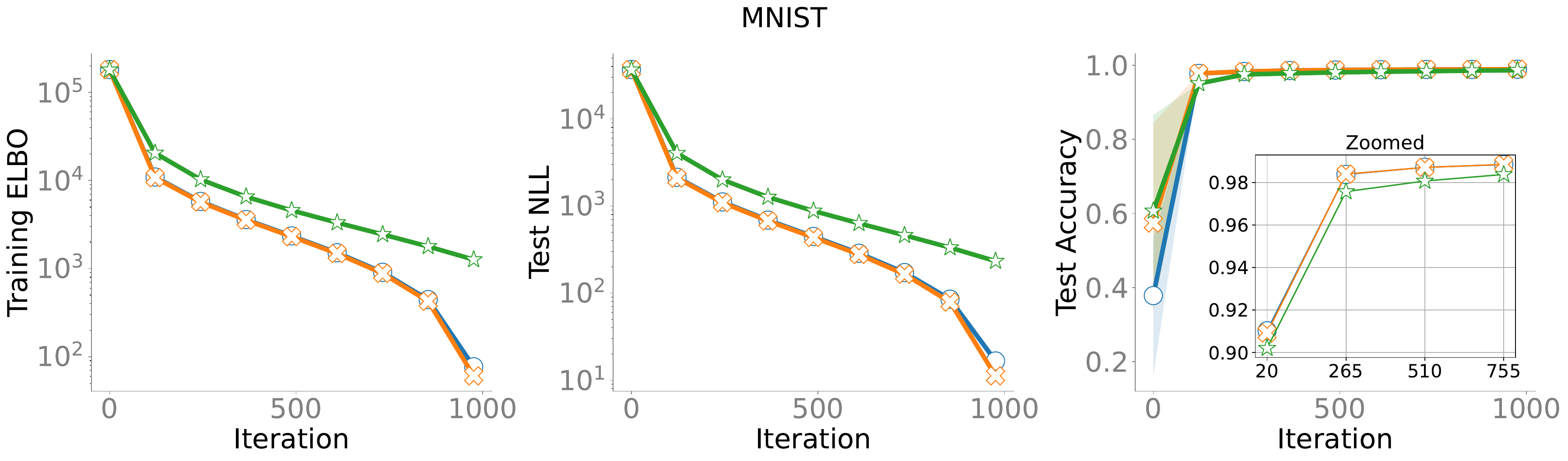}\\
    \includegraphics[width=\linewidth]{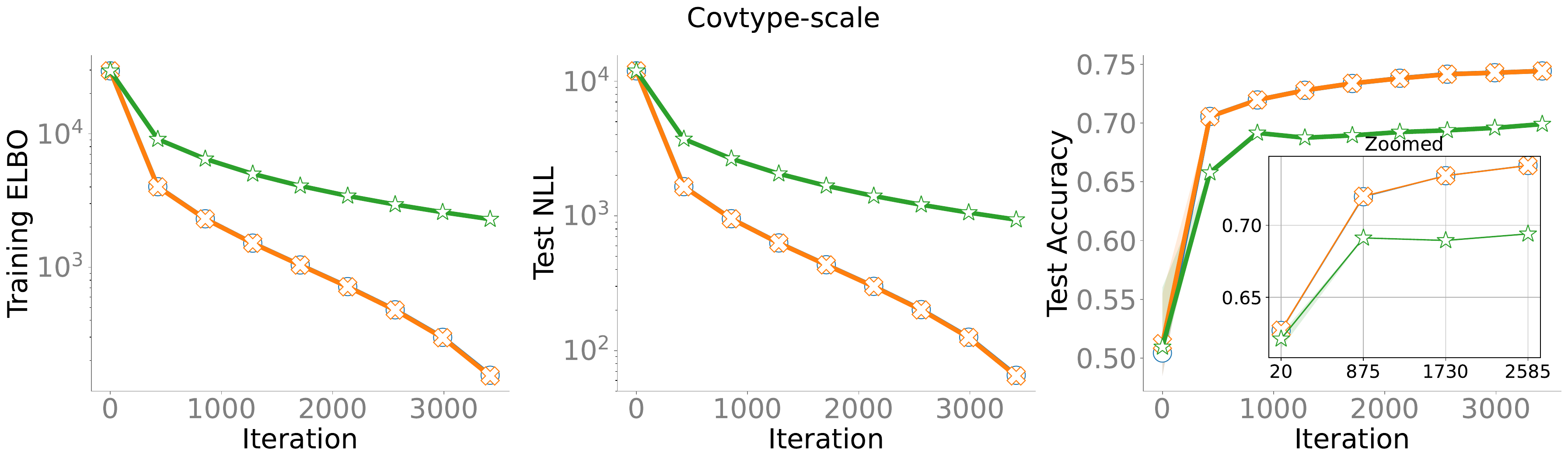}\\
    \includegraphics[width=\linewidth]{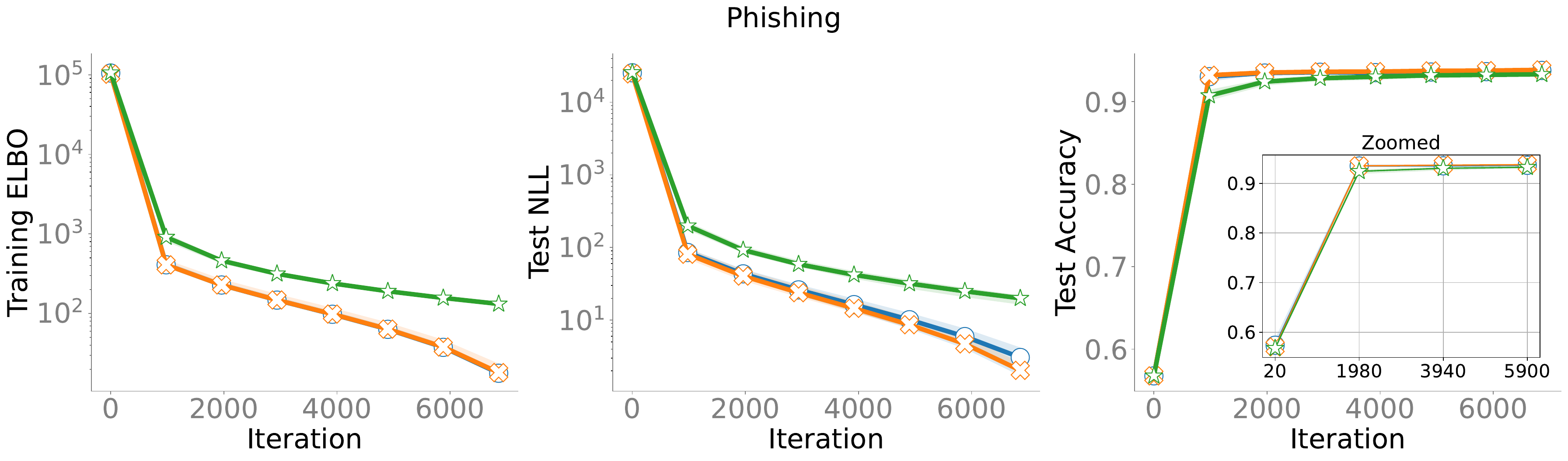}\\
    \includegraphics[width=\linewidth]{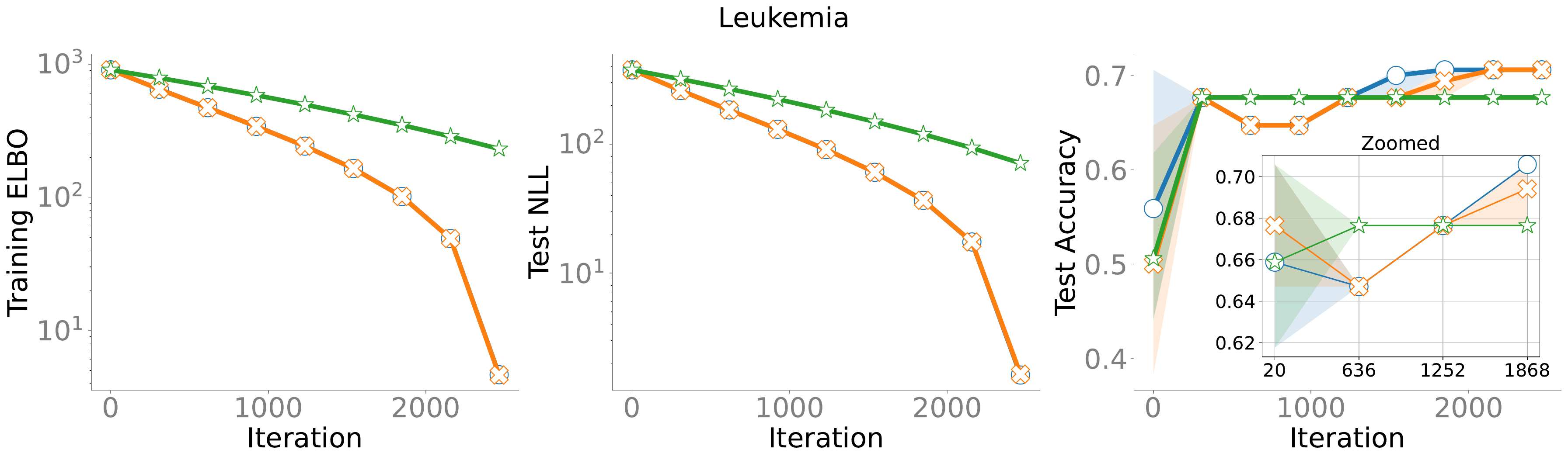}
\end{tabular}
\caption{Plots with respect to the number of iterations for large-scale datasets.}
% \label{}
\end{figure}

\newpage

\subsection{Time Plots with Accuracy}
% ************************

% ---------------------------------------------

\begin{figure}[htb!]
   \centering
\begin{tabular}{ccc}
\includegraphics[width=\linewidth]{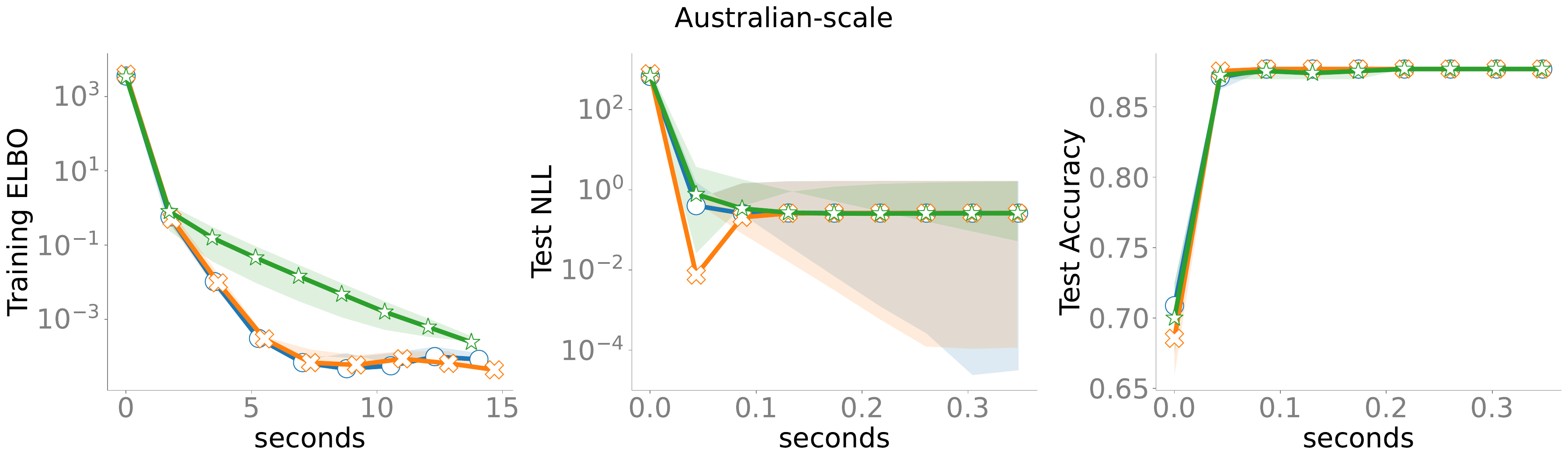}\\
\includegraphics[width=\linewidth]{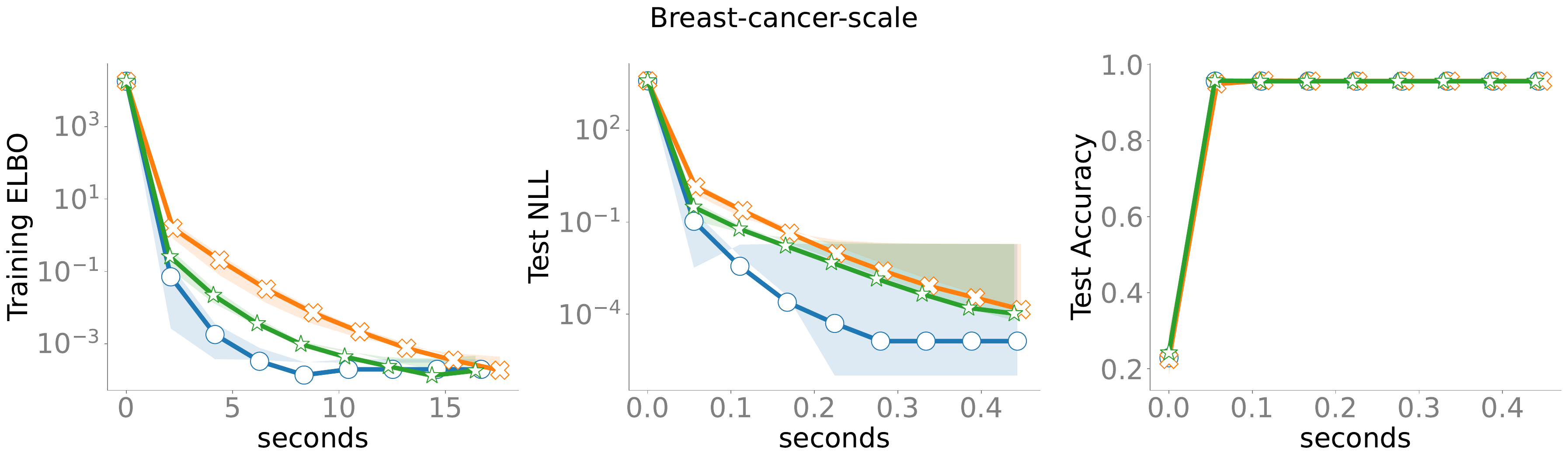}\\
\includegraphics[width=\linewidth]{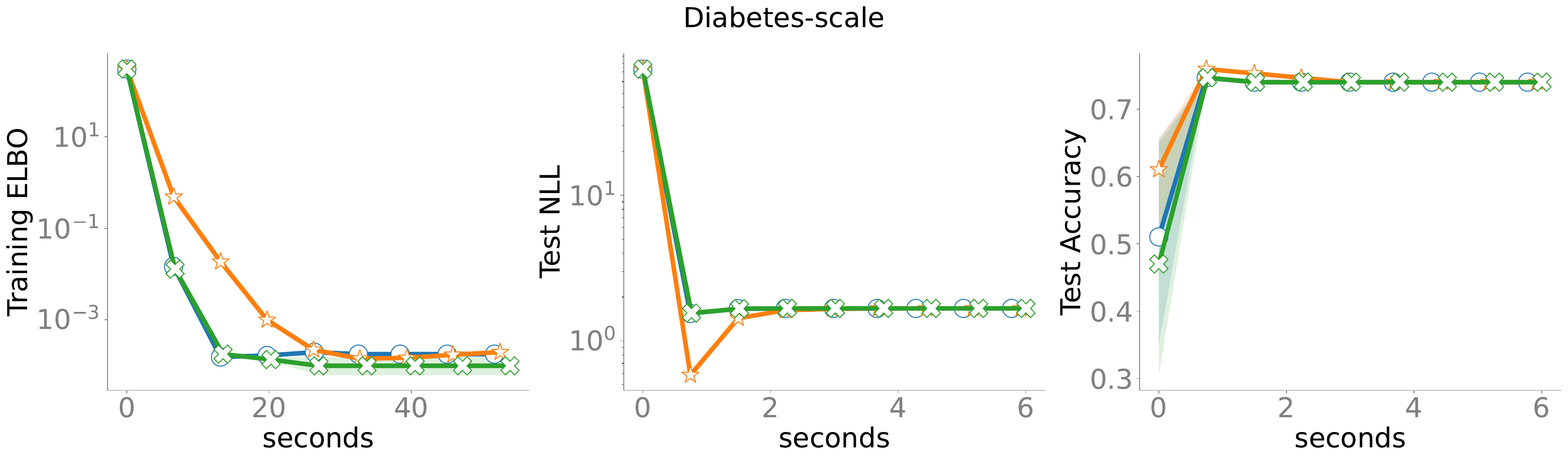}\\
\includegraphics[width=\linewidth]{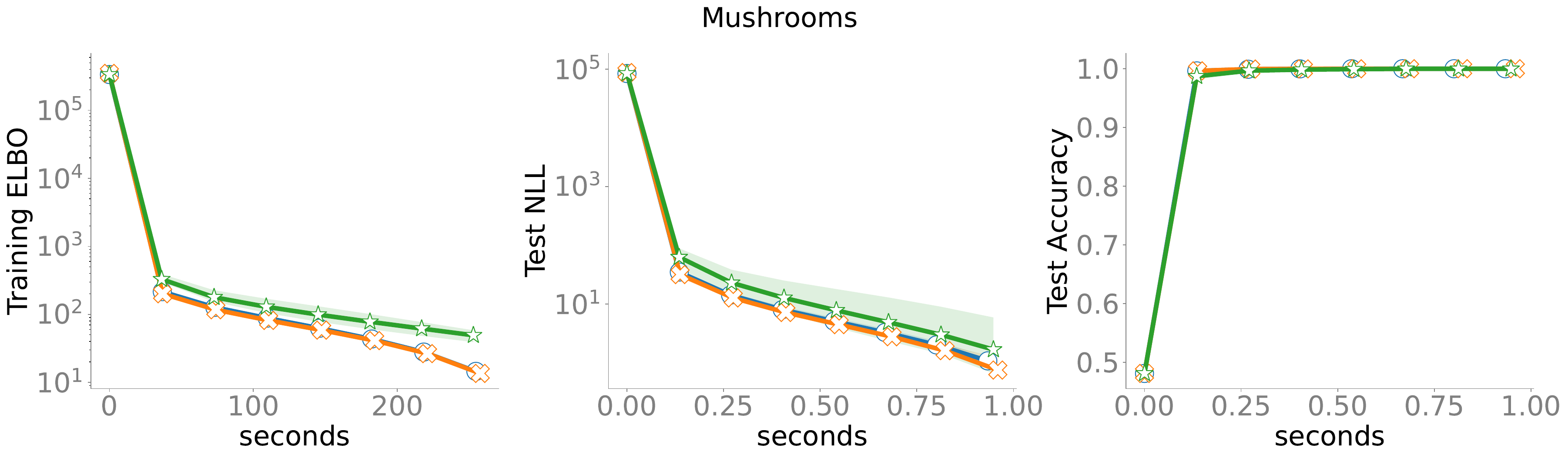}
\end{tabular}
\caption{Plots with respect to time for small-scale datasets.}
% \label{}
\end{figure}
% ---------------------------------------------
\newpage
% ---------------------------------------------

\begin{figure}[htb!]
   \centering
\begin{tabular}{ccc}
\includegraphics[width=\linewidth]{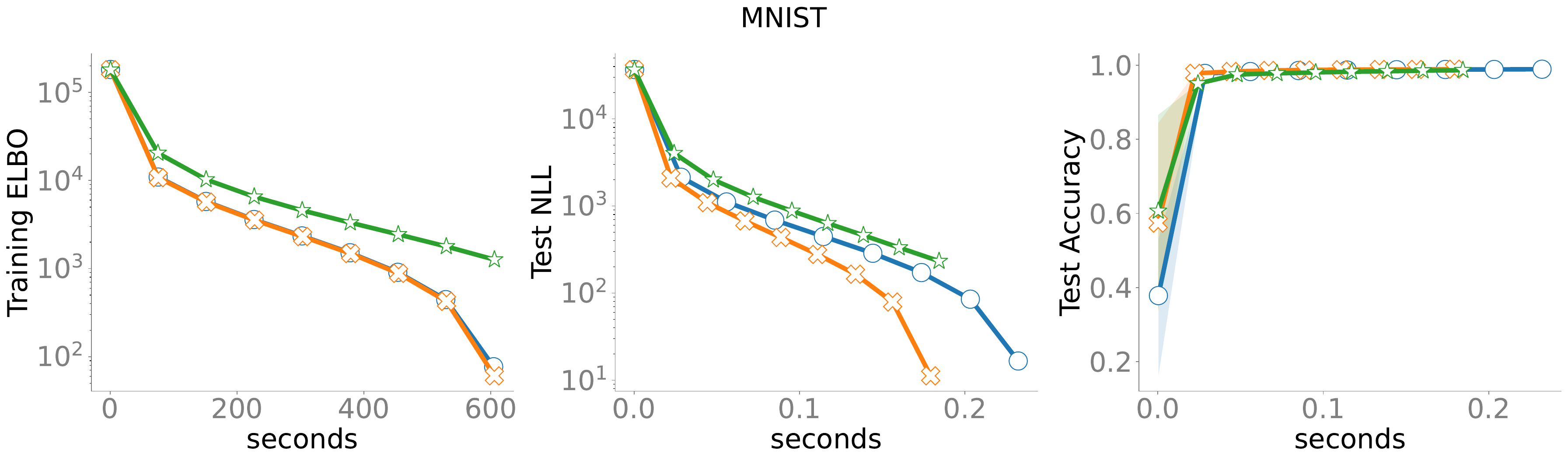}\\
\includegraphics[width=\linewidth]{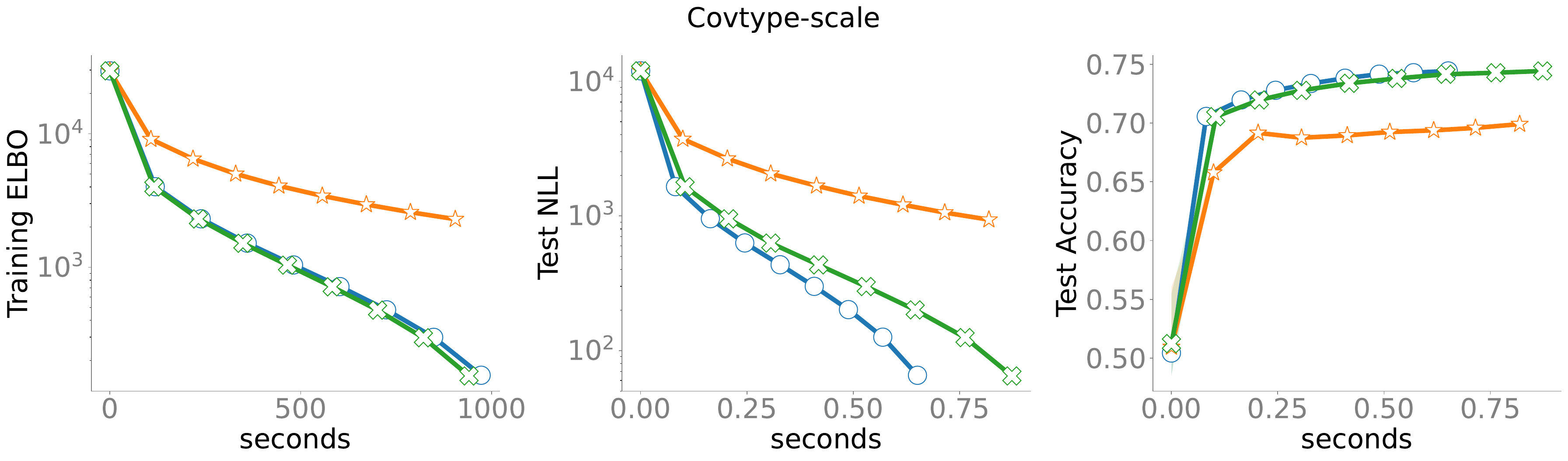}\\
\includegraphics[width=\linewidth]{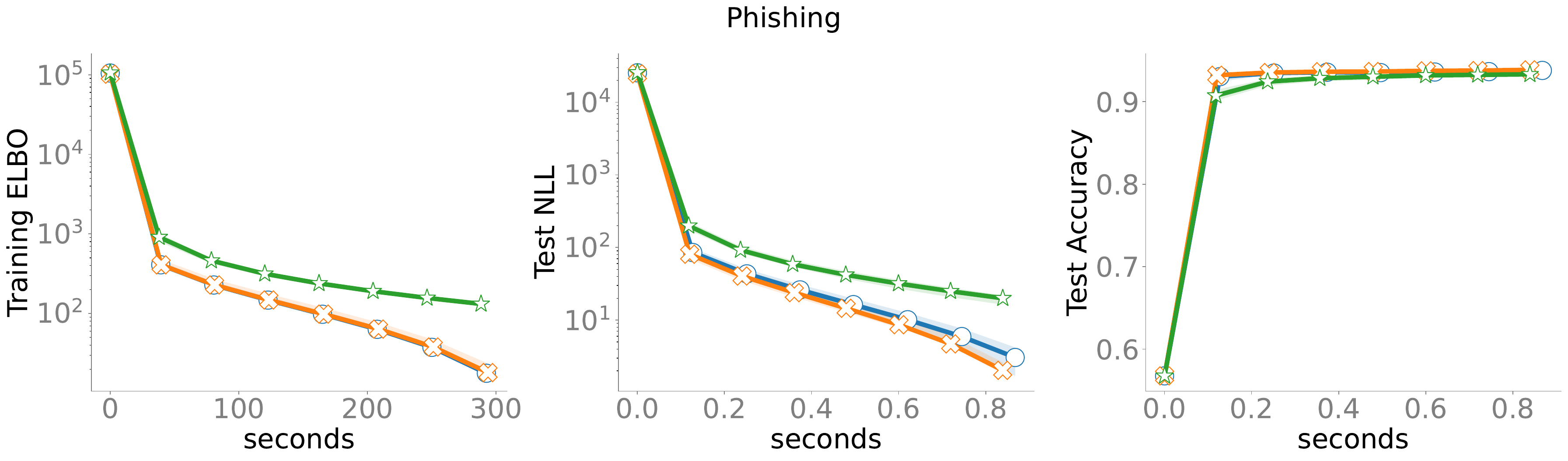}\\
\includegraphics[width=\linewidth]{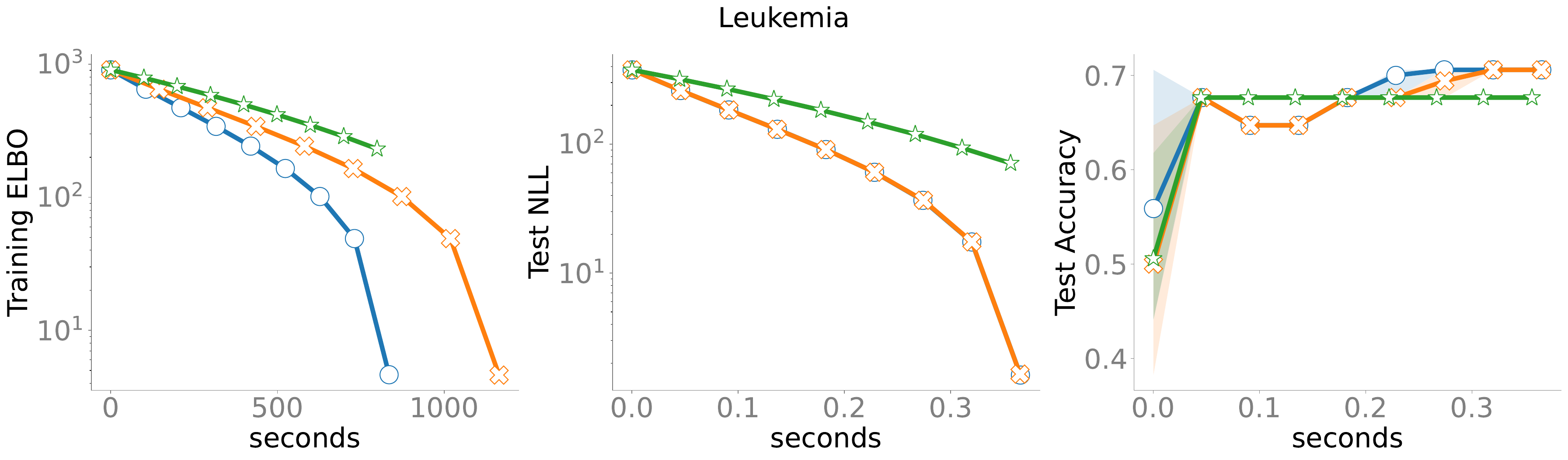}
\end{tabular}
\caption{Plots with respect to time for large-scale datasets.}
% \label{}
\end{figure}
% ---------------------------------------------

\newpage

\section{Comparing Geometries: SR-VN vs. Alternatives}
In this section, we discuss two complementary comparisons for SR-VN: (1) a theoretical analysis of its convergence rate relative to BW-GD, and (2) an empirical evaluation against an enhanced method tailored for the BW geometry.

\paragraph{Theoretical rates comparison, SR-VN vs BW-GD:} To estimate the convergence rate for BW-GD, we refer to \cite[Theorem 4]{lambert2022variational}, which—under the assumption of strong log-concavity—provides an exponential rate in the Wasserstein distance once we ignore the variance term because of its deterministic nature. Consequently, in terms of asymptotic guarantees, BW-GD and SR-VN both exhibit exponential convergence rates. However, the specific details—namely, \textbf{which geometry} (Fisher–Rao vs. Wasserstein) is deployed, and \textbf{which measure} (KL in parameter space vs. Wasserstein distance in probability space) is analyzed—separate these methods apart. As a result, establishing a theoretical comparison is challenging, leading us to focus on experimental evidence in the next section.

\paragraph{Additional experimental comparisons.} 
Although BW-GD is a natural baseline for SR‑VN in the BW geometry, more recent studies have proposed improved and more stable stochastic BW methods (e.g., \cite{diao2023forward, liu2024towards}). We focus here on \cite{liu2024towards}, which uses SVGD and has demonstrated better performance than \cite{diao2023forward}. Specifically, we implemented a \textbf{deterministic} version of \cite[Algorithm 1]{liu2024towards}—i.e., replacing sample estimates with the full expectation—and refer to this method as \textbf{SVGD-density}. It is worth mentioning that \cite{liu2024towards} demonstrated improvements in the stochastic setting compared to BW-GD. However, since we focus on the deterministic implementation here, it remains unclear whether the same advantages apply—this is precisely what we aim to address next.

We evaluated SVGD-density (using various step sizes) on the same 2D Bayesian linear regression problem from \Cref{fig:convergence-comparison}, as shown in \Cref{fig:comparison_SVGD} (upper). Additionally, we tested it on the Bayesian Logistic Regression scenario in \Cref{fig:large_scale}, as shown in \Cref{fig:comparison_SVGD} (lower). In both plots, SVGD-density shows the slowest convergence among the compared algorithms. Although it can approach the performance of BW-GD, given a finely tuned step size e.g. in \Cref{fig:comparison_SVGD} (upper), its performance is generally worse compared to SR-VN or VN.
% width=0.8\linewidth

\begin{figure}[htb!]
   \centering
\begin{tabular}{cc}
\includegraphics[ trim={2cm 0 0 0},clip,width=0.54\textheight]{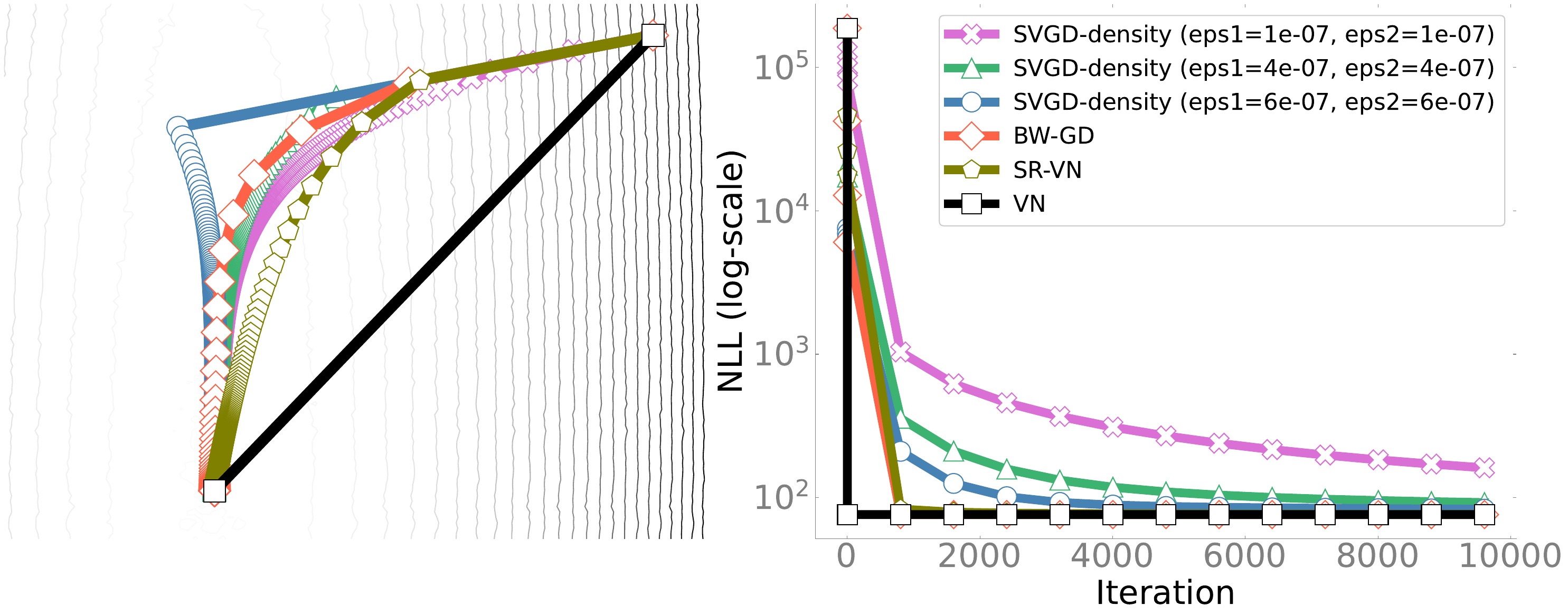}\\
\includegraphics[width=0.88\linewidth]{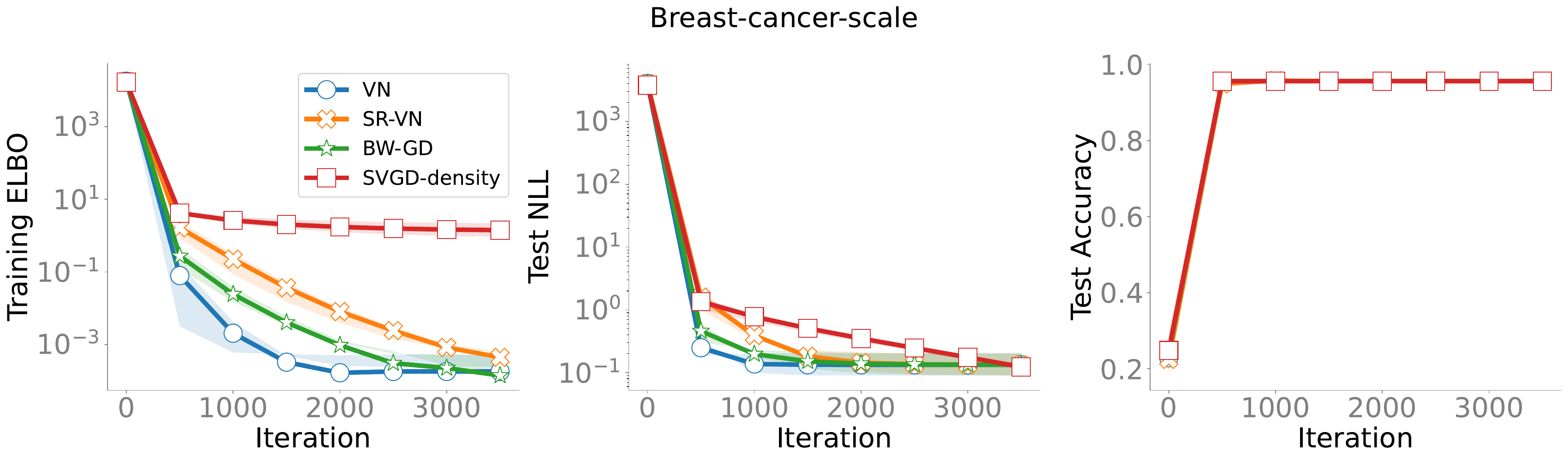}
\end{tabular}
\caption{\small{Comparison of SVGD-density: Bayesian linear regression (top) and Bayesian logistic regression (bottom).}}
\label{fig:comparison_SVGD}
\end{figure}

% \begin{figure}[htb!]
%    \centering
%    \begin{tabular}{c}
%        \begin{subfigure}{0.8\linewidth}
%            \centering
%            \includegraphics[width=\linewidth]{Images/2D_SVGD_w_NLL.pdf}
%            \caption{First subfigure label (e.g., "Comparison of NLL")}
%            \label{fig:subfig1}
%        \end{subfigure}\\
%        \begin{subfigure}{\linewidth}
%            \centering
%            \includegraphics[width=\linewidth]{Images/breast-cancer_scale.pdf}
%            \caption{Second subfigure label (e.g., "Breast Cancer Scale")}
%            \label{fig:subfig2}
%        \end{subfigure}
%    \end{tabular}
%    \caption{Overall caption for the figure.}
%    \label{fig:comparison_SVGD}
% \end{figure}

\section{Non-convex Model} \label{sec:non_convex}
In this section, we evaluate all algorithms on the same Bayesian logistic regression loss in \Cref{eq:logistic-loss} with the addition of a non-convex regularizer $\sum_{i=1}^d m_i^2/(1+m_i^2)$.

The results in \Cref{fig:non_convex} indicate that all compared methods still converge despite the non-convexity introduced in the mean parameter, suggesting promising directions for further investigation into non-convex scenarios.

% ---------------------------------------------
\begin{figure}[H]
   \centering
\includegraphics[width=\linewidth]{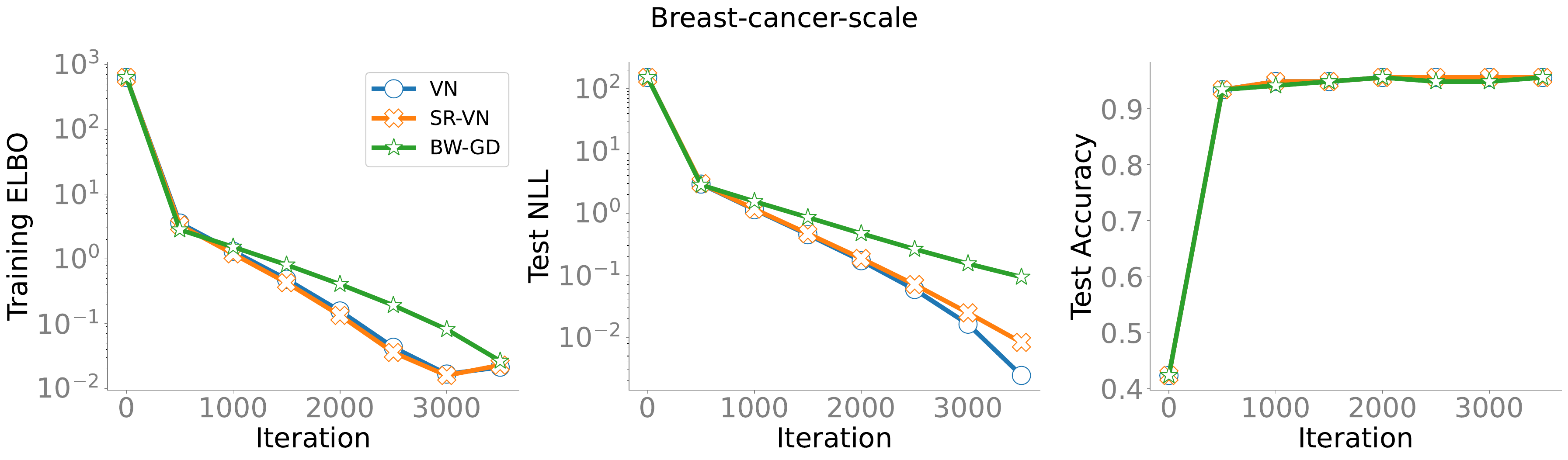}
\caption{\small{To evaluate performance under non-convex conditions, we augment the Bayesian logistic regression problem in \Cref{eq:logistic-loss} with a non-convex regularization term: $\sum_{i=1}^d m_i^2/(1+m_i^2)$. }}
\label{fig:non_convex}
\end{figure}
% ---------------------------------------------

\end{document}